\documentclass{article}

\PassOptionsToPackage{numbers}{natbib}

\usepackage[main, final]{neurips_2026}

\usepackage[utf8]{inputenc} 
\usepackage[T1]{fontenc}    
\usepackage{hyperref}       
\usepackage{url}            
\usepackage{booktabs}       
\usepackage{amsfonts}       
\usepackage{nicefrac}       
\usepackage{microtype}      
\usepackage{xcolor}         

\title{How Hard Is It for Message-Passing GNNs to Simulate One Step of the Weisfeiler-Lehman Test?}

\author{%
    Guanyu Cui\\
    Renmin University of China\\
    \texttt{cuiguanyu@ruc.edu.cn} \\
    \And
    Yuhe Guo\\
    Renmin University of China\\
    \texttt{guoyuhe@ruc.edu.cn} \\
    \And
    Zhewei Wei\thanks{Zhewei Wei is the corresponding author. The work was partially done at Gaoling School of Artificial Intelligence, Beijing Key Laboratory of Research on Large Models and Intelligent Governance, Engineering Research Center of Next-Generation Intelligent Search and Recommendation, MOE, and Pazhou Laboratory (Huangpu), Guangzhou, Guangdong 510555, China. }\\
    Renmin University of China\\
    \texttt{zhewei@ruc.edu.cn} \\
    \And
    Hsin-Hao Su\\
    Boston College\\
    \texttt{suhx@bc.edu} \\
}

\usepackage{amsmath}
\usepackage{amssymb}
\usepackage{mathtools}
\usepackage{amsthm}

\usepackage{xspace}
\usepackage{titletoc}
\usepackage[page, header]{appendix}

\usepackage{enumitem}
\usepackage{threeparttable}
\usepackage{multirow, multicol}
\usepackage{makecell}
\usepackage{mdframed}

\usepackage{float}
\usepackage{wrapfig}

\usepackage{bbding}

\usepackage{tikz}
\usetikzlibrary{positioning}
\usetikzlibrary{decorations.pathreplacing, decorations.pathmorphing}
\usetikzlibrary{arrows.meta}
\usetikzlibrary{patterns, patterns.meta}
\usetikzlibrary{calc}
\usetikzlibrary{shadows}
\usetikzlibrary{calligraphy}

\usepackage{thm-restate}
\usepackage{etoolbox}

\makeatletter
\expandafter\patchcmd\csname\string\thmt@restatable\endcsname
  {\else\@xa\protected@edef\csname the#2\endcsname}
  {\else
     \renewcommand\footnote[2][]{}%
     \@xa\protected@edef\csname the#2\endcsname}
  {}
  {\PackageWarning{thmrestate-footnote}{Failed to patch thm-restate}}
\makeatother

\theoremstyle{plain}
\newtheorem{theorem}{Theorem}[section]

\newtheorem{lemma}[theorem]{Lemma}

\theoremstyle{definition}
\newtheorem{definition}[theorem]{Definition}

\theoremstyle{remark}

\newcommand{\setclass}[1]{\mathcal{#1}}
\newcommand{\func}[1]{\mathrm{#1}}
\newcommand{\alg}[1]{\textsc{#1}}

\renewcommand{\Pr}[2][]{\mathop{\mathrm{Pr}}\limits_{#1}\left[#2\right]}
\newcommand{\E}[2][]{\mathop{\mathbb{E}}\limits_{#1}\left[#2\right]}
\newcommand{\Given}{~\big|~}
\newcommand{\bs}[1]{\boldsymbol{#1}}
\newcommand{\set}[1]{\left\{#1\right\}}
\newcommand{\mset}[1]{\left\{\!\!\left\{#1\right\}\!\!\right\}}
\newcommand{\MSG}{\mathrm{MSG}}
\newcommand{\UPD}{\mathrm{UPD}}

\newcommand{\HASH}{\mathrm{HASH}\xspace}

\newcommand{\WLtype}{\mathrm{type}}
\newcommand{\prv}{\textsf{prv}\xspace}
\newcommand{\pub}{\textsf{pub}\xspace}
\newcommand{\poly}{\mathrm{poly}}
\newcommand{\GF}{\mathrm{GF}}

\begin{document}

\maketitle

\begin{abstract}
Message-passing graph neural networks (MPGNNs) are commonly compared with the Weisfeiler-Lehman (WL) test, but this comparison does not quantify the resource parameters a network needs to implement a step of the WL test with bounded-size messages and finite numerical precision. We study the cost of simulating a single step of the WL test on unattributed graphs. We distinguish input-independent, or oblivious, simulation from instance-dependent simulation. In the former, the parameters, or their distributions in randomized models, are fixed before the input instance is known. Our results show that the local update rule of the WL test hides a global relabeling problem. In the oblivious setting, deterministic and zero-error randomized MPGNNs cannot solve this problem in the worst case using only shallow networks with small messages. We complement this lower bound with a nearly matching construction in a stronger rooted, port-aware model. By contrast, when the color set is large, bounded-error randomness can greatly reduce the cost, and a one-layer MPGNN with messages of logarithmic size and a logarithmic number of random bits suffices. We show that this logarithmic number of random bits is essentially necessary for shallow, small-message simulations. When the color set is small, we still obtain a rooted, port-aware simulation, but this construction requires more layers or larger messages. We also prove that this extra cost is partly unavoidable, as small color sets force a nontrivial trade-off between the number of layers and the message size. Finally, instance-dependent simulation can be much shallower, but the required instance-specific parameters are not necessarily easy to find. Together, these results reveal quantitative structure hidden behind the statement that MPGNNs match the WL test.
\end{abstract}

\section{Introduction}
Graph neural networks (GNNs) are a central class of models in graph machine learning, with strong empirical performance in applications such as recommendation systems~\cite{ying2018graph, wang2019neural, he2020simplifying, liang2023bagnn, chen2024how, yi2026graph}, weather forecasting~\cite{lam2023learning}, and drug discovery~\cite{gilmer2017neural, xiong2019pushing, han2025survey, zhang2025metagin}.
Alongside these empirical successes, a major theoretical question is to characterize their expressive power. 
A standard approach is to compare GNNs with the Weisfeiler-Lehman (WL) test~\cite{weisfeiler1968reduction}\footnote{Terminology and initialization conventions vary across the literature. We follow the one-dimensional refinement rule used by Xu et al.~\cite{xu2019how}, which combines a node's current color with the multiset of its neighbors' colors. On unattributed graphs, we start from a constant coloring.}.
In particular, Xu et al.~\cite{xu2019how} showed that message-passing GNNs (MPGNNs) are at most as powerful as the WL test and, over countable feature spaces, can attain this upper bound under suitable injectivity assumptions.
This perspective has motivated a broad range of more expressive architectures, including higher-order and subgraph-based models~\cite{morris2019weisfeiler, maron2019provably, geerts2022expressiveness, feng2023extending, alsentzer2020subgraph, cotta2021reconstruction, papp2021drop, feng2022how, frasca2022understanding, zhou2023from, zhang2023rethinking, zhou2023distance, tao2026exact}.
What these results do not quantify, however, is the resource cost of simulating one step of the WL test.
In particular, they do not specify the depth, message width, or numerical precision required for an input-independent MPGNN to implement even a single step of the WL test across many inputs.

To study this gap, we focus on a single step of the WL test.
This focus is natural because the WL test repeatedly applies the same refinement rule.
In one step, each node forms its WL-type, namely its current color together with the multiset of colors of its neighbors.
At first glance, simulating this step using MPGNNs appears straightforward, since the relevant information is local and injective multiset encodings are known to exist in the countable-feature setting~\cite{xu2019how}.
The subtle part is the exact, globally consistent relabeling.
Nodes with equal WL-types are required to receive the same new color, and distinct WL-types occurring in the current input are required to receive distinct new colors.
For a fixed graph and a fixed iteration, this relabeling can be chosen \textbf{\textit{after}} observing the WL-types that occur.
In contrast, an oblivious MPGNN must use a parameter-selection rule fixed \textbf{\textit{before}} the input instance is known.
We restrict attention to input colorings that arise along trajectories of the WL test on unattributed graphs, up to relabeling of color names.
We therefore ask what resources are required for an MPGNN to simulate\footnote{Here and throughout, ``simulation'' means recovery of the partition after the next step of the WL test, up to a relabeling of color names.} one step of the WL test.

To make this question precise, we distinguish two notions of simulation.
In \emph{oblivious} or input-independent simulation, all parameter-selection rules are fixed before the input instance is known.
Thus, deterministic MPGNNs use a single shared parameter choice, while randomized MPGNNs use an input-independent distribution over parameter choices, with success probability evaluated for each fixed input instance.
In \emph{instance-dependent} simulation, by contrast, the parameter choice, or, in the randomized case, the distribution over parameter choices, may depend on the particular input instance.
The dependence of parameter selection on the target instance motivates an analogy with inductive and transductive inference in statistical learning theory~\cite{vapnik2000nature, shai2014understanding, mohri2018foundations}.
Oblivious simulation uses parameters selected independently of the input instance and is analogous to inductive inference, as in GraphSAGE~\cite{hamilton2017inductive}.
Instance-dependent representation concerns the existence of suitable parameters for each observed instance and is analogous to transductive inference~\cite{joachims1999transductive}.

All results proved or discussed in this paper are summarized in Table~\ref{tab:results}. 
In the oblivious setting, we show that the apparent locality of a single step of the WL test conceals a global relabeling problem.
Any deterministic or zero-error randomized oblivious MPGNN must, in the worst case, either use large depth or transmit enough information globally (Theorem~\ref{thm:deterministic-lower}), while a nearly matching upper bound is achievable in the stronger rooted, port-aware model\footnote{In this model, one node is designated as the root, and each node has locally indexed ports, maintains a separate state for each port, and can send port-specific messages.} (Theorem~\ref{thm:deterministic-upper}).
The bounded-error randomized setting is qualitatively different. 
When the available color set is sufficiently large, hashing enables a one-layer oblivious MPGNN with logarithmic bandwidth and logarithmic or polylogarithmic randomness to simulate one step of the WL test with high probability (Lemma~\ref{lem:randomized-upper-aamand}~\cite{aamand2022exponentially}, Theorem~\ref{thm:randomized-upper-large-color}).
We further show that a logarithmic amount of randomness is essentially necessary for shallow, low-communication simulations (Theorem~\ref{thm:randomized-lower-randomness}). 
When the color set is only linearly large, however, a nontrivial depth-bandwidth trade-off reappears (Theorem~\ref{thm:randomized-lower-small-color}).
Finally, in the instance-dependent setting, we show that shallow representations become possible because the parameters may be chosen after the target instance is known (Theorem~\ref{thm:instance-dependent-representation}). 
However, we prove that finding such instance-specific parameters within our decentralized parameter-adaptation framework still requires global communication in the worst case (Theorem~\ref{thm:instance-dependent-adaptation-lower}).
\begin{table}[ht]
    \centering
    \caption{Summary of results. Here $n$ and $m$ denote the numbers of nodes and edges, respectively; $D$ denotes the diameter; $d$, $w$, and $p$ denote the depth, message width, and message precision of an MPGNN, respectively, as defined in Section~\ref{subsec:MPGNN}; $R$ denotes the number of random bits used; $C$ denotes the size of the color set, as defined in Section~\ref{subsec:setup}; and $p_f$ denotes the failure probability. The parameter $c_0 \ge 2$ depends on $p_f$. In our upper-bound constructions, we focus on communication resources (such as $d$, $w$ and $p$), and assume sufficiently large local states and sufficiently expressive local functions, including in the rooted, port-aware model.}
    \label{tab:results}
    \resizebox{\linewidth}{!}{
    \begin{tabular}{ccc}
        \toprule
        \textbf{Setting} & \textbf{Model and Assumptions} & \textbf{Results} \\
        \midrule
        
        \multirow{3}{*}{\makecell[c]{Oblivious, Deterministic \\ or Zero-Error Randomized}} & $\log C = o(m)$ & $d = \Omega\left(D + \frac{m}{wp}\right)$ (Thm. \ref{thm:deterministic-lower}) \\ \cmidrule{2-3}
         
        & Rooted, Port-Aware MPGNN & \makecell[c]{$d = O\left(D + \frac{m \log C}{wp}\right)$ (Thm. \ref{thm:deterministic-upper})} \\ \midrule
        
        \multirow{12}{*}{\makecell[c]{Oblivious, Bounded-Error \\ Randomized}} & \makecell[c]{$p_f = \Omega\left(\frac{1}{\poly(n)}\right)$, $C = \Theta(n^{c_0})$} & \makecell[c]{$d = 1$, $wp = O(\log n)$, \\$R = O(\log^2 n)$ (Lem. ~\ref{lem:randomized-upper-aamand}, \cite{aamand2022exponentially})} \\ \cmidrule{2-3}
        
        & $C \ge \left\lceil \frac{3n^2}{2p_f}\right\rceil$ & \makecell[c]{$d = 1$, $wp = \Theta(\log C)$, \\$R = \Theta\left(\log \frac{n\log C}{p_f}\right)$ (Thm. \ref{thm:randomized-upper-large-color})} \\ \cmidrule{2-3}

        & Rooted, Port-Aware MPGNN & \makecell[c]{$d = O\left(D + \frac{\log C + n\log(n / p_f)}{wp}\right)$, \\ $R = \Theta\left(\log \frac{n \log C}{p_f}\right)$ (Thm. \ref{thm:randomized-upper-small-color})} \\ \cmidrule{2-3}
        
        & $\log C = o(n)$, $dwp = O(\log C)$ & $R = \Omega\left(\log \frac{n}{p_f \log C}\right)$ (Thm. \ref{thm:randomized-lower-randomness}) \\ \cmidrule{2-3}
        
        & $p_f = \Omega\left(\frac{1}{\poly(n)}\right)$, $C = \Theta(n)$ & $d = \Omega\left(D + \min\set{\log^* n, \frac{n}{wp}}\right)$ (Thm. \ref{thm:randomized-lower-small-color}) \\ \midrule\midrule

        \makecell[c]{Instance-Dependent \\ Representation} & --- & $d = 1$, $wp = \Theta(\log n)$ (Thm. \ref{thm:instance-dependent-representation}) \\ \midrule

        \makecell[c]{Instance-Dependent \\ Adaptation, Deterministic or \\ Zero-Error Randomized} & \makecell[c]{Graph-Local (Section~\ref{subsec:instance-dependent-adaptation}) \\$\log C = o(m)$} & $r = \Omega\left(\frac{D}{d} + \frac{m}{dwp}\right)$ (Thm. \ref{thm:instance-dependent-adaptation-lower}) \\
        \bottomrule
    \end{tabular}
    }
\end{table}


\section{Preliminaries and Problem Setup}
\label{sec:preliminaries}
We introduce notation, review MPGNNs and the WL test, and define the simulation problem.

\subsection{Notation}
We write $\set{\cdot}$ for sets and $\mset{\cdot}$ for multisets. 
The notation $[n]$ stands for $\set{1, \ldots, n}$. 
Boldface lowercase letters, such as $\bs{x}$, denote vectors, and boldface uppercase letters, such as $\bs{X}$, denote matrices. 
We write $\bs{x}_i$ for the $i$-th coordinate of $\bs{x}$ and $\bs{X}_{i,:}$ for the $i$-th row of $\bs{X}$.
We consider simple undirected graphs $G = (V, E)$, where $E \subseteq \set{\set{u, v} : u, v \in V, u \neq v}$.
Unless stated otherwise, all graphs in this paper are finite, undirected, and connected. 
For a node $u \in V$, let $N_G(u) \coloneqq \set{v \in V : \set{u, v} \in E}$ denote its neighborhood, and write $N(u)$ when $G$ is clear from context. 
We define $\deg(u) \coloneqq |N(u)|$, $n \coloneqq |V|$, and $m \coloneqq |E|$. 
The diameter of $G$ is $D \coloneqq \max\limits_{u, v \in V} \func{dist}_G(u, v)$, where $\func{dist}_G(u, v)$ is the graph distance between $u$ and $v$.

\subsection{Message-Passing Graph Neural Networks}
\label{subsec:MPGNN}
Among the many formulations of GNNs, we use the message-passing GNN (MPGNN) framework from~\cite{gilmer2017neural}. 
For the model class studied here, each layer has the form
\begin{equation}
\label{eqn:message-passing}
\bs{X}^{(\ell+1)}_{u, :} = \UPD^{(\ell)}\left(\bs{X}^{(\ell)}_{u,:}, \mset{\MSG^{(\ell)}\left(\bs{X}^{(\ell)}_{v,:}\right): v \in N(u)}\right),
\end{equation}
where $\MSG^{(\ell)}(\cdot)$ maps a sender state to a message, and $\UPD^{(\ell)}(\cdot,\cdot)$ combines the current node state with the multiset of incoming messages. 
The update function is required to be permutation-invariant in its second argument. 
This formulation makes the communication structure explicit and corresponds to the aggregate-combine viewpoint of Xu et al.~\cite{xu2019how}.
Unless explicitly stated otherwise, this is the anonymous unordered-neighborhood model used throughout the lower-bound statements: nodes have no unique identifiers, no port numbers, and no persistent per-edge states.

In our analysis, the relevant resource measures are communication resources. 
The \textbf{depth} $d$ is the number of message-passing layers. 
The \textbf{width} $w$ is the maximum dimension of a message vector transmitted along an edge\footnote{In this paper, ``width'' refers to message dimension rather than hidden-state dimension. In many standard models the two coincide, but we keep the distinction explicit because our bounds are stated in terms of transmitted information.}, and the \textbf{precision} $p$ is the number of bits used to represent each coordinate of a transmitted message.
Thus, each edge can transmit at most $wp$ bits per layer, which we call the bandwidth, and at most $dwp$ bits over $d$ layers in one direction.
Local memory and computation are unrestricted in our models.
For our upper-bound constructions, the local functions are assumed to be sufficiently expressive to perform canonical encoding, hash evaluation, sorting, and ranking exactly.

\subsection{Weisfeiler-Lehman Test}
\label{subsec:WL}
The WL test iteratively refines node colors by injectively encoding each node's current color together with the multiset of its neighbors' colors~\cite{weisfeiler1968reduction, xu2019how}.
Let $G = (V, E)$ be a graph equipped with an initial coloring $\chi_G^{(0)} : V \to [C]$. 
The refinement rule is
\begin{equation}
\label{eqn:WL}
\chi_G^{(\ell+1)}(u) = \HASH^{(\ell)}\left(\chi_G^{(\ell)}(u), \mset{\chi_G^{(\ell)}(v) : v \in N(u)}\right),
\end{equation}
where $\chi_G^{(\ell)}(u)$ is the color of node $u$ after iteration $\ell$. 
The pair $\left(\chi_G^{(\ell)}(u), \mset{\chi_G^{(\ell)}(v) : v \in N(u)}\right)$ is the \emph{WL-type} of $u$ at iteration $\ell$, and $\HASH^{(\ell)}$ is required to be injective on the WL-types that occur at that iteration. 
For unattributed graphs, we start from the constant coloring $\chi_G^{(0)}(u)=1$ for all $u\in V$.
Since the WL test depends only on the partition induced by a coloring, rather than on the particular names of the colors, we write $\chi \equiv \chi'$ if $\chi(u) = \chi(v) \iff \chi'(u) = \chi'(v)$ for all $u, v \in V$.
Accordingly, all simulation guarantees in this paper are stated up to such a relabeling of color names.

\subsection{Problem Setup: Simulating One Step of the WL Test}
\label{subsec:setup}
We focus on one step of the WL test for colorings that can arise at arbitrary iterations of the WL test on unattributed graphs started from the constant coloring.
Fix a color set $[C]$ large enough to serve as the label space for both input and output colorings.
For a graph $G$, let $\chi_G^{(\ell)}$ denote the coloring after the $\ell$-th WL iteration, starting from the constant coloring $\chi_G^{(0)}(u)=1$ for all $u\in V(G)$.
A pair $(G, \chi)$ is called a \emph{valid input instance} if $\chi : V(G) \to [C]$ and $\chi \equiv \chi_G^{(\ell)}$ for some iteration $\ell \ge 0$.

We regard an MPGNN, together with its input encoding and output decoding, as a parameterized map from colorings to colorings.
Let $\mathcal{P}$ denote the set of all allowable parameter choices of an MPGNN.
For $\theta \in \mathcal{P}$ and a graph $G$, write $\mathcal N_{\theta}(\cdot; G) \colon [C]^{V(G)} \to [C]^{V(G)}$ for the coloring-to-coloring map induced by evaluating the MPGNN with parameters $\theta$ on the edge set of $G$.
Here $G$ is not part of the learned parameters; it specifies the graph on which the same parameter choice is evaluated.
For a valid input instance $(G, \chi)$, define the WL-type of a node $u$ under $\chi$ by $\WLtype_{G,\chi}(u) \coloneqq \left(\chi(u), \mset{\chi(v) : v \in N(u)}\right)$.
We say that the parameter choice $\theta$ \emph{simulates one step of the WL test} on a valid input instance $(G, \chi)$ if $\mathcal N_{\theta}(\chi; G)(u) = \mathcal N_{\theta}(\chi; G)(v) \iff \WLtype_{G,\chi}(u) = \WLtype_{G,\chi}(v)$ for all $u,v \in V(G)$.

For a graph family $\mathcal{G}$, an MPGNN \textbf{deterministically obliviously} simulates one step of the WL test on $\mathcal{G}$ over $[C]$ if there exists one parameter choice $\theta^\star \in \mathcal{P}$, chosen independently of the input instance, such that $\theta^\star$ simulates one step of the WL test on every valid input instance $(G, \chi)$ with $G \in \mathcal{G}$.

Extending deterministic oblivious simulation to the randomized setting involves a subtle choice between two notions of success.
Let $\rho\in\set{0,1}^R$ be a random seed drawn from an input-independent distribution.
A parameter-generation rule $\Phi:\set{0,1}^R\to\mathcal P$, fixed before the input instance is known, maps each seed $\rho$ to the network parameters $\theta_\rho=\Phi(\rho)$.
For a failure probability $p_f\in[0,1)$, consider the following two guarantees over all valid input instances $(G,\chi)$ with $G\in\mathcal G$:
\begin{enumerate}[label=(\roman*),leftmargin=*]
    \item \emph{Per-instance success.} The simulation succeeds with probability at least $1-p_f$ on each fixed input, i.e., for all $(G,\chi)$, $\Pr[\rho]{\theta_\rho\text{ succeeds on }(G,\chi)}\ge\nobreak 1-p_f$.
    \item \emph{Simultaneous success.} With probability at least $1-p_f$, the sampled parameters yield a simulator that is correct on every input instance, i.e., $\Pr[\rho]{\forall (G,\chi),\ \theta_\rho\text{ succeeds on }(G,\chi)}\ge\nobreak 1-p_f$.
\end{enumerate}
In fact, for $p_f<1$, setting~(ii) is equivalent to the deterministic oblivious setting.
Indeed, the simultaneous guarantee implies that some seed $\rho^\star$ succeeds on every instance; fixing this seed gives a deterministic oblivious simulator with no larger depth, width, or precision.
Conversely, any deterministic oblivious simulator satisfies~(ii) using a point-mass parameter distribution.
We therefore adopt~(i) as the randomized oblivious extension here.
The term ``oblivious'' means that both the seed distribution and the map $\Phi$ from seeds to network parameters are fixed independently of the input instance.
For each instance, a seed may be sampled from this same distribution, and the successful seeds may differ across instances.
Formally, an MPGNN with $R$ random bits \textbf{randomized obliviously} simulates one step of the WL test on $\mathcal G$ over $[C]$ with failure probability at most $p_f$ if there exist an input-independent seed distribution and map $\Phi$ satisfying~(i).
The special case $p_f=0$ of our per-instance guarantee is called \textbf{zero-error randomized oblivious} simulation.
Because the seed space $\set{0,1}^R$ is finite, any supported seed that failed on any instance would give that instance positive failure probability.
Hence every supported seed succeeds on every instance, and fixing any one gives a deterministic oblivious simulator with no larger $d,w,p$.
Together with the converse above, this shows that deterministic and zero-error randomized oblivious simulation are also equivalent in these resources.
We therefore treat these two settings together throughout the paper.

In contrast to oblivious simulation, where the parameter choice or its sampling distribution is fixed independently of the input, instance-dependent representation allows the parameters to depend on the input instance.
Formally, an MPGNN admits an \textbf{instance-dependent representation} of one step of the WL test on $\mathcal{G}$ over $[C]$ if for every valid input instance $(G, \chi)$ with $G \in \mathcal{G}$, there exists a parameter choice $\theta_{G, \chi} \in \mathcal{P}$, possibly depending on $(G, \chi)$, such that $\theta_{G, \chi}$ simulates one step of the WL test on that instance.
This definition concerns only the existence of suitable parameters for each instance, so there is no need to distinguish between deterministic and randomized instance-dependent representation.
Whenever random parameter selection succeeds with positive probability on a fixed instance, at least one seed produces a parameter choice that correctly simulates one step of the WL test on that instance.
Fixing such a seed separately for each instance gives the required instance-dependent parameters.

Instance-dependent representation concerns only the existence of suitable parameters; we further study \textbf{instance-dependent adaptation}, which asks how many adaptation rounds a network needs on a given instance before it can correctly simulate one step of the WL test.
The parameter initialization rule and adaptation procedure are fixed independently of the target instance and shared across the graph family.
On each valid input instance $(G,\chi)$, the procedure performs $r$ adaptation rounds using graph-local communication to adapt a depth-$d$, width-$w$, precision-$p$ MPGNN.
The adapted network is then evaluated on that same instance using $d$ message-passing layers.
Adaptation succeeds on $(G,\chi)$ if this evaluation correctly simulates one step of the WL test.
As in the randomized oblivious setting, randomized adaptation uses a finite random seed drawn from an input-independent distribution to govern parameter initialization, adaptation, and final evaluation, with the same per-instance success guarantee.
By the same finite-seed argument as above, zero-error randomized adaptation is equivalent to deterministic adaptation.
The full definition of instance-dependent adaptation is given in Section~\ref{subsec:instance-dependent-adaptation}.


\section{Oblivious Simulation Results}
\label{sec:oblivious-results}
In this section, we analyze the hardness of obliviously simulating one step of the WL test with MPGNNs.
Concretely, we study worst-case trade-offs among depth $d$, width $w$, and precision $p$.

\subsection{The Deterministic and Randomized Zero-Error Cases}
We first consider the deterministic and zero-error randomized settings, which are equivalent in attainable depth, width, and precision as shown in Section~\ref{subsec:setup}.

One might expect one step of the WL test to be computable by a depth-one MPGNN, because the WL-type of a node depends only on its own color and the multiset of neighboring colors.
This intuition overlooks the need for an exact, collision-free relabeling of all WL-types that occur into a bounded color set. 
In the oblivious setting, this relabeling may require information to be coordinated across distant parts of the graph. 
The following theorem formalizes this obstruction.
\begin{restatable}{theorem}{DeterLower}
    \label{thm:deterministic-lower}
    For every sufficiently large integer $n$, every integer $m \in [n, n^2]$, and every integer $D \in [8, n]$, there exists a family $\mathcal{G}_{n, m, D}$ of connected graphs, each with $\Theta(n)$ nodes, $\Theta(m)$ edges, and diameter $D$.
    Then, for any color set $[C]$ with $C \ge \max\limits_{G \in \mathcal{G}_{n, m, D}} |V(G)|$\footnote{Kiefer et al.~\cite{kiefer2020iteration} constructed a family of graphs for which the WL test requires $|V(G)| - 1$ rounds to stabilize, and hence requires $|V(G)|$ colors. Therefore, in the oblivious setting, the color set size $C$ should be at least $|V(G)|$.} and $\log C = o(m)$, any deterministic or zero-error randomized oblivious MPGNN that simulates one step of the WL test on $\mathcal{G}_{n, m, D}$ over $[C]$ must satisfy $d = \Omega\left(D + \frac{m}{w p}\right)$.
\end{restatable}
The proof is deferred to Appendix~\ref{app:proof-deterministic-lower}.
The proof idea is to construct a family of graphs in which two distant regions encode an instance $(\bs{a},\bs{b})$ of the Equality problem from communication complexity~\cite{yao1979some,kushilevitz1997communication-Ch3}. 
If an oblivious MPGNN correctly performs one step of the WL test on these graphs, then it can be used to construct a zero-error communication protocol for the Equality problem.
Since each edge can transmit at most $wp$ bits per layer, $d$ layers can transmit only $O(dwp)$ bits across the communication bottleneck in each graph in the family.
The zero-error communication lower bound for Equality yields the $\Omega\left(\frac{m}{wp}\right)$ term, while the separation between the relevant graph regions gives the additional $\Omega(D)$ locality term.

The lower bound gives a necessary amount of communication, measured through the product of depth, width, and precision, for oblivious exact simulation. 
We next ask to what extent this lower bound can be matched by an explicit construction.
A natural source of upper-bound algorithms is the CONGEST model~\cite{linial1987distributive, linial1992locality, peleg2000distributed, ghaffari2022distributed} from distributed computing (see Appendix~\ref{app:proof-deterministic-upper} for the definition), which was used by Loukas~\cite{loukas2020what} as a computational abstraction for MPGNNs. 
However, standard CONGEST algorithms often assume unique node identifiers or port-level control. 
Unique identifiers would break the anonymity and permutation equivariance of the standard MPGNN model, and may not generalize well to graphs with unseen identifiers, while port-level control is not available in the unordered-neighborhood update of Equation~\eqref{eqn:message-passing}.

For the upper bound only, we therefore consider a rooted, port-aware relaxation of the anonymous MPGNN model. 
One node receives a special one-bit marker designating it as the root. 
In addition, each node can distinguish its incident edge ports, maintain a hidden state for each port, send messages depending on the corresponding port state, and update each port state using the node state, the current port state, and the message received through that port (see Appendix~\ref{app:proof-deterministic-upper} for a detailed definition).
This relaxation is strictly stronger than the anonymous unordered-neighborhood MPGNN model used in the lower-bound theorem.
The following theorem gives an upper bound for simulating one step of the WL test under this model.
\begin{restatable}{theorem}{DeterUpper}
    \label{thm:deterministic-upper}
    Given any family of graphs with at most $n$ nodes, at most $m$ edges, and diameter at most $D$, where each graph is supplied with an arbitrary designated root marker, and a color set $[C]$ large enough to represent all input and output colors under consideration (e.g., $C \ge n$), there exists an oblivious rooted, port-aware MPGNN with depth $d = O\left(D + \frac{m \log C}{wp}\right)$, sufficiently large node and port state dimensions, and sufficiently expressive message, node-update, and port-update functions, that simulates one step of the WL test in both the deterministic and zero-error randomized settings.
\end{restatable}
The proof is deferred to Appendix~\ref{app:proof-deterministic-upper}.
The proof idea is to use the root marker and port states to construct a rooted spanning tree. 
The WL-types that occur are then aggregated to the root, where duplicate types are removed and the distinct types are sorted to assign ranks, which serve as the new color names. 
The resulting type-to-color map is broadcast back to all nodes along the tree, after which each node computes its new color by applying this map to its own WL-type.
This result shows that, in the stronger rooted, port-aware model, the deterministic lower bound is matched up to the $\log C$ factor in the communication term.
Whether a comparable deterministic upper bound can be achieved in the fully anonymous model remains open.

\subsection{The Bounded-Error Randomized Case}
We then turn to bounded-error randomized simulation, where the failure probability may be positive but must be at most $p_f$, leading to a picture quite different from the deterministic and zero-error randomized cases.

Aamand et al.~\cite{aamand2022exponentially} explored this question using universal hashing and pseudorandomness as tools.
We now restate their more efficient construction as Lemma~\ref{lem:randomized-upper-aamand}.
\begin{lemma}[cf.\ Section~3 of~\cite{aamand2022exponentially}]
    \label{lem:randomized-upper-aamand}
    Given any family of graphs with at most $n$ nodes and a failure probability $p_f \in (0, 1)$ with $p_f = \Omega\left(\frac{1}{\poly(n)}\right)$, there exists an MPGNN with depth $d = 1$ and bandwidth $wp = O(\log n)$ that simulates one step of the WL test with probability at least $1 - p_f$, using $R = O(\log^2 n)$ random bits and a color set $[C]$ of size $C = \Theta(n^{c_0})$ for some constant $c_0$ depending on $p_f$\footnote{For their hashing-based construction, ensuring that, with probability at least $1-p_f$, no two distinct WL-types that occur collide requires a range of size on the order of $\frac{n^2}{p_f}$ by a union-bound argument. In particular, this implies $c_0 \ge 2$.}.
\end{lemma}
We describe the construction behind Lemma~\ref{lem:randomized-upper-aamand} in Appendix~\ref{app:introduction-aamand}.
The advantage of their construction is that its local computation can be implemented by ReLU networks, and the authors also analyze lower bounds for it. 
A limitation of that construction is that the required color-set size depends on the failure probability $p_f$, and the constant $c_0$ is difficult to determine precisely.
Motivated by this, we present two alternative constructions in Theorem~\ref{thm:randomized-upper-large-color} and Theorem~\ref{thm:randomized-upper-small-color}.
The first construction gives an upper bound for large color sets in the fully anonymous MPGNN model.
\begin{restatable}{theorem}{RandUpperLarge}
    \label{thm:randomized-upper-large-color}
    Given any family of graphs with at most $n$ nodes, a failure probability $p_f \in (0, 1)$, and a color set $[C]$ with $C \ge \left\lceil \frac{3n^2}{2p_f}\right\rceil$, there exists an MPGNN with depth $d = 1$, $wp \ge \left\lceil \log_2 C\right\rceil$, sufficiently large node state dimension, and sufficiently expressive message and node-update functions, that simulates one step of the WL test with probability at least $1 - p_f$, using $R = \Theta\left(\log \frac{n \log C}{p_f}\right)$ random bits.
\end{restatable}
The second construction gives an upper bound for arbitrary color sets in the rooted, port-aware model.
\begin{restatable}{theorem}{RandUpperSmall}
    \label{thm:randomized-upper-small-color}
    Given any family of graphs with at most $n$ nodes and diameter at most $D$, where each graph is supplied with an arbitrary designated root marker, a failure probability $p_f \in (0, 1)$, a color set $[C]$ with $C \ge n$, there exists a rooted, port-aware MPGNN with depth $d = O\left(D + \frac{\log C + n \log (n / p_f)}{wp}\right)$, sufficiently large node and port state dimensions, and sufficiently expressive message, node-update, and port-update functions, that simulates one step of the WL test with probability at least $1 - p_f$, using $R = \Theta\left(\log \frac{n \log C}{p_f}\right)$ random bits.
\end{restatable}
The proofs of Theorem~\ref{thm:randomized-upper-large-color} and Theorem~\ref{thm:randomized-upper-small-color} are deferred to Appendix~\ref{app:proof-randomized-upper}.
The proofs mainly rely on almost universal hashing~\cite{carter1977universal, alon1992simple}. 
Once each node has formed its WL-type, it can hash this type to a range $[M]$ of size $M = \Theta\left(\frac{n^2}{p_f}\right)$. 
By a union bound over all pairs among the at most $n$ WL-types that occur, the probability that any two distinct WL-types that occur collide is at most $p_f$.
If the size of the color set is at least $\Theta\left(\frac{n^2}{p_f}\right)$, then the hash value itself can be treated as the new color, and no communication beyond forming the WL-type is needed.
When $C$ is below the $\Theta\left(\frac{n^2}{p_f}\right)$ threshold, we additionally apply a construction similar to Theorem~\ref{thm:deterministic-upper} to sort the hash values and use their ranks as the new colors.
In the first construction (Theorem~\ref{thm:randomized-upper-large-color}), consider the practical setting where $p_f = \Omega\left(\frac{1}{\poly(n)}\right)$ and $\left\lceil \frac{3n^2}{2p_f}\right\rceil \le C \le \poly(n)$.
If $wp = k\log n$ for a sufficiently large constant $k > 0$, then $d = 1$.
In this regime, the MPGNN depth and per-layer bandwidth $wp$ are comparable to those in Lemma~\ref{lem:randomized-upper-aamand}, but we reduce the number of random bits from $O(\log^2 n)$ to $\Theta\left(\log \frac{n\log C}{p_f}\right) = \Theta(\log n)$, and we can explicitly characterize the required color-set size.
In the second construction (Theorem~\ref{thm:randomized-upper-small-color}), the depth and total communication requirements are larger, but the color-set size can be any $C\ge n$ and no longer depends on $p_f$.

Both the construction of Aamand et al.~\cite{aamand2022exponentially} and our large-color construction achieve low total communication $dwp$.
However, it remains unclear whether the number of random bits required by such constructions is close to optimal. 
In Theorem~\ref{thm:randomized-lower-randomness} below, we prove a lower bound on the required number of random bits.
\begin{restatable}{theorem}{RandLowerRandomness}
    \label{thm:randomized-lower-randomness}
    For every sufficiently large integer $n$, there exists a family $\mathcal{G}_{n}$ of connected graphs, each with $\Theta(n)$ nodes.
    Then, for any failure probability $p_f \in (0, 1/9)$ and any color set $[C]$ with $C \ge \max\limits_{G \in \mathcal{G}_{n}} |V(G)|$ and $\log C = o(n)$, any randomized oblivious MPGNN with $dwp = O(\log C)$ that simulates one step of the WL test on $\mathcal{G}_{n}$ over $[C]$ with probability at least $1 - p_f$ must use $R = \Omega\left(\log \frac{n}{p_f \log C}\right)$ random bits.
\end{restatable}
The proof of Theorem~\ref{thm:randomized-lower-randomness} is deferred to Appendix~\ref{app:proof-randomized-lower-randomness}.
The proof again reduces the Equality problem to simulating a single step of the WL test, and then applies the lower bound from~\cite{jacobs2025communication} on the communication complexity of the Equality problem in terms of the error probability $p_f$ and the number of random bits $R$.
When instantiated with $wp = O(\log C)$, the construction in Theorem~\ref{thm:randomized-upper-large-color} has $dwp = O(\log C)$. 
For polynomial-size color sets satisfying $C \ge \left\lceil\frac{3n^2}{2p_f}\right\rceil$ and $p_f\in(0,1/9)$, the upper and lower bounds on the number of random bits match up to constant factors.

When the allowed color-set size is below the $\Theta\left(\frac{n^2}{p_f}\right)$ threshold required by Theorem~\ref{thm:randomized-upper-large-color}, Theorem~\ref{thm:randomized-upper-small-color} provides an alternative upper bound on the trade-off among depth, width, and precision.
A natural question is whether this upper bound can be improved under such a restricted color space.
In Theorem~\ref{thm:randomized-lower-small-color} below, we show that when $C = \Theta(n)$, any randomized oblivious MPGNN that simulates one step of the WL test on the hard family must have depth $d = \Omega\left(D + \min\set{\log^* n,\frac{n}{wp}}\right)$.
\begin{restatable}{theorem}{RandLowerSmallColor}
    \label{thm:randomized-lower-small-color}
    For every sufficiently large integer $n$ and every integer $D \in [8,n]$, there exists a family $\mathcal{G}_{n,D}$ of connected graphs, each with $\Theta(n)$ nodes and diameter $\Theta(D)$.
    Then, for any failure probability $p_f \in (0, 1/6)$ satisfying $p_f = \Omega\left(\frac{1}{\poly(n)}\right)$, and any color set $[C]$ with $C \ge \max\limits_{G \in \mathcal{G}_{n, D}} |V(G)|$ and $C = \Theta(n)$, any randomized oblivious MPGNN that simulates one step of the WL test on $\mathcal{G}_{n, D}$ over $[C]$ with probability at least $1 - p_f$ must satisfy $d = \Omega\left(D + \min\left\{\log^* n, \frac{n}{wp}\right\}\right)$.
\end{restatable}
The proof of Theorem~\ref{thm:randomized-lower-small-color} is deferred to Appendix~\ref{app:proof-randomized-lower-small-color}.
The proof reduces the $\func{ExistsEqual}$ problem (see, e.g.,~\cite{huang2020communication}) to simulating one step of the WL test with success probability at least $1 - p_f$.

To summarize, bounded-error randomness leads to a markedly different picture from the deterministic and zero-error regimes. 
When a sufficiently large polynomial-size color set is available, for example when $C=\poly(n)$, $C = \Omega\left(\frac{n^2}{p_f}\right)$, and $p_f = \Omega\left(\frac{1}{\poly(n)}\right)$, one step of the WL test can be simulated by an oblivious MPGNN with $d = 1$, $wp = O(\log n)$, and $R = \Theta(\log n)$ random bits.
At the minimum color set size $C = \Theta(n)$, however, for $p_f = \Omega\left(\frac{1}{\poly(n)}\right)$, Theorem~\ref{thm:randomized-lower-small-color} rules out constant-depth simulations with logarithmic per-layer bandwidth $wp = \Theta(\log n)$ in the anonymous MPGNN model.
This highlights that bounded-error randomness can dramatically reduce the cost of simulating one step of the WL test in the large-color regime, while small color sets still impose a nontrivial communication-depth trade-off.

We also note that the lower-bound proofs in this section do not rely on anonymity or unordered neighborhood aggregation in the basic MPGNN model.
They use only two constraints that remain valid in the rooted, port-aware strengthening used for our upper bounds: each edge transmits at most $O(wp)$ bits per layer, and a depth-$d$ message-passing computation has locality radius at most $d$.
Consequently, the lower bounds in this section also hold for the rooted, port-aware model.

\section{Instance-Dependent Simulation Results}
\label{sec:instance-dependent-results}
In the previous section, we studied the hardness of simulating one step of the WL test with MPGNNs in the oblivious setting.
In this section, we remove this input-independence requirement and study instance-dependent simulation.

We separate the issue into two questions.
First, if the parameters may be chosen by an oracle after seeing the target instance, can a shallow MPGNN represent one step of the WL test on that instance?
Second, if such parameters exist, can they be found efficiently by a local target-instance adaptation procedure?
The answer to the first question is positive: shallow instance-specific parameters always exist.
The answer to the second question is negative in the worst case: under a decentralized parameter-adaptation model, finding such parameters still requires global communication.

\subsection{Instance-Dependent Representation}
As discussed in Section~\ref{subsec:setup}, there is no need to distinguish between deterministic and randomized instance-dependent representation.
We first show that, once input-independence is removed, the representational difficulty of one step of the WL test essentially disappears.
\begin{restatable}{theorem}{InstDepRep}
\label{thm:instance-dependent-representation}
    Given any valid input instance $(G, \chi)$, where $G$ is a graph with $n$ nodes, any color set $[C]$ with $C \ge n$, and any width $w$ and precision $p$ satisfying $wp \ge \lceil \log_2 n\rceil$, there exists an instance-dependent parameter choice $\theta_{G, \chi}$ for a deterministic depth-one MPGNN with width $w$, precision $p$, sufficiently large hidden-state dimension, and instance-dependent embedding, message, update, and decoding functions, such that the resulting MPGNN simulates one step of the WL test on the fixed instance $(G, \chi)$, that is, for all $u, v\in V(G)$, $\mathcal{N}_{\theta_{G, \chi}}(\chi; G)(u) =\mathcal{N}_{\theta_{G, \chi}}(\chi; G)(v) \iff \WLtype_{G, \chi}(u)=\WLtype_{G, \chi}(v)$.
\end{restatable}
The proof of Theorem~\ref{thm:instance-dependent-representation} is deferred to Appendix~\ref{app:proof-instance-dependent-representation}.
The proof first compresses each input color injectively into $\lceil \log_2 n\rceil$-bit codes, so that each node can transmit its compressed color to all neighbors in one layer.
Each node then forms a canonical vectorization of its compressed WL-type.
Finally, because at most $n$ compressed WL-types occur and $C \ge n$, an instance-dependent decoder can injectively map the occurring WL-types to colors in $[C]$.

\subsection{Instance-Dependent Adaptation}
\label{subsec:instance-dependent-adaptation}
As discussed in Section~\ref{subsec:setup}, the existence of suitable parameters for each instance does not imply that they can be found efficiently.
We therefore study instance-dependent adaptation, for which we give the full definition below.
A motivating setting is target-specific adaptation on a graph too large for a single machine, where machines storing different graph partitions must communicate to coordinate the adaptation.
Our framework captures this coordination when the adaptation procedure obeys the communication constraints defined below.

The parameter initialization rule and adaptation algorithm are fixed independently of the target instance and shared across the graph family.
The initial parameters $\theta_0$ may, for example, be fixed, hand-designed, or pretrained on data independent of the target family.
Given a target instance $(G,\chi)$, each node initially knows only its own input color, its incident edges or ports, and the target-independent initialization.
Any additional target-dependent information, including any loss, gradient, collision statistic, or other statistic used by the adaptation rule, must be obtained through local computation and communication along the edges of $G$.
The desired refined color labels are not provided during adaptation.

One adaptation round allows arbitrary local computation and communication consistent with a forward-backward computation of a depth-$d$, width-$w$, precision-$p$ MPGNN.
In each round, at most $O(dwp)$ bits can cross each edge, and information can propagate at most $O(d)$ hops.
After $r$ rounds, the algorithm produces an adapted configuration $\Gamma_{r;G,\chi}=\alg{Adapt}^{(r)}(\theta_0,G,\chi)$ containing parameters and node-local states.
Each node may use only the part of this configuration stored locally or received through subsequent message passing.
After adaptation, a depth-$d$, width-$w$, precision-$p$ message-passing computation initialized with $\Gamma_{r;G,\chi}$ produces $\chi'=\mathcal N_{\Gamma_{r;G,\chi}}(\chi;G)$.
Each edge transmits at most $wp$ bits per layer in each direction.
The procedure succeeds on $(G,\chi)$ if $\chi'(u)=\chi'(v)\iff\WLtype_{G,\chi}(u)=\WLtype_{G,\chi}(v)$ for all $u,v\in V(G)$.

As discussed in Section~\ref{subsec:setup}, deterministic and zero-error randomized adaptation are equivalent.
We next establish a worst-case lower bound on the number of adaptation rounds in these settings.
\begin{restatable}{theorem}{InstDepAdaLower}
    \label{thm:instance-dependent-adaptation-lower}
    For every sufficiently large integer $n$, every integer $m \in [n,n^2]$, and every integer $D \in [8, n]$, there exists a family $\mathcal{G}_{n, m, D}$ of connected graphs, each with $\Theta(n)$ nodes, $\Theta(m)$ edges, and diameter $D$.
    For any color set $[C]$ with $C \ge \max\limits_{G \in \mathcal{G}_{n,m,D}} |V(G)|$ and $\log C = o(m)$, any deterministic or zero-error randomized decentralized parameter-adaptation framework that uses $r \ge 1$ adaptation rounds to adapt the parameters of a depth-$d$, width-$w$, precision-$p$ MPGNN, so that for every valid input instance $(G, \chi)$ with $G \in \mathcal G_{n, m, D}$ the resulting MPGNN simulates one step of the WL test on $(G,\chi)$, must satisfy $r =\Omega\left(\frac{D}{d} + \frac{m}{dwp}\right)$.
\end{restatable}
The proof is deferred to Appendix~\ref{app:proof-instance-dependent-adaptation-lower}.
The proof reduces Equality to repeated simulation of adaptation and final inference on the hard graphs of Theorem~\ref{thm:deterministic-lower}.
When $r=0$, the initial network must already be an oblivious simulator.
For $r\ge1$, the reduction also accounts for communication during adaptation and yields the stated lower bound on the number of adaptation rounds.

\section{Discussion and Limitations}
\label{sec:discussion}
Our results show that the qualitative comparison between MPGNNs and the WL test conceals a quantitative relabeling problem.
This motivates a finer-grained view of GNN expressivity under explicit resource constraints, with three implications for GNN design.
First, allowing randomization and a small failure probability can reduce the depth and communication needed to simulate one step of the WL test with a sufficiently large color space.
Second, rooted aggregation and port-aware routing yield explicit upper bounds, motivating further study of this model and its possible resource advantages over fully anonymous MPGNNs.
Finally, deterministic or zero-error randomized graph-local adaptation can require global communication even when a depth-one representation exists, so representational capacity and adaptation cost should be assessed separately.

Meanwhile, our results have several limitations.
First, some of our upper bounds in the oblivious setting are proved in a rooted, port-aware relaxation, and whether comparable bounds can be obtained in the fully anonymous unordered-neighborhood MPGNN model remains open.
Second, in the randomized setting, the intermediate regime between $C = \Theta(n)$ and $C = \Theta\left(\frac{n^2}{p_f}\right)$ is not fully characterized.
Third, our upper bounds allow large local memory and expressive local computations; incorporating explicit constraints on local state size or on implementability by standard neural architectures is an important direction for future work.
Fourth, our upper bounds apply to every valid instance under their stated assumptions, and hence under any distribution supported on such instances, whereas our lower bounds are witnessed by specially constructed hard families.
Establishing comparable average-case lower bounds for natural or real-world graph distributions remains open.
Finally, the lower bound for instance-dependent adaptation relies on a local adaptation model, and it does not preclude centralized algorithms with global access to the whole target graph.

Extending our bounds to higher-order WL tests, such as $k$-WL, requires a corresponding communication model beyond standard anonymous MPGNNs on the original graph.
A natural direction is to study relabeling on message-passing structures over tuples of vertices~\cite{morris2019weisfeiler}, with new reductions tailored to their communication pattern.

\begin{ack}
This research was supported in part by National Natural Science Foundation of China (No. 92470128, No. U2241212). We also wish to acknowledge the support provided by the fund for building world-class universities (disciplines) of Renmin University of China, by Engineering Research Center of Next-Generation Intelligent Search and Recommendation, Ministry of Education, by Intelligent Social Governance Interdisciplinary Platform, Major Innovation \& Planning Interdisciplinary Platform for the ``Double-First Class'' Initiative, Public Policy and Decision-making Research Lab, and Public Computing Cloud, Renmin University of China. 
\end{ack}


\bibliography{ref}

@article{weisfeiler1968reduction,
    title={The reduction of a graph to canonical form and the algebra which appears therein},
    author={Weisfeiler, Boris and Leman, Andrei},
    journal={nti, Series},
    volume={2},
    number={9},
    pages={12--16},
    year={1968}
}

@inproceedings{carter1977universal,
    title = {Universal classes of hash functions},
    author = {Carter, J Lawrence and Wegman, Mark N},
    booktitle = {Proceedings of the ninth annual ACM symposium on Theory of computing},
    pages = {106--112},
    year = {1977}
}

@inproceedings{yao1979some, 
    author = {Yao, Andrew Chi-Chih}, 
    title = {Some complexity questions related to distributive computing (Preliminary Report)}, 
    year = {1979}, 
    isbn = {9781450374385}, 
    publisher = {Association for Computing Machinery}, 
    address = {New York, NY, USA}, 
    booktitle = {Proceedings of the Eleventh Annual ACM Symposium on Theory of Computing}, 
    pages = {209-213}, 
    numpages = {5}, 
    location = {Atlanta, Georgia, USA}, 
    series = {STOC '79}
}

@inproceedings{linial1987distributive,
    author       = {Nathan Linial},
    title        = {Distributive Graph Algorithms-Global Solutions from Local Data},
    booktitle    = {28th Annual Symposium on Foundations of Computer Science, Los Angeles,
                  California, USA, 27-29 October 1987},
    pages        = {331--335},
    publisher    = {{IEEE} Computer Society},
    year         = {1987},
}

@article{alon1992simple,
    title = {Simple constructions of almost k-wise independent random variables},
    author = {Alon, Noga and Goldreich, Oded and H{\aa}stad, Johan and Peralta, Ren{\'e}},
    journal = {Random Structures \& Algorithms},
    volume = {3},
    number = {3},
    pages = {289--304},
    year = {1992},
    publisher = {Wiley Online Library}
}

@article{linial1992locality,
    author       = {Nathan Linial},
    title        = {Locality in Distributed Graph Algorithms},
    journal      = {{SIAM} J. Comput.},
    volume       = {21},
    number       = {1},
    pages        = {193--201},
    year         = {1992},
}

@book{kushilevitz1997communication-Ch3,
    author       = {Eyal Kushilevitz and
                  Noam Nisan},
    title        = {Communication complexity},
    publisher    = {Cambridge University Press},
    year         = {1997},
    isbn         = {978-0-521-56067-2},
	pages        = {28--41}  
}

@inproceedings{joachims1999transductive,
    author       = {Thorsten Joachims},
    editor       = {Ivan Bratko and
                  Saso Dzeroski},
    title        = {Transductive Inference for Text Classification using Support Vector
                  Machines},
    booktitle    = {Proceedings of the Sixteenth International Conference on Machine Learning
                  {(ICML} 1999), Bled, Slovenia, June 27 - 30, 1999},
    pages        = {200--209},
    publisher    = {Morgan Kaufmann},
    year         = {1999},
}

@book{peleg2000distributed, 
    author       = {Peleg, David}, 
    title        = {Distributed computing: a locality-sensitive approach}, 
    year         = {2000}, 
    isbn         = {978-0-89871-464-7}, 
    publisher    = {Society for Industrial and Applied Mathematics}, 
    address      = {USA}
}

@book{vapnik2000nature,
    author       = {Vladimir Naumovich Vapnik},
    title        = {The Nature of Statistical Learning Theory, Second Edition},
    series       = {Statistics for Engineering and Information Science},
    publisher    = {Springer},
    year         = {2000},
    isbn         = {978-0-387-98780-4},
}

@inproceedings{saglam2013communication,
    title = {On the communication complexity of sparse set disjointness and exists-equal problems},
    author = {Saglam, Mert and Tardos, G{\'a}bor},
    booktitle = {2013 IEEE 54th Annual Symposium on Foundations of Computer Science},
    pages = {678--687},
    year = {2013},
    organization = {IEEE}
}

@book{shai2014understanding,
    author       = {Shai Shalev{-}Shwartz and
                  Shai Ben{-}David},
    title        = {Understanding Machine Learning - From Theory to Algorithms},
    publisher    = {Cambridge University Press},
    year         = {2014},
    isbn         = {978-1-10-705713-5},
}

@article{brody2016certifying,
    title = {Certifying equality with limited interaction},
    author = {Brody, Joshua and Chakrabarti, Amit and Kondapally, Ranganath and Woodruff, David P and Yaroslavtsev, Grigory},
    journal = {Algorithmica},
    volume = {76},
    number = {3},
    pages = {796--845},
    year = {2016},
    publisher = {Springer}
}

@inproceedings{gilmer2017neural,
    author       = {Justin Gilmer and
                  Samuel S. Schoenholz and
                  Patrick F. Riley and
                  Oriol Vinyals and
                  George E. Dahl},
    editor       = {Doina Precup and
                  Yee Whye Teh},
    title        = {Neural Message Passing for Quantum Chemistry},
    booktitle    = {Proceedings of the 34th International Conference on Machine Learning,
                  {ICML} 2017, Sydney, NSW, Australia, 6-11 August 2017},
    series       = {Proceedings of Machine Learning Research},
    volume       = {70},
    pages        = {1263--1272},
    publisher    = {{PMLR}},
    year         = {2017},
}

@inproceedings{hamilton2017inductive,
    author       = {William L. Hamilton and
                  Zhitao Ying and
                  Jure Leskovec},
    editor       = {Isabelle Guyon and
                  Ulrike von Luxburg and
                  Samy Bengio and
                  Hanna M. Wallach and
                  Rob Fergus and
                  S. V. N. Vishwanathan and
                  Roman Garnett},
    title        = {Inductive Representation Learning on Large Graphs},
    booktitle    = {Advances in Neural Information Processing Systems 30: Annual Conference
                  on Neural Information Processing Systems 2017, December 4-9, 2017,
                  Long Beach, CA, {USA}},
    pages        = {1024--1034},
    year         = {2017},
}

@inproceedings{ying2018graph,
    author       = {Rex Ying and
                  Ruining He and
                  Kaifeng Chen and
                  Pong Eksombatchai and
                  William L. Hamilton and
                  Jure Leskovec},
    editor       = {Yike Guo and
                  Faisal Farooq},
    title        = {Graph Convolutional Neural Networks for Web-Scale Recommender Systems},
    booktitle    = {Proceedings of the 24th {ACM} {SIGKDD} International Conference on
                  Knowledge Discovery {\&} Data Mining, {KDD} 2018, London, UK,
                  August 19-23, 2018},
    pages        = {974--983},
    publisher    = {{ACM}},
    year         = {2018},
}

@book{mohri2018foundations,
    title={Foundations of machine learning},
    author={Mohri, Mehryar and Rostamizadeh, Afshin and Talwalkar, Ameet},
    year={2018},
    publisher={MIT press}
}

@inproceedings{xu2019how,
	author       = {Keyulu Xu and
					Weihua Hu and
					Jure Leskovec and
					Stefanie Jegelka},
	title        = {How Powerful are Graph Neural Networks?},
	booktitle    = {7th International Conference on Learning Representations, {ICLR} 2019,
					New Orleans, LA, USA, May 6-9, 2019},
	publisher    = {OpenReview.net},
	year         = {2019},
}

@inproceedings{morris2019weisfeiler,
	author       = {Christopher Morris and
					Martin Ritzert and
					Matthias Fey and
					William L. Hamilton and
					Jan Eric Lenssen and
					Gaurav Rattan and
					Martin Grohe},
	title        = {Weisfeiler and Leman Go Neural: Higher-Order Graph Neural Networks},
	booktitle    = {The Thirty-Third {AAAI} Conference on Artificial Intelligence, {AAAI}
					2019, The Thirty-First Innovative Applications of Artificial Intelligence
					Conference, {IAAI} 2019, The Ninth {AAAI} Symposium on Educational
					Advances in Artificial Intelligence, {EAAI} 2019, Honolulu, Hawaii,
					USA, January 27 - February 1, 2019},
	pages        = {4602--4609},
	publisher    = {{AAAI} Press},
	year         = {2019},
}

@inproceedings{wang2019neural,
    author       = {Xiang Wang and
                  Xiangnan He and
                  Meng Wang and
                  Fuli Feng and
                  Tat{-}Seng Chua},
    editor       = {Benjamin Piwowarski and
                  Max Chevalier and
                  {\'{E}}ric Gaussier and
                  Yoelle Maarek and
                  Jian{-}Yun Nie and
                  Falk Scholer},
    title        = {Neural Graph Collaborative Filtering},
    booktitle    = {Proceedings of the 42nd International {ACM} {SIGIR} Conference on
                  Research and Development in Information Retrieval, {SIGIR} 2019, Paris,
                  France, July 21-25, 2019},
    pages        = {165--174},
    publisher    = {{ACM}},
    year         = {2019},
}

@inproceedings{maron2019provably,
	author       = {Haggai Maron and
					Heli Ben{-}Hamu and
					Hadar Serviansky and
					Yaron Lipman},
	editor       = {Hanna M. Wallach and
					Hugo Larochelle and
					Alina Beygelzimer and
					Florence d'Alch{\'{e}}{-}Buc and
					Emily B. Fox and
					Roman Garnett},
	title        = {Provably Powerful Graph Networks},
	booktitle    = {Advances in Neural Information Processing Systems 32: Annual Conference
					on Neural Information Processing Systems 2019, NeurIPS 2019, December
					8-14, 2019, Vancouver, BC, Canada},
	pages        = {2153--2164},
	year         = {2019},
}

@article{xiong2019pushing,
    title = {Pushing the boundaries of molecular representation for drug discovery with the graph attention mechanism},
    author = {Xiong, Zhaoping and Wang, Dingyan and Liu, Xiaohong and Zhong, Feisheng and Wan, Xiaozhe and Li, Xutong and Li, Zhaojun and Luo, Xiaomin and Chen, Kaixian and Jiang, Hualiang and others},
    journal = {Journal of medicinal chemistry},
    volume = {63},
    number = {16},
    pages = {8749--8760},
    year = {2019},
    publisher = {ACS Publications}
}

@inproceedings{loukas2020what,
	author       = {Andreas Loukas},
	title        = {What graph neural networks cannot learn: depth vs width},
	booktitle    = {8th International Conference on Learning Representations, {ICLR} 2020,
					Addis Ababa, Ethiopia, April 26-30, 2020},
	publisher    = {OpenReview.net},
	year         = {2020},
}

@inproceedings{kiefer2020iteration,
    author       = {Sandra Kiefer and
                  Brendan D. McKay},
    editor       = {Artur Czumaj and
                  Anuj Dawar and
                  Emanuela Merelli},
    title        = {The Iteration Number of Colour Refinement},
    booktitle    = {47th International Colloquium on Automata, Languages, and Programming,
                  {ICALP} 2020, Saarbr{\"{u}}cken, Germany (Virtual Conference),
                  July 8-11, 2020},
    series       = {LIPIcs},
    pages        = {73:1--73:19},
    publisher    = {Schloss Dagstuhl - Leibniz-Zentrum f{\"{u}}r Informatik},
    year         = {2020},
}

@inproceedings{he2020simplifying,
    author       = {Xiangnan He and
                  Kuan Deng and
                  Xiang Wang and
                  Yan Li and
                  Yong{-}Dong Zhang and
                  Meng Wang},
    editor       = {Jimmy X. Huang and
                  Yi Chang and
                  Xueqi Cheng and
                  Jaap Kamps and
                  Vanessa Murdock and
                  Ji{-}Rong Wen and
                  Yiqun Liu},
    title        = {LightGCN: Simplifying and Powering Graph Convolution Network for Recommendation},
    booktitle    = {Proceedings of the 43rd International {ACM} {SIGIR} conference on
                  research and development in Information Retrieval, {SIGIR} 2020, Virtual
                  Event, China, July 25-30, 2020},
    pages        = {639--648},
    publisher    = {{ACM}},
    year         = {2020},
}

@inproceedings{alsentzer2020subgraph,
	author       = {Emily Alsentzer and
					Samuel G. Finlayson and
					Michelle M. Li and
					Marinka Zitnik},
	editor       = {Hugo Larochelle and
					Marc'Aurelio Ranzato and
					Raia Hadsell and
					Maria{-}Florina Balcan and
					Hsuan{-}Tien Lin},
	title        = {Subgraph Neural Networks},
	booktitle    = {Advances in Neural Information Processing Systems 33: Annual Conference
					on Neural Information Processing Systems 2020, NeurIPS 2020, December
					6-12, 2020, virtual},
	year         = {2020},
}

@inproceedings{huang2020communication,
    author = {Huang, Dawei and Pettie, Seth and Zhang, Yixiang and Zhang, Zhijun},
    title = {The communication complexity of set intersection and multiple equality testing},
    year = {2020},
    publisher = {Society for Industrial and Applied Mathematics},
    address = {USA},
    booktitle = {Proceedings of the Thirty-First Annual ACM-SIAM Symposium on Discrete Algorithms},
    pages = {1715–1732},
    numpages = {18},
    location = {Salt Lake City, Utah},
    series = {SODA '20}
}

@inproceedings{cotta2021reconstruction,
    author       = {Leonardo Cotta and
                  Christopher Morris and
                  Bruno Ribeiro},
    editor       = {Marc'Aurelio Ranzato and
                  Alina Beygelzimer and
                  Yann N. Dauphin and
                  Percy Liang and
                  Jennifer Wortman Vaughan},
    title        = {Reconstruction for Powerful Graph Representations},
    booktitle    = {Advances in Neural Information Processing Systems 34: Annual Conference
                  on Neural Information Processing Systems 2021, NeurIPS 2021, December
                  6-14, 2021, virtual},
    pages        = {1713--1726},
    year         = {2021},
}

@inproceedings{papp2021drop,
    author       = {P{\'{a}}l Andr{\'{a}}s Papp and
                  Karolis Martinkus and
                  Lukas Faber and
                  Roger Wattenhofer},
    editor       = {Marc'Aurelio Ranzato and
                  Alina Beygelzimer and
                  Yann N. Dauphin and
                  Percy Liang and
                  Jennifer Wortman Vaughan},
    title        = {DropGNN: Random Dropouts Increase the Expressiveness of Graph Neural
                  Networks},
    booktitle    = {Advances in Neural Information Processing Systems 34: Annual Conference
                  on Neural Information Processing Systems 2021, NeurIPS 2021, December
                  6-14, 2021, virtual},
    pages        = {21997--22009},
    year         = {2021},
}

@inproceedings{geerts2022expressiveness,
    author       = {Floris Geerts and
                  Juan L. Reutter},
    title        = {Expressiveness and Approximation Properties of Graph Neural Networks},
    booktitle    = {The Tenth International Conference on Learning Representations, {ICLR}
                  2022, Virtual Event, April 25-29, 2022},
    publisher    = {OpenReview.net},
    year         = {2022},
}

@inproceedings{aamand2022exponentially,
    author       = {Anders Aamand and
                  Justin Y. Chen and
                  Piotr Indyk and
                  Shyam Narayanan and
                  Ronitt Rubinfeld and
                  Nicholas Schiefer and
                  Sandeep Silwal and
                  Tal Wagner},
    editor       = {Sanmi Koyejo and
                  S. Mohamed and
                  A. Agarwal and
                  Danielle Belgrave and
                  K. Cho and
                  A. Oh},
    title        = {Exponentially Improving the Complexity of Simulating the Weisfeiler-Lehman
                  Test with Graph Neural Networks},
    booktitle    = {Advances in Neural Information Processing Systems 35: Annual Conference
                  on Neural Information Processing Systems 2022, NeurIPS 2022, New Orleans,
                  LA, USA, November 28 - December 9, 2022},
    year         = {2022},
}

@inproceedings{feng2022how,
	author       = {Jiarui Feng and
					Yixin Chen and
					Fuhai Li and
					Anindya Sarkar and
					Muhan Zhang},
	editor       = {Sanmi Koyejo and
					S. Mohamed and
					A. Agarwal and
					Danielle Belgrave and
					K. Cho and
					A. Oh},
	title        = {How Powerful are K-hop Message Passing Graph Neural Networks},
	booktitle    = {Advances in Neural Information Processing Systems 35: Annual Conference
					on Neural Information Processing Systems 2022, NeurIPS 2022, New Orleans,
					LA, USA, November 28 - December 9, 2022},
	year         = {2022},
}

@inproceedings{frasca2022understanding,
	author       = {Fabrizio Frasca and
					Beatrice Bevilacqua and
					Michael M. Bronstein and
					Haggai Maron},
	editor       = {Sanmi Koyejo and
					S. Mohamed and
					A. Agarwal and
					Danielle Belgrave and
					K. Cho and
					A. Oh},
	title        = {Understanding and Extending Subgraph GNNs by Rethinking Their Symmetries},
	booktitle    = {Advances in Neural Information Processing Systems 35: Annual Conference
					on Neural Information Processing Systems 2022, NeurIPS 2022, New Orleans,
					LA, USA, November 28 - December 9, 2022},
	year         = {2022},
}

@misc{ghaffari2022distributed,
    author       = {Mohsen Ghaffari},
    title        = {Distributed Graph Algorithms},
    howpublished = {Course Notes for Distributed Algorithms},
    institution  = {EECS, MIT},
    year         = {2022},
    note         = {Accessed: September, 2024},
}

@inproceedings{zhang2023rethinking,
	author = {Bohang Zhang and Shengjie Luo and Liwei Wang and Di He},
	title        = {Rethinking the Expressive Power of GNNs via Graph Biconnectivity},
	booktitle    = {The Eleventh International Conference on Learning Representations,
					{ICLR} 2023, Kigali, Rwanda, May 1-5, 2023},
	publisher    = {OpenReview.net},
	year         = {2023},
}

@inproceedings{zhou2023from,
	author       = {Cai Zhou and
					Xiyuan Wang and
					Muhan Zhang},
	editor       = {Andreas Krause and
					Emma Brunskill and
					Kyunghyun Cho and
					Barbara Engelhardt and
					Sivan Sabato and
					Jonathan Scarlett},
	title        = {From Relational Pooling to Subgraph GNNs: {A} Universal Framework
					for More Expressive Graph Neural Networks},
	booktitle    = {International Conference on Machine Learning, {ICML} 2023, 23-29 July
					2023, Honolulu, Hawaii, {USA}},
	series       = {Proceedings of Machine Learning Research},
	volume       = {202},
	pages        = {42742--42768},
	publisher    = {{PMLR}},
	year         = {2023},
}

@inproceedings{feng2023extending,
	author       = {Jiarui Feng and
					Lecheng Kong and
					Hao Liu and
					Dacheng Tao and
					Fuhai Li and
					Muhan Zhang and
					Yixin Chen},
	editor       = {Alice Oh and
					Tristan Naumann and
					Amir Globerson and
					Kate Saenko and
					Moritz Hardt and
					Sergey Levine},
	title        = {Extending the Design Space of Graph Neural Networks by Rethinking Folklore Weisfeiler-Lehman},
	booktitle    = {Advances in Neural Information Processing Systems 36: Annual Conference
					on Neural Information Processing Systems 2023, NeurIPS 2023, New Orleans,
					LA, USA, December 10 - 16, 2023},
	year         = {2023},
}

@inproceedings{zhou2023distance,
	author       = {Junru Zhou and
					Jiarui Feng and
					Xiyuan Wang and
					Muhan Zhang},
	editor       = {Alice Oh and
					Tristan Naumann and
					Amir Globerson and
					Kate Saenko and
					Moritz Hardt and
					Sergey Levine},
	title        = {Distance-Restricted Folklore Weisfeiler-Leman GNNs with Provable Cycle Counting Power},
	booktitle    = {Advances in Neural Information Processing Systems 36: Annual Conference
					on Neural Information Processing Systems 2023, NeurIPS 2023, New Orleans,
					LA, USA, December 10 - 16, 2023},
	year         = {2023},
}

@article{lam2023learning,
    title = {Learning skillful medium-range global weather forecasting},
    author = {Lam, Remi and Sanchez-Gonzalez, Alvaro and Willson, Matthew and Wirnsberger, Peter and Fortunato, Meire and Alet, Ferran and Ravuri, Suman and Ewalds, Timo and Eaton-Rosen, Zach and Hu, Weihua and others},
    journal = {Science},
    volume = {382},
    number = {6677},
    pages = {1416--1421},
    year = {2023},
    publisher = {American Association for the Advancement of Science}
}

@article{liang2023bagnn,
    title = {{BA-GNN}: Behavior-aware graph neural network for session-based recommendation},
    author = {Liang, Yongquan and Song, Qiuyu and Zhao, Zhongying and Zhou, Hui and Gong, Maoguo},
    journal = {Frontiers of Computer Science},
    volume = {17},
    number = {6},
    pages = {176613},
    year = {2023}
}

@article{chen2024how,
    title = {How graph convolutions amplify popularity bias for recommendation?},
    author = {Chen, Jiajia and Wu, Jiancan and Chen, Jiawei and Xin, Xin and Li, Yong and He, Xiangnan},
    journal = {Frontiers of Computer Science},
    volume = {18},
    number = {5},
    pages = {185603},
    year = {2024}
}

@inproceedings{jacobs2025communication,
    author = {Jacobs, Dale and Jeang, John and Podolskii, Vladimir and Prior, Morgan and Volkovich, Ilya},
    title =	{Communication Complexity of Equality and Error-Correcting Codes},
    booktitle = {45th IARCS Annual Conference on Foundations of Software Technology and Theoretical Computer Science (FSTTCS 2025)},
    pages = {37:1--37:19},
    series = {Leibniz International Proceedings in Informatics (LIPIcs)},
    ISBN = {978-3-95977-406-2},
    ISSN = {1868-8969},
    year = {2025},
    volume = {360},
    editor = {Aiswarya, C. and Mehta, Ruta and Roy, Subhajit},
}

@article{zhang2025metagin,
    title = {{MetaGIN}: a lightweight framework for molecular property prediction},
    author = {Zhang, Xuan and Chen, Cheng and Wang, Xiaoting and Jiang, Haitao and Zhao, Wei and Cui, Xuefeng},
    journal = {Frontiers of Computer Science},
    volume = {19},
    number = {5},
    pages = {195912},
    year = {2025}
}

@article{han2025survey,
    title = {A survey of geometric graph neural networks: data structures, models and applications},
    author = {Han, Jiaqi and Cen, Jiacheng and Wu, Liming and Li, Zongzhao and Kong, Xiangzhe and Jiao, Rui and Yu, Ziyang and Xu, Tingyang and Wu, Fandi and Wang, Zihe and Xu, Hongteng and Wei, Zhewei and Zhao, Deli and Liu, Yang and Rong, Yu and Huang, Wenbing},
    journal = {Frontiers of Computer Science},
    volume = {19},
    number = {11},
    pages = {1911375},
    year = {2025}
}

@article{yi2026graph,
    title = {Graph contrastive learning view construction methods in recommender systems: a survey},
    author = {Yi, Zhihang and Wang, Hairong and Chen, Fangping and Xu, Zhaojing and Yang, Jianling},
    journal = {Frontiers of Computer Science},
    volume = {20},
    number = {7},
    pages = {2007338},
    year = {2026}
}

@article{tao2026exact,
    title = {Exact acceleration of subgraph graph neural networks by eliminating computation redundancy},
    author = {Tao, Qian and Wang, Xiyuan and Zhang, Muhan and Hu, Shuxian and Yu, Wenyuan and Zhou, Jingren},
    journal = {Frontiers of Computer Science},
    volume = {20},
    number = {12},
    pages = {2012626},
    year = {2026}
}
\bibliographystyle{plain}

\newpage
\appendix
\section{Communication Complexity Background and Proof of Theorem~\ref{thm:deterministic-lower}}
\label{app:proof-deterministic-lower}
\subsection{Introduction to Communication Complexity}
The proof of Theorem~\ref{thm:deterministic-lower} relies on tools from communication complexity, so we briefly review the necessary background to make the required concepts self-contained.
Communication complexity was first introduced by Yao~\cite{yao1979some}.
In two-party communication complexity, two parties, Alice and Bob, wish to compute a function $f : X \times Y \to Z$, where $X$ and $Y$ are their input domains, respectively. 
They agree on a strategy beforehand but are separated before receiving their inputs $(x, y) \in X \times Y$. 
They then exchange messages to compute $f(x, y)$, with the goal of minimizing the total number of bits exchanged. 

In deterministic communication, the strategy is fixed, and the minimum number of bits required to compute $f$ in this setting is known as the deterministic communication complexity, denoted by $D(f)$. 
Similarly, in randomized communication, where Alice and Bob can use random bits and a two-sided error of $\epsilon$ is allowed, the minimum number of bits required is the randomized communication complexity.
If the randomness is private, it is denoted by $R_{\epsilon}^{\prv}(f)$, and if it is public, it is denoted by $R_{\epsilon}^{\pub}(f)$.

The Equality (EQ) problem between two $n$-bit strings, denoted by $\func{EQ}_n : \set{0, 1}^n \times \set{0, 1}^n \to \set{0, 1}$, is defined as
\begin{equation*}
    \func{EQ}_n(\bs{x}, \bs{y}) =
    \begin{cases}
    1, & \bs{x} = \bs{y}, \\
    0, & \text{otherwise}.
    \end{cases}
\end{equation*}
It is one of the most well-studied problems in two-party communication complexity.
We summarize its communication complexity under different settings in Table \ref{tab:communication-complexity-EQ} below.
\begin{table}[H]
    \centering
    \caption{Communication complexity of the $\func{EQ}_n$ function under different settings.}
    \label{tab:communication-complexity-EQ}
    \begin{threeparttable}
        \begin{tabular}{c|c|cc|cc}
            \toprule
            \multirow{3}{*}{Function} & \multirow{2}{*}{Deterministic} & \multicolumn{4}{c}{Randomized} \\
            & & \multicolumn{2}{c|}{Private Coin} & \multicolumn{2}{c}{Public Coin} \\
            & $D(\cdot)$ & $R^{\prv}_0(\cdot)$ & $R^{\prv}_{1/3}(\cdot)$ & $R_0^{\pub}(\cdot)$ & $R^{\pub}_{1/3}(\cdot)$ \\
            \midrule
            $\func{EQ}_n$ & $\Theta(n)$\tnote{$\dagger$} & $\Theta(n)$\tnote{$\dagger$} & $\Theta(\log n)$\tnote{$\dagger$} & $\Theta(n)$\tnote{*} & $\Theta(1)$\tnote{$\dagger$} \\
            \bottomrule
        \end{tabular}
        \begin{tablenotes}
            \item[$\dagger$] The proofs can be found in~\cite{kushilevitz1997communication-Ch3}.
            \item[*] This follows from the bound $R_0^{\prv}(f) = O\left(R_0^{\pub}(f) + \log n\right)$ (see Exercise~3.15 in~\cite{kushilevitz1997communication-Ch3}).
        \end{tablenotes}
    \end{threeparttable}
\end{table}

\subsection{Proof of Theorem~\ref{thm:deterministic-lower}}
We now prove Theorem \ref{thm:deterministic-lower}.
\DeterLower*
\begin{proof}
    At a high level, we prove the result by reducing Equality to the problem of simulating one step of the WL test.
    To be concrete, for every sufficiently large positive integer $n$, every $m \in [n, n^2]$, every $D \in [8, n]$, and for every $\func{EQ}_m$ instance $(\bs{a}, \bs{b}) \in \set{0, 1}^m \times \set{0, 1}^m$, we construct a connected uncolored graph $G$ with $\Theta(n)$ nodes, $\Theta(m)$ edges, and diameter exactly $D$, together with two distinguished nodes $x^{(A)}$ and $x^{(B)}$ satisfying $\func{dist}_{G}\left(x^{(A)}, x^{(B)}\right) = \Theta(D)$, such that $\chi^{(3)}_{G}\left(x^{(A)}\right) = \chi^{(3)}_{G}\left(x^{(B)}\right) \iff \bs{a} = \bs{b}$.
    Thus, by chaining together three oblivious MPGNNs that simulate one step of the WL test, and then comparing the colors of two nodes via communication, we obtain a communication protocol for solving the Equality problem.
    The communication-complexity lower bound for Equality then yields the desired lower bound on the relevant parameters of the MPGNN.

\paragraph{Construction.}
    The entire construction is divided into four steps. 
    In the first step, we construct a fixed basic graph which is partitioned between Alice ($A$) and Bob ($B$). 
    In the second step, given an input instance $(\bs{a}, \bs{b})$ of $\func{EQ}_m$, we add edges to the basic graph and assign virtual colors to certain nodes, thereby obtaining a virtually colored hard instance.
    In the third step, we further augment the above virtually colored hard instance by adding nodes and edges to decolorize it, thus obtaining an uncolored hard instance.
    In the last step, we add filler nodes so that the resulting graph satisfies the desired size requirement.

    \textbf{\textit{Step One: Fixed Basic Graph.}}
    Let $r \coloneqq \left\lceil \sqrt m\right\rceil$, and $q \coloneqq \left\lceil \frac{m}{r}\right\rceil$ be two positive integers.
    For each side $\sigma \in \set{A, B}$, create nodes $x^{(\sigma)}$, $w^{(\sigma)}_1, \cdots, w^{(\sigma)}_q$, $u^{(\sigma)}_1, \cdots, u^{(\sigma)}_r$, and $v^{(\sigma)}_1, \cdots, v^{(\sigma)}_r$; and add the fixed edges $\set{x^{(\sigma)}, w^{(\sigma)}_i}$ for $i\in[q]$, $\set{x^{(\sigma)}, u^{(\sigma)}_j}$ and $\set{x^{(\sigma)},v^{(\sigma)}_j}$ for all $j\in[r]$. 
    Then, connect the two sides by a path, for example by adding nodes $c_1, \cdots, c_{D-7}$ and edges $\set{x^{(A)}, c_1}, \set{c_1, c_2}, \cdots, \set{c_{D - 7}, x^{(B)}}$\footnote{The distinguished nodes are used only by the reduction to specify which two output colors Alice and Bob compare. They are not given as additional node features to the MPGNN or to the WL test.}.
    
    Next, Alice holds the induced subgraph on the vertices with superscript $(A)$, while Bob holds the induced subgraph on the vertices with superscript $(B)$. 
    We partition the connecting path at an arbitrary edge, assigning all vertices on one side of this edge to Alice and all vertices on the other side to Bob.
    Thus the Alice-Bob cut contains exactly one graph edge.
    Figure~\ref{fig:deterministic-lower-1} below gives an illustration of the fixed basic graph. 
    The edges $\set{x^{(\sigma)}, u^{(\sigma)}_j}$ and $\set{x^{(\sigma)}, v^{(\sigma)}_j}$ are omitted, and the connecting path is simplified.

    \begin{figure*}[ht]
		\centering
		\begin{tikzpicture}[scale = 0.75, transform shape]
			\tikzset{every node/.style = {minimum size = 1.0cm, thick, inner sep = 0pt}}		
			\node[circle, draw] (xA) at (-.8, 0) {$x^{(A)}$};
			\foreach \i in {1, ..., 5}
			{
				\ifnum \i = 4
					\node (wA\i) at (-3, 1.2 * 3 - 1.2 * \i) {$\dots$};
				\else
					\node[circle, draw] (wA\i) at (-3, 1.2 * 3 - 1.2 * \i)
					{
						\ifnum \i < 4 $w^{(A)}_{\i}$ \else $w^{(A)}_{q}$ \fi
					};
				\fi
			}
			\foreach \i in {1, 2, 3, 5}
			{
				\draw[thick] (xA) -- (wA\i);
			}
		    
			\foreach \i in {1, ..., 7} 
			{
				\ifnum \i = 5
					\node (uA\i) at (-3 - 0.8 * \i, 6.5 - 0.8 * \i) {$\dots$};
				\else
					\node[circle, draw] (uA\i) at (-3 - 0.8 * \i, 6.5 - 0.8 * \i)
					{
						\ifnum \i < 5 $u^{(A)}_{\i}$ \else
						\ifnum \i = 6 $u^{(A)}_{r - 1}$ \else $u^{(A)}_{r}$ \fi \fi
					};
				\fi
    		}
		
			\foreach \i in {1, ..., 7} 
			{

				\ifnum \i = 5
					\node (vA\i) at (-3 - 0.8 * \i, -6.5 + 0.8 * \i) {$\dots$};
				\else
					\node[circle, draw] (vA\i) at (-3 - 0.8 * \i, -6.5 + 0.8 * \i)
					{
						\ifnum \i < 5 $v^{(A)}_{\i}$ \else
						\ifnum \i = 6 $v^{(A)}_{r - 1}$ \else $v^{(A)}_{r}$ \fi \fi
					};
				\fi
			}
		
			\node[circle, draw] (xB) at (.8, 0) {$x^{(B)}$};
			\foreach \i in {1, ..., 5}
			{
				\ifnum \i = 4
					\node (wB\i) at (3, 1.2 * 3 - 1.2 * \i) {$\dots$};
				\else
					\node[circle, draw] (wB\i) at (3, 1.2 * 3 - 1.2 * \i)
					{
						\ifnum \i < 4 $w^{(B)}_{\i}$ \else $w^{(B)}_{q}$ \fi
					};
				\fi
			}
			\foreach \i in {1, 2, 3, 5}
			{
				\draw[thick] (xB) -- (wB\i);
			}
		
			\foreach \i in {1, ..., 7}
			{
				\ifnum \i = 5
					\node (uB\i) at (3 + 0.8 * \i, 6.5 - 0.8 * \i) {$\dots$};
				\else
					\node[circle, draw,] (uB\i) at (3 + 0.8 * \i, 6.5 - 0.8 * \i)
					{
						\ifnum \i < 5 $u^{(B)}_{\i}$ \else
						\ifnum \i = 6 $u^{(B)}_{r - 1}$ \else $u^{(B)}_{r}$ \fi \fi
					};
				\fi
			}
		
			\foreach \i in {1, ..., 7}
			{
				\ifnum \i = 5
					\node (vB\i) at (3 + 0.8 * \i, -6.5 + 0.8 * \i) {$\dots$};
				\else
					\node[circle, draw] (vB\i) at (3 + 0.8 * \i, -6.5 + 0.8 * \i)
					{%
						\ifnum \i < 5 $v^{(B)}_{\i}$ \else
						\ifnum \i = 6 $v^{(B)}_{r - 1}$ \else $v^{(B)}_{r}$ \fi \fi
					};
				\fi
			}
		
			\draw[thick, decorate, decoration={snake, amplitude=1mm, segment length = 2mm}] (xA) -- (xB);
			\draw [dashed, very thick] (0, 7) -- (0, -7);

			\node[anchor = south east] at (-1, 6.) {\Large Alice};
			\node[anchor = south west] at (1, 6.) {\Large Bob};
		\end{tikzpicture}
		\caption{The constructed basic graph. The edges $\set{x^{(\sigma)}, u^{(\sigma)}_j}$ and $\set{x^{(\sigma)}, v^{(\sigma)}_j}$ are omitted, and the connecting path is simplified.}
		\label{fig:deterministic-lower-1}
	\end{figure*}

    \textbf{\textit{Step Two: Virtually Colored Hard Instance.}}
    We now encode an instance of $\func{EQ}_m$ and assign virtual colors to selected nodes.
    Fix an injection $\tau:[m]\to [q] \times[r]$ by $\tau(k) = \left(\left\lceil\frac{k}{r}\right\rceil, (k-1)\bmod r + 1\right)$. 
    Given inputs $\bs{a}, \bs{b}\in \set{0, 1}^m$, define bits $a_{i, j}$ and $b_{i, j}$ as follows.
    If $(i, j) = \tau(k)$ for some $k \in [m]$, set $a_{i, j} \coloneqq \bs a_k$ and $b_{i, j} \coloneqq \bs b_k$. 
    Otherwise set $a_{i, j} = b_{i, j} \coloneqq 0$.
    For every $(i, j)\in [q] \times [r]$, Alice adds the edge $\set{w^{(A)}_i, u^{(A)}_j}$ if $a_{i, j} = 0$, and the edge $\set{w^{(A)}_i, v^{(A)}_j}$ if $a_{i, j} = 1$. 
    Bob adds the analogous edge $\set{w^{(B)}_i, u^{(B)}_j}$ if $b_{i, j} = 0$, and the edge $\set{w^{(B)}_i, v^{(B)}_j}$ if $b_{i, j} = 1$. 

    We next assign some virtual colors to certain nodes.
    Define virtual colors by $\kappa\left(w^{(\sigma)}_i\right) = i$, $\kappa\left(u^{(\sigma)}_j\right) = q + j$, and $\kappa\left(v^{(\sigma)}_j\right) = q + r + j$.
    The nodes $x^{(A)}$, $x^{(B)}$ and connecting nodes $c_{i}$ do not need virtual colors.
    Figure~\ref{fig:deterministic-lower-2} below gives an illustration of the virtually colored hard instance. 

    \begin{figure*}[ht]
		\centering
		\begin{tikzpicture}[scale = 0.75, transform shape]
			\tikzset{every node/.style = {minimum size = 1.0cm, thick, inner sep = 0pt}}
			\definecolor{deepu}{rgb}{0.0, 0.5, 0.5}
			\definecolor{lightu}{rgb}{0.8, 1.0, 1.0}

			\definecolor{deepv}{rgb}{1.0, 0.55, 0.0}
			\definecolor{lightv}{rgb}{1.0, 0.9, 0.8}

			\definecolor{deepw}{rgb}{0.0, 0.0, 0.5}
			\definecolor{lightw}{rgb}{0.7, 0.7, 1.0}
		
			\node[circle, draw] (xA) at (-.8, 0) {$x^{(A)}$};
			\foreach \i in {1, ..., 5}
			{
				\pgfmathparse{100 * (\i - 1)/(7 - 1)}
				\colorlet{nodecolor}{lightw!\pgfmathresult!deepw}

				\ifnum \i = 4
					\node (wA\i) at (-3, 1.2 * 3 - 1.2 * \i) {$\dots$};
				\else
					\node[circle, draw, fill = nodecolor, text = white] (wA\i) at (-3, 1.2 * 3 - 1.2 * \i)
					{
						\ifnum \i < 4 $w^{(A)}_{\i}$ \else $w^{(A)}_{q}$ \fi
					};
				\fi
			}
			\foreach \i in {1, 2, 3, 5}
			{
				\draw[thick] (xA) -- (wA\i);
			}
		    
			\foreach \i in {1, ..., 7} 
			{
				\pgfmathparse{100 * (\i - 1)/(7 - 1)}
				\colorlet{nodecolor}{lightu!\pgfmathresult!deepu}

				\ifnum \i = 5
					\node (uA\i) at (-3 - 0.8 * \i, 6.5 - 0.8 * \i) {$\dots$};
				\else
					\node[circle, draw, fill = nodecolor] (uA\i) at (-3 - 0.8 * \i, 6.5 - 0.8 * \i)
					{
						\ifnum \i < 5 $u^{(A)}_{\i}$ \else
						\ifnum \i = 6 $u^{(A)}_{r - 1}$ \else $u^{(A)}_{r}$ \fi \fi
					};
				\fi
    		}
		
			\foreach \i in {1, ..., 7} 
			{
				\pgfmathparse{100 * (\i - 1)/(7 - 1)}
				\colorlet{nodecolor}{lightv!\pgfmathresult!deepv}

				\ifnum \i = 5
					\node (vA\i) at (-3 - 0.8 * \i, -6.5 + 0.8 * \i) {$\dots$};
				\else
					\node[circle, draw, fill = nodecolor] (vA\i) at (-3 - 0.8 * \i, -6.5 + 0.8 * \i)
					{
						\ifnum \i < 5 $v^{(A)}_{\i}$ \else
						\ifnum \i = 6 $v^{(A)}_{r - 1}$ \else $v^{(A)}_{r}$ \fi \fi
					};
				\fi
			}
		
			\node[circle, draw] (xB) at (.8, 0) {$x^{(B)}$};
			\foreach \i in {1, ..., 5}
			{
				\pgfmathparse{100 * (\i - 1)/(7 - 1)}
				\colorlet{nodecolor}{lightw!\pgfmathresult!deepw}
				\ifnum \i = 4
					\node (wB\i) at (3, 1.2 * 3 - 1.2 * \i) {$\dots$};
				\else
					\node[circle, draw, fill = nodecolor, text = white] (wB\i) at (3, 1.2 * 3 - 1.2 * \i)
					{
						\ifnum \i < 4 $w^{(B)}_{\i}$ \else $w^{(B)}_{q}$ \fi
					};
				\fi
			}
			\foreach \i in {1, 2, 3, 5}
			{
				\draw[thick] (xB) -- (wB\i);
			}
		
			\foreach \i in {1, ..., 7}
			{
				\pgfmathparse{100 * (\i - 1)/(7 - 1)}
				\colorlet{nodecolor}{lightu!\pgfmathresult!deepu}
				\ifnum \i = 5
					\node (uB\i) at (3 + 0.8 * \i, 6.5 - 0.8 * \i) {$\dots$};
				\else
					\node[circle, draw, fill = nodecolor] (uB\i) at (3 + 0.8 * \i, 6.5 - 0.8 * \i)
					{
						\ifnum \i < 5 $u^{(B)}_{\i}$ \else
						\ifnum \i = 6 $u^{(B)}_{r - 1}$ \else $u^{(B)}_{r}$ \fi \fi
					};
				\fi
			}
		
			\foreach \i in {1, ..., 7}
			{
				\pgfmathparse{100 * (\i - 1)/(7 - 1)}
				\colorlet{nodecolor}{lightv!\pgfmathresult!deepv}
				\ifnum \i = 5
					\node (vB\i) at (3 + 0.8 * \i, -6.5 + 0.8 * \i) {$\dots$};
				\else
					\node[circle, draw, fill = nodecolor] (vB\i) at (3 + 0.8 * \i, -6.5 + 0.8 * \i)
					{%
						\ifnum \i < 5 $v^{(B)}_{\i}$ \else
						\ifnum \i = 6 $v^{(B)}_{r - 1}$ \else $v^{(B)}_{r}$ \fi \fi
					};
				\fi
			}
		
			\draw[thick, decorate, decoration={snake, amplitude=1mm, segment length = 2mm}] (xA) -- (xB);
			\draw [dashed, very thick] (0, 7) -- (0, -7);

			\node[anchor = south east] at (-1, 6.) {\Large Alice};
			\node[anchor = south west] at (1, 6.) {\Large Bob};

			\draw[dashed] (wA1) -- node[below left] {$\bs{a}_1 = 0$?} (uA1);
			\draw[dashed] (wA1) -- node[above left] {$\bs{a}_4 = 1$?} (vA4);
		\end{tikzpicture}
		\caption{The constructed virtually colored hard instance graph. The edges $\set{x^{(\sigma)}, u^{(\sigma)}_j}$ and $\set{x^{(\sigma)}, v^{(\sigma)}_j}$ are omitted, and the connecting path is simplified.}
		\label{fig:deterministic-lower-2}
	\end{figure*}

    \textbf{\textit{Step Three: Decolorization.}}
    We perform the decolorization process separately on Alice's side and Bob's side, so no new edge crosses the Alice-Bob cut. 
    The idea is to add auxiliary nodes and connect certain nodes to them in order to adjust their degrees, thereby encoding the desired colors.
    
    Let $Z^{(\sigma)} \coloneqq \set{w^{(\sigma)}_i: i\in[q]} \cup \set{u^{(\sigma)}_j, v^{(\sigma)}_j : j \in [r]}$ be the colored nodes on each side, and $K \coloneqq q + 2r$ be the number of used virtual colors.
    Let $\Delta_0$ be an upper bound on the degree of every node in $Z^{(\sigma)}$ before decolorization.
    Indeed, before decolorization, each $w_i^{(\sigma)}$ has degree $r + 1$, while each $u_j^{(\sigma)}$ or $v_j^{(\sigma)}$ has degree at most $q + 1$. 
    Hence $\Delta_0 = q + r + 1$ is a valid upper bound.
    We will assign a distinct target degree $T_k \coloneqq \Delta_0 + k$ for each virtual color $k \in [K]$, and also define the padding target degree $T_\perp \coloneqq \Delta_0 + K + 1$.
    Then all $T_k$ are pairwise distinct, all are larger than $\Delta_0$, and none equals $T_\perp$.

    Now fix one side $\sigma \in \set{A, B}$. 
    For each node $z \in Z^{(\sigma)}$, let $d_0(z)$ be its current degree in the virtually colored hard instance, and define $p_z \coloneqq T_{\kappa(z)} - d_0(z)$. 
    Since $T_{\kappa(z)} > \Delta_0\ge d_0(z)$, we have $p_z > 0$. 
    Let $M^{(\sigma)}\coloneqq \max_{z\in Z^{(\sigma)}} p_z > 0$. 
    Create padding nodes $P^{(\sigma)}_1, \cdots, P^{(\sigma)}_{M^{(\sigma)}}$. 
    Since $M^{(\sigma)} > 0$, there exists at least one padding node on each side.
    For every $z\in Z^{(\sigma)}$, connect $z$ to the first $p_z$ padding nodes, namely $P^{(\sigma)}_1, \cdots, P^{(\sigma)}_{p_z}$. 
    Hence the resulting degree of $z$ becomes $d_0(z) + p_z = T_{\kappa(z)}$.
    Also, every padding node $P_i^{(\sigma)}$ is adjacent to some node in $Z^{(\sigma)}$ by the definition of $M^{(\sigma)}$.

    Finally, we further modify the degrees of the padding nodes to ensure that they have the same first-round WL color.
    Let $s^{(\sigma)}_i \coloneqq \left|\set{z\in Z^{(\sigma)}:p_z\ge i}\right|$ be the number of neighbors currently adjacent to $P^{(\sigma)}_i$. 
    Since $s^{(\sigma)}_i\le \left|Z^{(\sigma)}\right| = K < T_\perp$, we can regularize the degrees of the padding nodes to $T_{\perp}$. 
    Let $R^{(\sigma)} \coloneqq \max_{i\in [M^{(\sigma)}]} (T_\perp - s^{(\sigma)}_i)$. 
    Create regularizer nodes $Q^{(\sigma)}_1, \cdots,Q^{(\sigma)}_{R^{(\sigma)}}$. 
    Since $R^{(\sigma)} > 0$, there exists at least one regularizer node on each side.
    For each $i\in [M^{(\sigma)}]$, connect $P^{(\sigma)}_i$ to $Q^{(\sigma)}_1, \cdots, Q^{(\sigma)}_{T_\perp - s^{(\sigma)}_i}$. 
    Then every padding node has degree exactly $s^{(\sigma)}_i + (T_\perp - s^{(\sigma)}_i) = T_\perp$. 
    Moreover, every regularizer node is adjacent to some padding node by the definition of $R^{(\sigma)}$, so this step does not create isolated nodes. 
    Figure~\ref{fig:deterministic-lower-3} below gives an illustration of the decolorization process. 

    \begin{figure}[ht]
        \centering
        \begin{tikzpicture}[
            vtx/.style={circle, draw, fill=white, inner sep=0pt, minimum size=6mm},
            reg/.style={circle, draw, fill=white, inner sep=0pt, minimum size=6mm},
            pad/.style={circle, draw, fill=white, inner sep=0pt, minimum size=6mm},
            lbl/.style={draw=none, fill=none, inner sep=0pt},
            every node/.style={font=\small}
        ]
            \def\yG{0}
            \def\yR{2}
            \def\yP{4.}
            \def\yQ{6.}
        
            \node[vtx, fill=teal!25] (x) at (0,\yG) {$x$};
            \node[vtx, fill=blue!25] (y) at (4,\yG) {$y$};
        
            \draw[very thick] (x) -- (y)
                node[midway, below=4pt] {$\set{x, y}$};
        
            \node[lbl, anchor=east] at ($(x)+(-.45,0)$) {colored nodes};
        
            \node[reg] (rx) at (0,\yR) {$x$};
            \node[reg] (ry) at (4,\yR) {$y$};
        
            \draw[very thick] (rx) -- (ry);
        
            \draw[very thick, -Latex, shorten <=2pt, shorten >=4pt]
                (x.north) -- (rx.south);
            \draw[very thick, -Latex, shorten <=2pt, shorten >=4pt]
                (y.north) -- (ry.south);
        
            \node[lbl, anchor=east] at ($(rx)+(-.5,0)$) {decolorized nodes};
        
            \draw[rounded corners = 8pt, fill = orange!8, draw = orange!35, thick]
                (-2, 3.5) rectangle (7.,4.5);
        
            \node[pad] (P1)   at (-1.5,\yP) {\tiny $P_1$};
            \node[pad] (P2)   at (-0.5,\yP) {\tiny $P_2$};
            \node[lbl] (e1)   at ( 0.5,\yP) {$\cdots$};
            \node[pad] (Ppu)  at ( 1.5,\yP) {\tiny $P_{p_x}$};
            \node[pad] (Ppu1) at ( 2.5,\yP) {\tiny $P_{p_x+1}$};
            \node[lbl] (e2)   at ( 3.5,\yP) {$\cdots$};
            \node[pad] (Ppv)  at ( 4.5,\yP) {\tiny $P_{p_y}$};
            \node[lbl] (e3)   at ( 5.5,\yP) {$\cdots$};
            \node[pad] (PM)   at ( 6.5,\yP) {\tiny $P_M$};
        
            \node[lbl, anchor=east] at (-2.2, \yP) {padding nodes};
    
            \draw[rounded corners = 8pt, fill = purple!8, draw = purple!35, thick]
                (-2, 5.5) rectangle (2.,6.5);
    
            \node[reg] (Q1)   at (-1.5,\yQ) {\tiny $Q_1$};
            \node[reg] (Q2)   at (-0.5,\yQ) {\tiny $Q_2$};
            \node[lbl] (qe1)  at ( 0.5,\yQ) {$\cdots$};
            \node[reg] (QR)  at ( 1.5,\yQ) {\tiny $Q_{R}$};
    
            \node[lbl, anchor=east] at (-2.2, \yQ) {regularizer nodes};
    
            \foreach \p in {P1,P2,Ppu}{
                \draw[teal!70!black, thick, opacity=.85]
                    (rx) -- (\p);
            }
        
            \foreach \p in {P1,P2,Ppu,Ppu1,Ppv}{
                \draw[blue!70!black, thick, opacity=.75]
                    (ry) -- (\p);
            }
    
            \foreach \p in {P1,P2,Ppu}{
                \draw[gray!70, thick, opacity=.45](\p.north) -- (Q1.south);
            }
            \foreach \p in {Ppu1,Ppv}{
                \foreach \q in {Q1,Q2}{
                    \draw[gray!70, thick, opacity=.45]
                        (\p.north) -- (\q.south);
                }
            }
            \node[lbl, align=left, anchor=west] at (4.5, 2.75)
                {$x$ connects to first $p_x$ padding nodes \\
                 $y$ connects to first $p_y$ padding nodes};
    
            \node[lbl, align=left, anchor = west] at (3.5, 5.25)
                {$P_i$ connects to first $T_\perp - s_i$ regularizer nodes \\
                 so every padding node has degree $T_\perp$};
        \end{tikzpicture}
        \caption{An illustration of the decolorization process.}
        \label{fig:deterministic-lower-3}
    \end{figure}
    
    \textbf{\textit{Step Four: Filler Nodes.}}
    Finally, to make the number of nodes $\Theta(n)$, we add $3n$ new leaves adjacent to $x^{(A)}$ and $3n$ new leaves adjacent to $x^{(B)}$.
    We denote the constructed graph by $G_{(n,m,D;\bs{a},\bs{b})}$, and for notational convenience, we sometimes simply write it as $G$.
        
    \paragraph{Analysis of Graph Parameters.}
    Next, we verify that the constructed graph has the desired number of nodes, number of edges, and diameter.
    \begin{itemize}[leftmargin=*]
        \item \textbf{\textit{Number of Nodes.}}
        For each side $\sigma \in \set{A, B}$, there is one $x^{(\sigma)}$-node, $q = \Theta(\sqrt{m})$ $w^{(\sigma)}$-nodes, $r = \Theta(\sqrt{m})$ $u^{(\sigma)}$-nodes, $r = \Theta(\sqrt{m})$ $v^{(\sigma)}$-nodes, and $\Theta(n)$ filler nodes attached to $x^{(\sigma)}$.
        In addition, there are at most $M^{(\sigma)} \le \max_{z \in Z^{(\sigma)}} T_{\kappa(z)} \le \Delta_0 + K + 1 = \Theta(\sqrt{m})$ padding nodes, and at most $R^{(\sigma)} \le T_{\perp} = \Delta_0 + K + 1 = \Theta(\sqrt{m})$ regularizer nodes.
        Together with the $\Theta(D)$ connecting-path nodes, the whole graph has $\Theta(\sqrt{m} + n + D) = \Theta(n)$ nodes.

        \item \textbf{\textit{Number of Edges.}}
        For each side $\sigma \in \set{A, B}$, there are $q = \Theta(\sqrt{m})$ edges between the $x^{(\sigma)}$-node and the $w^{(\sigma)}$-nodes, $2r = \Theta(\sqrt{m})$ edges between the $x^{(\sigma)}$-node and the $u^{(\sigma)}$- and $v^{(\sigma)}$-nodes, and $\Theta(n)$ edges between the $x^{(\sigma)}$-node and the filler nodes.
        For each $w^{(\sigma)}$-node, there are $r = \Theta(\sqrt{m})$ edges connecting it to the $u^{(\sigma)}$-nodes or the $v^{(\sigma)}$-nodes.
        For every node $z \in Z^{(\sigma)}$, there are $p_z \le T_{\kappa(z)} \le \Delta_0 + K + 1 = O(\sqrt{m})$ edges connecting it to padding nodes.
        Hence, this part contributes $O((q+2r)\sqrt{m}) = O(m)$ edges.
        For every padding node, there are at most $T_{\perp}$ edges connecting it to regularizer nodes.
        Hence, this part contributes $O(M^{(\sigma)}\sqrt{m}) = O(m)$ edges.
        Together with the $\Theta(D)$ edges on the path connecting the two sides, the whole graph has $\Theta(m + n + D) = \Theta(m)$ edges.

        \item \textbf{\textit{Diameter.}}
        Let $L = D-6$ be the length of the connecting path between $x^{(A)}$ and $x^{(B)}$.
        Every non-path vertex attached to the $A$-side is within distance at most $3$ from $x^{(A)}$, and every non-path vertex attached to the $B$-side is within distance at most $3$ from $x^{(B)}$. 
        Moreover, every path vertex lies on the $x^{(A)}$-$x^{(B)}$ path of length $L$.
        Thus any two vertices attached to opposite side gadgets have distance at most $3 + L + 3 = D$. 
        Two non-path vertices attached to the same side gadget have distance at most $6\le D$. 
        Distances involving path vertices are also at most $D$, since every path vertex is within distance at most $L$ from one of $x^{(A)}$ and $x^{(B)}$, and every non-path vertex is within distance at most $3$ from the corresponding endpoint. 
        Hence the diameter is at most $D$.

        On each side there exists at least one regularizer node at distance exactly $3$ from the corresponding distinguished node. 
        Since the only connection between the two side gadgets is the connecting path, a regularizer node on the $A$-side and a regularizer node on the $B$-side are at distance exactly $3 + L + 3 = D$. 
        Therefore, the diameter of $G$ is exactly $D$.
    \end{itemize}
    
    \paragraph{Analysis of WL Colors.}
    Let $\chi^{(t)}$ denote the coloring after $t$ rounds of the WL test starting from the uniform coloring.
    \begin{itemize}[leftmargin=*]
        \item \textbf{\textit{Round $1$.}} 
        At round $1$, WL colors are determined by the degrees of the nodes.
        \begin{itemize}[leftmargin=*]
            \item (1) For any $z \in Z \coloneqq Z^{(A)} \cup Z^{(B)}$, we have $\deg_G(z) = T_{\kappa(z)}$.
            Thus, for any $z,z' \in Z$, $\chi^{(1)}(z) = \chi^{(1)}(z')$ if and only if $\kappa(z) = \kappa(z')$, since the target degrees $\Delta_0 + 1 = T_1 < T_2 < \cdots < T_K = \Delta_0 + K < T_\perp = \Delta_0 + K + 1$ are distinct.
            In other words, $\chi^{(1)}\left(w_i^{(A)}\right) = \chi^{(1)}\left(w_i^{(B)}\right)$ for each $i \in [q]$, and $\chi^{(1)}\left(u_j^{(A)}\right) = \chi^{(1)}\left(u_j^{(B)}\right)$ and $\chi^{(1)}\left(v_j^{(A)}\right) = \chi^{(1)}\left(v_j^{(B)}\right)$ for each $j \in [r]$; moreover, these round-$1$ WL colors are pairwise distinct.

            \item (2) In addition, all padding nodes have the same round-$1$ WL color, which is different from that of the nodes in $Z^{(A)} \cup Z^{(B)}$, since they all have the distinct target degree $T_\perp$.

            \item (3) Moreover, since $\deg_G\left(x^{(A)}\right) = \deg_G\left(x^{(B)}\right) = q + 2r + 3n + 1$, we have $\chi^{(1)}\left(x^{(A)}\right) = \chi^{(1)}\left(x^{(B)}\right)$.
            Note that $m \le n^2$, so $\sqrt{m} \le n$, and hence $r = \left\lceil \sqrt{m} \right\rceil \le n$.
            Also, since $r \ge \sqrt{m}$, we have $\frac{m}{r} \le \sqrt{m} \le r$, and therefore $q = \left\lceil \frac{m}{r} \right\rceil \le r$.
            Consequently, $q + r \le 2n$.
            Meanwhile, $\deg_G\left(x^{(\sigma)}\right) - T_\perp = q + 2r + 3n + 1 - (2q + 3r + 2) = 3n - (q + r) - 1$.
            Using $q + r \le 2n$, we obtain $3n - (q + r) - 1 \ge 3n - 2n - 1 = n - 1 > 0$ for all $n \ge 2$.
            Therefore, the degrees of the nodes $x^{(A)}$ and $x^{(B)}$ are greater than $T_\perp$, so their round-$1$ color is also different from all the colors mentioned above.

            \item (4) Finally, all filler nodes adjacent to the $x$-nodes have the same round-$1$ WL color, and all connecting-path nodes have the same round-$1$ WL color.
            The filler-node color and the connecting-path-node color are different from the round-$1$ colors of the nodes in $Z^{(A)}\cup Z^{(B)}$, the padding nodes, and the distinguished nodes $x^{(A)}, x^{(B)}$.
            Indeed, filler nodes have degree $1$, connecting-path nodes have degree $2$, nodes in $Z^{(A)}\cup Z^{(B)}$ have degrees in $\set{T_1,\cdots,T_K}$, padding nodes have degree $T_\perp$, and $x^{(A)},x^{(B)}$ have degree larger than $T_\perp$.
            The colors of the regularizer nodes are irrelevant for the subsequent argument.
        \end{itemize}
        
        \item \textbf{\textit{Round $2$.}}
        The round-$2$ WL color of a node is determined by its round-$1$ WL color together with the multiset of round-$1$ WL colors of its neighbors.
        \begin{itemize}[leftmargin=*]
            \item (1) Fix a side $\sigma \in \set{A, B}$ and consider the nodes $u^{(\sigma)}_j$ and $v^{(\sigma)}_j$.
            For each $j \in [r]$, define $S^{(A)}_{0, j} \coloneqq \set{i \in [q]: a_{i, j} = 0}$, $S^{(A)}_{1, j} \coloneqq \set{i \in [q]: a_{i, j} = 1}$, $S^{(B)}_{0, j} \coloneqq \set{i \in [q]: b_{i, j} = 0}$, and $S^{(B)}_{1, j} \coloneqq \set{i \in [q]: b_{i, j} = 1}$.
            The neighbors of $u^{(\sigma)}_j$ consist of the node $x^{(\sigma)}$, the nodes
            $w^{(\sigma)}_i$ with $i\in S^{(\sigma)}_{0, j}$, and $T_{q + j} - \left(1 + \left|S^{(\sigma)}_{0, j}\right|\right)$ padding nodes. 
            Similarly, the neighbors of $v^{(\sigma)}_j$ consist of the node $x^{(\sigma)}$, the nodes $w^{(\sigma)}_i$ with $i\in S^{(\sigma)}_{1, j}$, and $T_{q + r + j} - \left(1+\left|S^{(\sigma)}_{1,j}\right|\right)$ padding nodes.
            By the round-$1$ analysis, we have $\chi^{(1)}\left(u^{(A)}_j\right)=\chi^{(1)}\left(u^{(B)}_j\right)$, $\chi^{(1)}\left(v^{(A)}_j\right)=\chi^{(1)}\left(v^{(B)}_j\right)$.
            Moreover, the round-$1$ colors of the nodes $w_1, \cdots, w_q$ are pairwise distinct and agree across the two sides for the same index, while the padding color is common and is distinct from all colors of the $w$-nodes.
            Therefore the multiset of round-$1$ colors seen by $u^{(A)}_j$ equals the one seen by $u^{(B)}_j$ if and only if $S^{(A)}_{0, j} = S^{(B)}_{0, j}$.
            Since the entries are binary, this is equivalent to $(a_{1, j}, \cdots, a_{q, j}) = (b_{1, j}, \cdots, b_{q, j})$.
            Hence $\chi^{(2)}\left(u^{(A)}_j\right) = \chi^{(2)}\left(u^{(B)}_j\right) \iff (a_{1, j}, \cdots, a_{q, j}) = (b_{1, j}, \cdots, b_{q, j})$.
            The same argument applied to the sets $S^{(\sigma)}_{1, j}$ gives $\chi^{(2)}\left(v^{(A)}_j\right) = \chi^{(2)}\left(v^{(B)}_j\right) \iff (a_{1, j}, \cdots, a_{q, j}) = (b_{1, j}, \cdots, b_{q, j})$.
            
            \item (2) Next, consider a node $w^{(\sigma)}_i$.
            Its neighbors consist of one node $x^{(\sigma)}$, one of $u^{(\sigma)}_j$ or $v^{(\sigma)}_j$ for each $j \in [r]$, and $T_i - r - 1$ padding nodes.
            By our analysis of the round-$1$ colors, we have $\chi^{(1)}\left(w_i^{(A)}\right) = \chi^{(1)}\left(w_i^{(B)}\right)$, and $\chi^{(1)}\left(x^{(A)}\right) = \chi^{(1)}\left(x^{(B)}\right)$, and the padding neighbors also contribute the same color with the same multiplicity.
            Thus, we have $\chi^{(2)}\left(w^{(A)}_i\right) = \chi^{(2)}\left(w^{(B)}_i\right)$ if and only if $(a_{i, 1}, \cdots, a_{i,r}) = (b_{i,1}, \cdots, b_{i, r})$.

            \item (3) Then, consider the node $x^{(\sigma)}$.
            It has $q$ neighbors $w_i^{(\sigma)}$, $r$ neighbors $u_i^{(\sigma)}$, $r$ neighbors $v_i^{(\sigma)}$, $3n$ filler-node neighbors, and one connecting-path node neighbor.
            By our analysis of the round-$1$ colors, we have $\chi^{(1)}\left(x^{(A)}\right) = \chi^{(1)}\left(x^{(B)}\right)$, $\chi^{(1)}\left(w_i^{(A)}\right) = \chi^{(1)}\left(w_i^{(B)}\right)$ for each $i \in [q]$, and $\chi^{(1)}\left(u_j^{(A)}\right) = \chi^{(1)}\left(u_j^{(B)}\right)$ and $\chi^{(1)}\left(v_j^{(A)}\right) = \chi^{(1)}\left(v_j^{(B)}\right)$ for each $j \in [r]$.
            Moreover, all filler nodes adjacent to the $x$-nodes have the same round-$1$ WL color, and all connecting-path nodes have the same round-$1$ WL color.
            Therefore, $\chi^{(2)}\left(x^{(A)}\right) = \chi^{(2)}\left(x^{(B)}\right)$.

            \item (4) Moreover, all filler nodes adjacent to the $x$-nodes have the same round-$2$ WL color, since they all have the same round-$1$ color and the round-$1$ color of their neighbor $x^{(\sigma)}$ is identical; and connecting-path nodes $c_1$ and $c_{D - 7}$ have the same round-$2$ WL color for a similar reason.
            We do not need to consider the round-$2$ WL colors of the regularizer nodes, the padding nodes, or the other connecting-path nodes.
        \end{itemize}
        
        \item \textbf{\textit{Round $3$.}}
        The round-$3$ WL color of a node is determined by its round-$2$ WL color together with the multiset of round-$2$ WL colors of its neighbors.
        For node $x^{(\sigma)}$, we have already shown that $\chi^{(2)}\left(x^{(A)}\right) = \chi^{(2)}\left(x^{(B)}\right)$.
        The filler-node contribution and the connecting-path contribution are also identical on the two sides.
        It remains to compare the round-$2$ colors contributed by the nodes in $Z^{(\sigma)}$.
        
        \begin{itemize}[leftmargin=*]
            \item If $\bs{a} = \bs{b}$, then every row and every column of the matrices $(a_{i, j})$ and $(b_{i, j})$ agree.
            By the round-$2$ analysis above, we then have $\chi^{(2)}\left(w^{(A)}_i\right) = \chi^{(2)}\left(w^{(B)}_i\right)$ for all $i \in [q]$, and $\chi^{(2)}\left(u^{(A)}_j\right) = \chi^{(2)}\left(u^{(B)}_j\right)$, $\chi^{(2)}\left(v^{(A)}_j\right) = \chi^{(2)}\left(v^{(B)}_j\right)$ for all $j \in [r]$.
            Hence the entire multiset of round-$2$ colors seen by $x^{(A)}$ is the same as the one seen by $x^{(B)}$, and therefore $\chi^{(3)}\left(x^{(A)}\right) = \chi^{(3)}\left(x^{(B)}\right)$.

            \item Conversely, suppose $\bs{a} \neq \bs{b}$. 
            Since $\tau$ is injective and all unused entries are set to $0$ on both sides, there exists some $(i, j)\in[q] \times [r]$ such that $a_{i, j}\neq b_{i, j}$. 
            Then the $i$-th rows differ, so by the round-$2$ analysis of the $w$-nodes, $\chi^{(2)}\left(w_i^{(A)}\right)\neq \chi^{(2)}\left(w_i^{(B)}\right)$.
            Because WL colors always refine previous-round colors, any node whose round-$2$ color equals $\chi^{(2)}\left(w_i^{(A)}\right)$ must have the same round-$1$ color as $w_i^{(A)}$.
            Among the neighbors of $x^{(B)}$, the only node with this round-$1$ color is $w_i^{(B)}$.
            However, $\chi^{(2)}\left(w_i^{(B)}\right)\neq\chi^{(2)}\left(w_i^{(A)}\right)$. 
            Therefore the color $\chi^{(2)}\left(w_i^{(A)}\right)$ appears with multiplicity one in the round-$2$ neighbor-color multiset of $x^{(A)}$, but with multiplicity zero in the corresponding multiset of $x^{(B)}$. 
            Hence the two round-$2$ neighbor-color multisets are different, and therefore $\chi^{(3)}\left(x^{(A)}\right)\neq \chi^{(3)}\left(x^{(B)}\right)$.
        \end{itemize}
    \end{itemize}
    Thus, from the above analysis, we obtain $\chi^{(3)}\left(x^{(A)}\right) = \chi^{(3)}\left(x^{(B)}\right) \iff \bs{a} = \bs{b}$.

\paragraph{Lower Bound.}
Now suppose that we have an oblivious MPGNN with depth $d$, width $w$, and precision $p$ that can simulate one step of the WL test with zero error.
We now derive the lower bound from two perspectives:
\begin{itemize}[leftmargin=*]
    \item \textbf{\textit{Communication Complexity.}}
    We can use this MPGNN to obtain a communication protocol for $\func{EQ}_m$.
    By composing three copies of the MPGNN, we obtain an MPGNN of depth $3d$, width $w$, and precision $p$ that simulates three steps of the WL test.
    We then apply it to the hard-instance graph constructed above and compare the colors of $x^{(A)}$ and $x^{(B)}$.

    Alice and Bob simulate the MPGNN on their respective sides.
    The Alice-Bob cut contains exactly one graph edge.
    Therefore, in each message-passing layer, only $O(wp)$ bits need to cross the cut.
    Over $3d$ layers, the total communication is $O(dwp)$ bits.
    Finally, the simulator outputs labels in the color set $[C]$, so the colors of $x^{(A)}$ and $x^{(B)}$ can be compared using $O(\log C)$ bits of communication.
    Therefore, we obtain an $O(dwp+\log C)$-bit communication protocol for $\func{EQ}_m$.

    Since the deterministic and randomized zero-error communication complexities of $\func{EQ}_m$ satisfy $D(\func{EQ}_m) = R_0(\func{EQ}_m) = \Omega(m)$ bits, we get $O\left(dwp + \log C\right) = \Omega(m)$.
    When $\log C = o(m)$, this implies $dwp = \Omega(m)$, and hence $d = \Omega\left(\frac{m}{wp}\right)$.

    \item \textbf{\textit{Locality.}}
    Let $\ell \coloneqq \func{dist}_G\left(x^{(A)}, x^{(B)}\right) = \Theta(D)$.
    We claim that the composed MPGNN of depth $3d$ must satisfy $3d \ge \frac{\ell}{2} = \Theta(D)$.
    
    Suppose for contradiction that $3d < \frac{\ell}{2}$. 
    Then, by locality of message passing, the output color of $x^{(A)}$ after the composed MPGNN is independent of Bob's input $\bs{b}$, and the output color of $x^{(B)}$ is independent of Alice's input $\bs{a}$. 
    Therefore, there exist functions $F, H:\set{0, 1}^m \to [C]$ such that the output color of $x^{(A)}$ is $F(\bs{a})$ and the output color of $x^{(B)}$ is $H(\bs{b})$\footnote{For zero-error randomized simulation, fix any seed in the support as justified in Section~\ref{subsec:setup} and apply the deterministic argument below.}.
    
    Since the composed MPGNN simulates three steps of the WL test with zero error, and the hard instance satisfies $\chi_G^{(3)}\left(x^{(A)}\right) = \chi_G^{(3)}\left(x^{(B)}\right) \iff \bs{a} = \bs{b}$, we must have $F(\bs{a}) = H(\bs{b}) \iff \bs{a} = \bs{b}$, for all $\bs{a}, \bs{b}\in \set{0, 1}^m$.
    In particular, $F(\bs{a}) = H(\bs{a})$ for every $\bs{a} \in \set{0, 1}^m$. 
    Moreover, $F$ must be injective: if $F(\bs{a}) = F(\bs{a}')$, then $F(\bs{a}') = F(\bs{a}) = H(\bs{a})$, which implies $\bs{a}' = \bs{a}$ by the equivalence above. 
    Thus $|\func{Im}(F)| = 2^m$, and hence $C\ge 2^m$.
    This contradicts the assumption $\log C = o(m)$ for sufficiently large $m$. 
    Therefore $3d\ge \frac{\ell}{2} = \Theta(D)$, and hence $d = \Omega(D)$.
\end{itemize}
Combining the two bounds gives $d = \Omega\left(\max\set{D, \frac{m}{wp}}\right) = \Omega\left(D + \frac{m}{wp}\right)$.
\end{proof}

\section{CONGEST, Our Rooted Port-Aware MPGNN, and Proof of Theorem~\ref{thm:deterministic-upper}}
\label{app:proof-deterministic-upper}
\subsection{CONGEST Model and Our Rooted Port-Aware MPGNN Model}
Before proving Theorem~\ref{thm:deterministic-upper}, we first introduce the CONGEST model, the rooted, port-aware MPGNN model used in our construction, and several CONGEST algorithms that can be transferred to this model.

The CONGEST model was introduced to study the number of communication rounds required to solve distributed computing problems in bandwidth-limited networks~\cite{linial1987distributive, linial1992locality, peleg2000distributed, ghaffari2022distributed}. 
Due to the similarity between the update process in the CONGEST model (Equation~\eqref{eqn:CONGEST}) and the message-passing mechanism of MPGNNs (Equation~\eqref{eqn:message-passing}), prior work~\cite{loukas2020what} has proposed using the CONGEST model to characterize MPGNNs and derived lower bounds for tasks such as cycle detection.
The definition of the CONGEST model is:
\begin{definition}[CONGEST Model]
The CONGEST model is defined on an $n$-node graph $G = (V = [n], E)$. 
Initially, each node knows the total number of nodes $n$, its own initial features, and its unique identifier, which lies in $[\poly(n)]$ or $[n]$.
In each communication round, a node receives an $O(w)$-word message\footnote{In the standard setting, message sizes are restricted to $O(\log n)$ bits, that is, $O(1)$ words when $p = \Theta(\log n)$. In~\cite{loukas2020what}, this constraint is relaxed to $O(w)$ words, that is, $O(w \log n)$ bits when $p = \Theta(\log n)$, to match the width of MPGNNs.} from each neighbor, performs local computation based on its current knowledge, and then sends an $O(w)$-word message to each neighbor, where the messages sent to different neighbors may be different.
This process can be described by the following update rule:
\begin{equation}
\label{eqn:CONGEST}
s_{u}^{(\ell + 1)} = \UPD^{(\ell)}\left(s_u^{(\ell)}, \mset{\left(v, \MSG^{(\ell)}\left(s_v^{(\ell)}, u, v\right)\right) : v \in N(u)}\right),
\end{equation}
where $s_u^{(\ell)}$ denotes the internal state (which may not necessarily be a vector) of node $u$. 
\end{definition}
In~\cite{loukas2020what}, the authors view the CONGEST model as a computational model for MPGNNs, so that MPGNN constructions can be obtained by designing CONGEST algorithms.

However, the use of unique identifiers in CONGEST is a rather strong assumption. 
While it greatly enhances the computational power of the model, it also breaks anonymity. 
Moreover, when the update and message functions depend directly on node IDs, it may destroy the permutation invariance or permutation equivariance property commonly required by MPGNNs, and may also make it difficult to generalize to graphs with previously unseen IDs.
For this reason, we do not adopt CONGEST directly as the model for studying upper bounds.
Instead, in order to obtain concrete constructions while keeping nodes from being uniquely identified and preserving the permutation invariance or permutation equivariance of the model as much as possible, we work with the following weaker rooted, port-aware model.

\begin{definition}[Rooted, Port-Aware MPGNN Model]
Our rooted, port-aware MPGNN model is defined on a graph $G=(V,E)$ with one designated root node $r$. 
Initially, each node $u$ receives its input feature vector $\bs{x}_u$ and a root bit $\bs{b}_u \in \set{0, 1}$, where $\bs{b}_u = 1$ if and only if $u = r$. 
Nodes do not have other unique identifiers.
Each node $u$ has an arbitrary but fixed local ordering of its incident ports, denoted by $P(u) = [\deg(u)]$.
These port numbers are purely local and are not globally consistent across endpoints of an edge.

At layer $\ell$, node $u$ maintains an arbitrary finite node state $\bs{s}_u^{(\ell)}$ and arbitrary finite port states $\left(\bs{s}_{u, p}^{(\ell)}\right)_{p\in P(u)}$. 
Through each port $p$, it receives a message $\bs{m}_{u, p}^{(\ell)}$ of at most $wp$ bits. 
It then applies a shared local transition rule
\begin{equation}
    \left(
    \bs{s}_u^{(\ell + 1)}, \left(\bs{s}_{u, p}^{(\ell + 1)}\right)_{p\in P(u)}, \left(\bs{m}_{u, p}^{(\ell + 1)}\right)_{p\in P(u)}
    \right) 
    = \UPD^{(\ell)}\left(\bs{s}_u^{(\ell)}, \bs{b}_u, \left(p, \bs{s}_{u, p}^{(\ell)}, \bs{m}_{u, p}^{(\ell)}\right)_{p\in P(u)}
    \right),
\end{equation}
and sends $\bs{m}_{u, p}^{(\ell + 1)}$ through port $p$. 
The same transition rule $\UPD^{(\ell)}$ is used at all nodes. 
Thus the model has root information and local port awareness, but no unique node identifiers.
\end{definition}

\subsection{Three CONGEST Algorithms}
We then present several classical CONGEST algorithms and point out that they can still be implemented in the weaker rooted, port-aware MPGNN model.

First, a spanning tree rooted at a node $u$ can be constructed using the \alg{Flood} algorithm in the CONGEST model.
\begin{lemma}[cf.\ Chapter 3 of~\cite{peleg2000distributed}, \alg{Flood} Algorithm]
There exists a CONGEST algorithm in which a designated node $u \in V$ can construct a spanning tree $T$ rooted at $u$ with depth $\func{depth}(T) = \max_v\set{d_G(u, v)}$ in $\max_v \set{d_G(u, v)} = O(D)$ rounds, where $D$ is the diameter of the graph.
\end{lemma}
The idea behind the \alg{Flood} algorithm is straightforward. 
Initially, the source node $u$ sends a special token to all its neighbors.
Each node, upon receiving the token for the first time, stores it and forwards it to its neighbors.
If a node receives the token again, it discards it and does nothing.
Note that the round complexity of \alg{Flood} is $O(D)$ and it does not depend on the CONGEST parameters $w$ (the number of words per edge per round) or $p$ (the number of bits per word), since the token can be encoded using a constant number of bits.

Additionally, the following lemmas describe the ability of the CONGEST model to broadcast and collect messages to and from a designated node.
\begin{lemma}[cf.\ Chapter 4 of~\cite{peleg2000distributed}, \alg{Downcast} Algorithm]
There exists a CONGEST algorithm where given $M$ messages (of $\Theta(\log n)$ bits) stored at a designated node $u \in V$, and a spanning tree $T$ rooted at $u$, the messages can be broadcast to other nodes in $O(\func{depth}(T) + M)$ rounds.
\end{lemma}
\begin{lemma}[cf.\ Chapter 4 of~\cite{peleg2000distributed}, \alg{Upcast} Algorithm]
There exists a CONGEST algorithm where given $M$ messages (of $\Theta(\log n)$ bits) stored at different nodes and a spanning tree $T$ rooted at $u$, the messages can be collected at node $u$ in $O(\func{depth}(T) + M)$ rounds.
\end{lemma}
The key idea behind these two algorithms is to use a queue as a buffer to pipeline messages along the spanning tree.
Take \alg{Downcast} as an example.
Each node maintains a queue as a buffer.
In each round, the root sends one message to each of its children. 
Any other node enqueues the message received from its parent on the spanning tree, then dequeues the message at the head of the queue and forwards it to all of its children.
Under this strategy, the root releases one message per round, so it terminates after $M$ rounds.
Moreover, a message that leaves the root in round $r$ reaches its destination by round at most $r + \func{depth}(T)$.
The \alg{Upcast} algorithm proceeds in the opposite direction and can be analyzed similarly.

It is important to note that the conclusions for the \alg{Downcast} and \alg{Upcast} algorithms above are derived under the standard CONGEST model, where each edge can transmit only $O(1)$ messages of size $\Theta(\log n)$ bits per communication round.
If the total encoded information to be broadcast or collected is $B$ bits, and we relax the per-edge restriction to allow the transmission of $w$ messages of $p$ bits each per round, then the round complexities of both algorithms reduce to $O\left(\func{depth}(T) + \frac{B}{wp}\right)$ by streaming the encoded information as detailed below.

\subsection{Implementations in the Rooted, Port-Aware MPGNN Model}
In fact, none of the above three algorithms uses the unique identifiers assumed in the CONGEST model. 
They only require that each node knows whether it is the root, and that each edge can decide whether to continue transmitting information according to whether it is a tree edge. 
Therefore, it is intuitive that all three CONGEST algorithms can also be implemented in our rooted, port-aware MPGNN model. 
The following lemma formalizes this claim.
\begin{lemma}[Rooted, Port-Aware MPGNN Implementations of \alg{Flood}, \alg{Downcast}, and \alg{Upcast}]
\label{lem:rooted-port-aware-congest-subroutines}
Let $G = (V, E)$ be an $n$-node graph with a designated root $r$ and diameter $D$. 
Let $\func{depth}(T)$ denote the depth of a rooted tree $T$, namely the maximum tree distance from the root to any node. 
Suppose that in each MPGNN layer, each direction on an edge can transmit at most $w$ words, and each word consists of $p$ bits, so each edge can transmit at most $wp$ bits per layer in each direction. 
Then the following CONGEST algorithms admit rooted, port-aware MPGNN implementations.
\begin{enumerate}[leftmargin = *]
    \item \textbf{\alg{Flood}.} There exists a rooted, port-aware MPGNN of network depth $d_{\alg{Flood}} = O(D)$ that constructs a spanning tree $T$ rooted at $r$. 
    The constructed tree satisfies $\func{depth}(T) = \max_{v\in V} d_G(r, v) = O(D)$.

    \item \textbf{\alg{Downcast}.} Suppose a spanning tree $T$ rooted at $r$ is already encoded in the port states, and suppose the root stores a payload of total length $B$ bits. 
    There exists a rooted, port-aware MPGNN of network depth $d_{\alg{Downcast}} = O\left(\func{depth}(T) + \left\lceil\frac{B}{wp}\right\rceil\right)$ that broadcasts the payload from $r$ to all nodes.

    \item \textbf{\alg{Upcast}.} Suppose a spanning tree $T$ rooted at $r$ is already encoded in the port states, and suppose the nodes jointly store payloads of total encoded length $B$ bits, including the stream delimiters described below.
    There exists a rooted, port-aware MPGNN of network depth $d_{\alg{Upcast}} = O\left(\func{depth}(T) + \left\lceil \frac{B}{wp}\right\rceil\right)$ that collects all payloads at the root $r$.
\end{enumerate}
\end{lemma}
\begin{proof}
We describe how to simulate the three algorithms in the rooted, port-aware MPGNN model. 

For every node $u$ and every port $p\in P(u)$, the port state stores a role $\rho_{u, p}\in\set{\mathsf{unknown}, \mathsf{parent}, \mathsf{child}, \mathsf{offtree}}$\footnote{In practice, this role information is encoded using two bits; we write it symbolically here only for ease of exposition.}. 
By the definition of the rooted, port-aware model, node $u$ receives the port-indexed collection of incoming messages and port states. 
Hence it can distinguish its incident ports, compare their local indices, and update the corresponding port states jointly. 
We write $\iota_{u, p}\in [\deg(u)]$ for the local index of port $p$. 
This index is purely local and is not a unique node identifier.

We use a fixed alphabet containing the data symbols $0,1$ and the control symbols needed below.
Each symbol has a constant-length binary encoding, so a logical communication round transmitting up to $wp$ symbols per edge in each direction can be implemented by a constant number of MPGNN layers.
Unused slots contain an idle symbol, which receivers discard.
All rounds in the following implementations are logical communication rounds.

\begin{itemize}[leftmargin=*]
    \item For \alg{Flood}, initially only the root $r$ is active. 
    The root sends a $\mathsf{FLOOD}$ token through all ports. 
    When a non-root node $u$ receives a $\mathsf{FLOOD}$ token for the first time, it chooses one receiving port as its parent port. 
    If several tokens arrive in the same round, $u$ chooses the port with the smallest local index $\iota_{u, p}$.
    It marks this port as $\mathsf{parent}$, sends a $\mathsf{ACK}$ token back through this parent port, and forwards the $\mathsf{FLOOD}$ token through all other ports. 
    Whenever a node receives a $\mathsf{ACK}$ token through a port, it marks that port as $\mathsf{child}$. 
    After the flooding and acknowledgement phase has run for $O(D)$ rounds, any port that has not been marked as $\mathsf{parent}$ or $\mathsf{child}$ is marked as $\mathsf{offtree}$.
    
    A node at graph distance $h$ from the root first receives the $\mathsf{FLOOD}$ token after $h$ rounds, and therefore chooses as parent a neighbor at distance $h - 1$ from the root.
    Hence the chosen parent pointers form a spanning tree rooted at $r$, and the tree depth is $\func{depth}(T) = \max_{v\in V} d_G(r, v)$. 
    Since every node is reached within at most $D$ rounds, and the $\mathsf{ACK}$ confirmations only add a constant number of rounds, the required network depth is $d_{\alg{Flood}} = O(D)$.

    \item For \alg{Downcast}, assume that a rooted spanning tree $T$ has already been encoded in the port states.
    The root splits its $B$-bit payload into $L = \left\lceil \frac{B}{wp}\right\rceil$ packets, each containing at most $wp$ data symbols.
    Each node maintains a FIFO queue. 
    Initially, the root queue contains all $L$ packets, while other queues are empty. 
    In each round, the root sends the head packet of its queue to all child ports.
    Every non-root node accepts data packets only from its parent port, enqueues the received packet, and if its queue is nonempty, forwards the head packet to all child ports.
    This behavior is implemented by the port-indexed local transition rule: the node ignores data packets received on non-parent ports and produces data packets only on child ports\footnote{For example, one can add a gating vector to each port and use the Hadamard product to control message passing.}.

    The resulting procedure is the standard pipelined broadcast on a rooted tree. 
    The root releases one packet per round, and a packet needs at most $\func{depth}(T)$ additional rounds to reach every node.
    Therefore all $L$ packets reach all nodes after $O\left(\func{depth}(T) + L\right) = O\left(\func{depth}(T) + \left\lceil\frac{B}{wp}\right\rceil\right)$ rounds.
    Hence the network depth is $d_{\alg{Downcast}} = O\left(\func{depth}(T) + \left\lceil\frac{B}{wp}\right\rceil\right)$.

    \item For \alg{Upcast}, again assume that the rooted spanning tree $T$ has already been encoded in the port states.
    The idea is to assemble each subtree's payloads into an ordered stream and pipeline it toward the root.
    Delimiters let the root recover the individual payloads.

    Each node $u$ forms a subtree stream encoding the payloads of all nodes in the subtree of $T$ rooted at $u$.
    It maintains a separate buffer for each child port.
    The stream begins with the payload at $u$ and a separator, followed by the subtree streams of its children in local port order and a final end marker.
    Separators and end markers use binary encodings distinct from those of data and idle symbols.
    The subtree streams received from children are forwarded verbatim, so this encoding has only constant-factor expansion of the payload bits and $O(n)$ additional framing bits.
    Each node forwards the longest available prefix of its stream, up to $wp$ symbols per round.
    The subtree stream from the current child is forwarded as its symbols arrive, until its matching outer end marker is reached, while streams from later child ports are buffered.
    When this subtree stream ends, the node fills any remaining slots in the same round from the next part of its stream.
    The root uses the same rule to store the combined stream and then parses it to recover all payloads.

    Let $L_u$ be the number of symbols in the subtree stream of $u$, and let $h_u$ be its subtree height.
    We show by induction on $t$ that after $t$ rounds, each node $u$ has forwarded (or stored at the root) at least $\min\{L_u,wp\max\{0,t-h_u\}\}$ symbols.
    The claim is immediate for $t\le h_u$.
    For $t>h_u$, the induction hypothesis ensures that each child $v$ has supplied at least $\min\{L_v,wp(t-h_u)\}$ symbols by round $t-1$.
    Thus either the subtree streams of all children are complete, or the first incomplete stream already supplies $wp(t-h_u)$ symbols, so a prefix of length at least $\min\{L_u,wp(t-h_u)\}$ is available.
    Together with the previous round's bound and the capacity to forward $wp$ symbols, this proves the claim.
    Since $L_r=\Theta(B)$, the number of logical rounds, and hence the MPGNN depth, is $O\left(\func{depth}(T)+\left\lceil B/(wp)\right\rceil\right)$.
\end{itemize}
Therefore the three CONGEST algorithms have rooted, port-aware MPGNN implementations with the claimed network depths.
\end{proof}

\subsection{Proof of Theorem~\ref{thm:deterministic-upper}}
With these tools in hand, we are now ready to prove the upper bound in the theorem.
\DeterUpper*
\begin{proof}
    By the equivalence established in Section~\ref{subsec:setup}, it suffices to give a deterministic construction; this also establishes the zero-error randomized guarantee and the simultaneous guarantee with $p_f<1$.
    We design a deterministic rooted, port-aware MPGNN with depth $d = O\left(D + \frac{m\log C}{wp}\right)$ and width $w$ words per edge per round and word precision $p$ bits.
    The main idea is to encode each WL-type canonically and sort the distinct encoded types at the root, and use the resulting rank as the new color, which automatically satisfies the hash function condition.
    We present the framework of our construction and analyze the depth for each step:
	\begin{enumerate}[leftmargin=*]
    	\item Each node $u$ sends a message $\bs{x}_u$ to its neighbors and receives messages from them to form its WL-type as $M_u = \left(\bs{x}_u, \mset{\bs{x}_v : v \in N(u)}\right)$. 
        This process can be done by a rooted, port-aware MPGNN with depth $O\left(\max\left\{1, \frac{\log C}{wp}\right\}\right) = O\left(1 + \frac{\log C}{wp}\right)$.
	    \item The root node $r$ initiates the \alg{Flood} algorithm to construct a BFS spanning tree rooted at $r$. 
        This process can be done by an MPGNN with depth $O(D)$.
        \item The \alg{Upcast} algorithm is used to collect all sets $M_u$ at the root node $r$ along the spanning tree.
        This process can be done by an MPGNN with depth $O\left(D + \frac{m \log C}{wp}\right)$, since there are $\sum_{u \in V} O(\deg(u)) = O(m)$ colors to gather, each of size $O(\log C)$ bits\footnote{We use a self-delimiting encoding of each WL-type, including its degree or separators. This adds $O(n\log n)$ bits, which is absorbed by $O(m\log C)$ for connected graphs and \(C\ge n\).}.
        \item The root node merges all $M_u$ and deduplicates them to form the set $K = \set{M_u : u \in V}$.
        \item The root node sorts $K$ according to an arbitrarily fixed total order on all WL-types, for example lexicographic order, and forms a mapping $\func{rk}: K \to [|K|]$ that assigns to each element of $K$ its rank in $[|K|]$, represented by the set of pairs $R = \set{(k, \func{rk}(k)) : k \in K}$.
        \item The \alg{Downcast} algorithm is used to send the mapping $R$ back to every node along the spanning tree.
        This process can be done by an MPGNN with depth $O\left(D + \frac{m\log C}{wp}\right)$, since the size of $R$ is $\sum_{u \in V} O(\deg(u)\log C) + O(n\log n) = O(m\log C)$ bits.
        \item Finally, each node determines its new color from its own WL-type and the returned mapping $\func{rk}$.
	\end{enumerate}
	Thus, one step of the WL test can be computed by an MPGNN with depth
    \begin{equation*}
        d = O\left(1 + \frac{\log C}{wp} + D + D + \frac{m\log C}{wp} + D + \frac{m\log C}{wp}\right) = O\left(D + \frac{m\log C}{wp}\right).
    \end{equation*}
\end{proof}

\section{Introduction to the Bounded-Error Randomized Construction by Aamand et al.}
\label{app:introduction-aamand}
In this section, we introduce the MPGNN construction of Aamand et al.~\cite{aamand2022exponentially} for bounded-error randomized simulation of one step of the WL test.

Their Definition~1.1 states a pairwise success guarantee for each fixed graph $G$ and input coloring $\chi$.
Specifically, for any fixed pair of nodes $u, v \in V$, they require that if $\WLtype_{G, \chi}(u) = \WLtype_{G, \chi}(v)$, then $\chi'(u) = \chi'(v)$ with probability $1$, and if $\WLtype_{G, \chi}(u) \neq \WLtype_{G, \chi}(v)$, then $\chi'(u) \neq \chi'(v)$ with probability at least $1 - q$, where $\chi'$ is the computed new coloring.
In contrast, in our formulation, simulating one step of the WL test with success probability at least $1 - p_f$ means producing the correct new colors of all nodes with probability at least $1 - p_f$.
By a union bound over all unordered node pairs of this fixed instance, their pairwise guarantee implies our all-pairs correctness guarantee when setting $q = O\left(\frac{p_f}{n^2}\right)$.
This gives a per-instance guarantee as defined in Section~\ref{subsec:setup}.

In their construction, each node carries a color label $\bs{x}_u \in \set{0, 1}^{c\log n}$ for a constant $c$, which is equivalent to a color in $[n^{c_0}]$ for some constant $c_0$ and can be represented by $\Theta(\log n)$ bits.
In our single-step setting, we resample the shared randomness for each fixed instance and encode its input colors into these vectors in the message function.
Each node exchanges its encoded color with its neighbors and performs local computation to obtain its new color, so the total message exchanged on each edge is $wp = \Theta(\log n)$ bits.
The construction requires access to public randomness: all nodes can read a shared random vector $\bs{b} \in [F]^{c\log n}$, where each coordinate is uniform in $[F]$.
In addition, the construction samples $t = c\log n$ $\epsilon$-biased vectors $\bs{a}_1, \cdots, \bs{a}_t \in \set{0, 1}^F$.
By Corollary~3.3 of~\cite{aamand2022exponentially}, each $\bs{a}_i$ can be generated using $O(\log F + \log(1/\epsilon))$ truly random bits.
To bound the number of random bits, note that their analysis bounds the failure probability by $\frac{1}{F}$, so to ensure $\frac{1}{F} \le q$ it suffices to take $F = \Theta\left(\frac{n^2}{p_f}\right)$, which implies $\log F = \Theta(\log n)$ when $p_f = \Omega\left(\frac{1}{\poly(n)}\right)$.
Finally, the parameters $c$ and $\epsilon$ are chosen appropriately to satisfy $\left(\frac{1 + \epsilon}{2}\right)^{c\log n} \le q = O\left(\frac{p_f}{n^2}\right)$, which yields the result.

\section{Hash Function Background and Proofs of Theorem~\ref{thm:randomized-upper-large-color} and Theorem~\ref{thm:randomized-upper-small-color}}
\label{app:proof-randomized-upper}
In this section, we present the proofs of Theorem~\ref{thm:randomized-upper-large-color} and Theorem~\ref{thm:randomized-upper-small-color}. 

\subsection{Introduction to Hash Functions}
We begin by introducing hash functions as a key proof tool.
Hash functions are widely used in computer science, for example in cryptography, pseudorandomness, and data structures.
Given two finite sets $U$ and $[M]$ (usually $|U| \gg M$), a hash family $\setclass{H}$ is a collection of functions from $U$ to $[M]$ (i.e., $\setclass{H} \subseteq [M]^{U}$).
Carter and Wegman~\cite{carter1977universal} defined universal hash functions as follows: given a hash family $\setclass{H}$ from $U$ to $[M]$, we say that $\setclass{H}$ is universal if for any two distinct inputs $x \neq y$ in $U$, the collision probability satisfies $\Pr{h(x) = h(y)} \le \frac{1}{M}$, where $h$ is sampled uniformly at random from $\setclass{H}$.
In some applications, it suffices to relax the bound $\frac{1}{M}$ and consider $\epsilon$-almost universal hashing: for $\frac{1}{M} \le \epsilon < 1$, we say that $\setclass{H}$ is $\epsilon$-almost universal if for any two distinct inputs $x \neq y$ in $U$, $\Pr{h(x) = h(y)} \le \epsilon$.

Suppose we want to design a universal hash family from $U = \set{0, 1}^u$ to $[M] = \set{0, 1}^t$.
A simple choice is the family of linear hash functions $\setclass{H} = \set{h_{\bs{A}}(\bs{x}) \coloneqq \bs{A}\bs{x} : \bs{A} \in \set{0, 1}^{t \times u}}$, where the multiplication between a binary matrix and a binary vector is over $\GF(2)$ (i.e., modulo $2$).
Sampling a hash function from $\setclass{H}$ requires $tu$ random bits to choose a uniformly random matrix $\bs{A}$.
If we allow $O\left(\frac{1}{M}\right)$-almost universal hashing, then a hash function can be sampled using only $\Theta(\log u + t)$ random bits.
This is formalized in Lemma~\ref{lem:hash}:
\begin{lemma}
    \label{lem:hash}
    For any constant $c > 1$, there exists a $\frac{c}{2^t}$-almost universal hash family $\setclass{H}$ from $\set{0, 1}^u$ to $\set{0, 1}^t$, and sampling $h \in \setclass{H}$ can be done using $\Theta(\log u + t)$ random bits.
\end{lemma}
\begin{proof}
We define a family $\setclass{H} = \set{h_{x, y} : x, y \in \GF(2^b)}$ where each $h_{x, y}$ is a linear map of the form $h_{x, y}(\bs{z}) = \bs{A}(x, y)\bs{z}$ for some binary matrix $\bs{A}(x, y) \in \set{0, 1}^{t \times u}$ generated from the seed $(x, y)$.
We will choose $b$ later such that the family is $\frac{c}{2^t}$-almost universal.

We first fix $b$ and a binary bijective representation map $\func{bin}: \GF(2^b) \to \set{0, 1}^b$ such that $\func{bin}(0) = (0, 0, \cdots, 0)$ and $\func{bin}(u + v) = \func{bin}(u) \oplus \func{bin}(v)$ for all $u, v \in \GF(2^b)$.
For example, viewing an element $\alpha \in \GF(2^b)$ as a binary polynomial $a_1 + a_2x + \cdots + a_{b}x^{b - 1}$ of degree at most $b - 1$ with $a_i \in \set{0, 1}$, one can set $\func{bin}(\alpha) = (a_1, a_2, \cdots, a_{b}) \in \set{0, 1}^b$.

Let $N = tu \ge 2$. 
We generate a matrix $\bs{A} = \bs{A}(x, y)$ using limited randomness as follows:
\begin{enumerate}[leftmargin=*]
    \item Sample two elements $x, y$ independently and uniformly at random from $\GF(2^b)$.
    \item For $i = 1, \cdots, N$, let $r_i \coloneqq \langle \func{bin}(x^{i - 1}), \func{bin}(y)\rangle$, where $\langle \cdot, \cdot \rangle$ denotes the inner product over $\GF(2)$.
    Concretely, for $\bs{a} = (\bs{a}_1, \cdots, \bs{a}_{b})$ and $\bs{b} = (\bs{b}_1, \cdots, \bs{b}_{b})$ in $\set{0, 1}^b$, we define $\langle \bs{a}, \bs{b}\rangle \coloneqq \sum_{j = 1}^{b} \bs{a}_j \bs{b}_j \bmod 2$.
    \item Fill the bits $r_1, \dots, r_{N}$ into the matrix $\bs{A}$ row by row, i.e., for each $j \in [t]$ and $k \in [u]$ we set $\bs{A}_{j, k} \coloneqq r_{(j - 1)u + k}$.
\end{enumerate}
Let $\bs{r} = (r_1, \dots, r_{N}) \in \set{0, 1}^N$ denote this sampled string.
We now analyze the collision probability.

Fix any two distinct inputs $\bs{x}, \bs{x}' \in \set{0, 1}^u$.
A collision under $h_{x, y}$ occurs when
$\bs{A}\bs{x} = \bs{A}\bs{x}' \iff \bs{A}(\bs{x} \oplus \bs{x}') = \bs{0} \in \set{0, 1}^t$.
Since $\bs{x} \neq \bs{x}'$, letting $\bs{d} \coloneqq \bs{x} \oplus \bs{x}' \neq \bs{0}$, the collision probability reduces to the probability that the random matrix $\bs{A}$ maps a fixed nonzero vector $\bs{d} \in \set{0, 1}^u$ to the zero vector.

Define a random variable $\bs{z} \coloneqq \bs{A}\bs{d}\in \set{0, 1}^t$. 
For each $\bs{s} \in \set{0, 1}^t$, define $\chi_{\bs{s}}(\bs{z}) = (-1)^{\left\langle \bs{s}, \bs{z}\right\rangle}$, where $\left\langle \bs{s}, \bs{z}\right\rangle$ denotes the mod-$2$ inner product.
It is straightforward to verify the identity
$\mathbb{I}\left[\bs{z} = \bs{0}\right] = 2^{-t}\sum_{\bs{s} \in \set{0, 1}^t}\chi_{\bs{s}}(\bs{z})$.
Then
\begin{equation}
    \label{eqn:collision}
    \begin{aligned}
        \Pr{h_{x,y}(\bs{x}) = h_{x,y}(\bs{x}')} = \Pr{\bs{z} = \bs{0}}
        & = \E{\mathbb{I}\left[\bs{z} = \bs{0}\right]} \\
        & = 2^{-t}\sum_{\bs{s} \in \set{0, 1}^t}\E{\chi_{\bs{s}}(\bs{z})} \\
        & = 2^{-t}\left(\E{\chi_{\bs{0}}(\bs{z})} + \sum_{\bs{0} \neq \bs{s} \in \set{0, 1}^t}\E{\chi_{\bs{s}}(\bs{z})}\right) \\
        & \le 2^{-t}\left(1 + (2^t - 1) \max_{\bs{0} \neq \bs{s} \in \set{0, 1}^t}\left|\E{\chi_{\bs{s}}(\bs{z})}\right|\right) \\
        & \le 2^{-t} + \max_{\bs{0} \neq \bs{s} \in \set{0, 1}^t}\left|\E{\chi_{\bs{s}}(\bs{z})}\right|.
    \end{aligned}
\end{equation}
It remains to bound $\max_{\bs{0} \neq \bs{s} \in \set{0, 1}^t}\left|\E{\chi_{\bs{s}}(\bs{z})}\right|$.
We first rewrite $\langle \bs{s}, \bs{z}\rangle$ as a linear test on $\bs{r}$.
For any fixed $\bs{s}\in \set{0,1}^t$ and $\bs{d}\in \set{0,1}^u$,
\begin{align*}
    \langle \bs{s}, \bs{z}\rangle
    = \sum_{j = 1}^{t} \bs{s}_j \bs{z}_j \bmod 2
    = \sum_{j = 1}^{t} \bs{s}_j \left(\sum_{k = 1}^{u} \bs{A}_{j, k} \bs{d}_k \right) \bmod 2
    = \sum_{j = 1}^{t}\sum_{k = 1}^{u} (\bs{s}_j \bs{d}_k) \bs{A}_{j, k} \bmod 2.
\end{align*}
Define $\bs{w} = \bs{w}(\bs{s}, \bs{d})\in \set{0, 1}^N$ by $\bs{w}_{(j - 1)u + k} \coloneqq \bs{s}_j \bs{d}_k$ for all $j \in [t]$ and $k \in [u]$.
Since $r_{(j - 1)u + k} = \bs{A}_{j, k}$ by construction, we have $\langle \bs{s}, \bs{z}\rangle = \langle \bs{w}, \bs{r}\rangle$, and therefore $\E{\chi_{\bs{s}}(\bs{z})} = \E{(-1)^{\langle \bs{s}, \bs{z}\rangle}} = \E{(-1)^{\langle \bs{w}, \bs{r}\rangle}}$.
Moreover, because $\bs{d}\neq \bs{0}$, we have $\bs{w}\neq \bs{0}$ whenever $\bs{s}\neq \bs{0}$.
Hence it suffices to upper bound $\left|\E{(-1)^{\langle \bs{w}, \bs{r}\rangle}}\right|$ for all $\bs{0}\neq \bs{w}\in \set{0,1}^N$.

Proposition~3 of~\cite{alon1992simple} implies
$\max_{\bs{0}\neq \bs{w}\in \set{0, 1}^N}\left|\E{(-1)^{\langle \bs{w}, \bs{r}\rangle}}\right| \le \frac{N - 1}{2^b}$.
For completeness, we prove this bound below.

Fix some $\bs{w} \in \set{0, 1}^N \setminus \set{\bs{0}}$, and consider the linear test
$\left\langle \bs{w}, \bs{r}\right\rangle = \sum_{i = 1}^{N} \bs{w}_i \bs{r}_i \bmod 2$.
In the sampling procedure above, $\bs{r} \in \set{0, 1}^{N}$ is a function of the field elements $x, y \in \GF(2^b)$.
Thus, we can rewrite
\begin{align*}
    \left\langle \bs{w}, \bs{r}(x, y)\right\rangle
    & = \sum_{i = 1}^{N} \bs{w}_i \left\langle \func{bin}(x^{i - 1}), \func{bin}(y)\right\rangle \\
    & = \left\langle \func{bin}\left(\sum_{i = 1}^{N} \bs{w}_i x^{i - 1}\right), \func{bin}(y)\right\rangle,
\end{align*}
where the second equality uses linearity of the inner product in its first argument and the property of $\func{bin}$.
Let $p_{\bs{w}}(X) \coloneqq \sum_{i = 1}^{N}\bs{w}_i X^{i - 1}$ be a polynomial over $\GF(2)$, and then $\left\langle \bs{w}, \bs{r}(x, y)\right\rangle = \left\langle \func{bin}(p_{\bs{w}}(x)), \func{bin}(y)\right\rangle$.
We consider two cases:
\begin{itemize}[leftmargin=*]
    \item If $p_{\bs{w}}(x) = 0$, then we have $\langle \bs{w}, \bs{r}(x, y)\rangle = \langle \func{bin}(0), \func{bin}(y) \rangle = 0$.
    Hence the probability $\Pr{\langle \bs{w}, \bs{r}(x, y)\rangle = 1 \Given p_{\bs{w}}(x) = 0} = 0$.

    \item If $p_{\bs{w}}(x) \neq 0$, then for any fixed nonzero vector $\func{bin}(p_{\bs{w}}(x)) \in \set{0, 1}^b$, exactly half of the $\func{bin}(y) \in \set{0, 1}^b$ make the inner product
    $\langle \func{bin}(p_{\bs{w}}(x)), \func{bin}(y)\rangle$ equal to $0$.
    Since $y$ is uniform over $\GF(2^b)$ and $\func{bin}$ is a bijection, $\func{bin}(y)$ is uniform over $\set{0, 1}^b$.
    Therefore,
    $\Pr{\langle \bs{w}, \bs{r}(x, y)\rangle = 0 \Given p_{\bs{w}}(x) \neq 0}
    = \Pr{\langle \bs{w}, \bs{r}(x, y)\rangle = 1 \Given p_{\bs{w}}(x) \neq 0}
    = \frac{1}{2}$.
\end{itemize}
Since $\deg(p_{\bs{w}}) \le N - 1$ and $p_{\bs{w}}$ is nonzero, it has at most $N - 1$ roots in $\GF(2^b)$, and thus $\Pr{p_{\bs{w}}(x) = 0} \le \frac{N - 1}{2^b}$.
By the law of total probability,
\begin{align*}
    \Pr{\langle \bs{w}, \bs{r}(x, y)\rangle = 1}
    &= \Pr{\langle \bs{w}, \bs{r}(x, y)\rangle = 1 \Given p_{\bs{w}}(x) = 0}\Pr{p_{\bs{w}}(x) = 0} \\
    &\quad + \Pr{\langle \bs{w}, \bs{r}(x, y)\rangle = 1 \Given p_{\bs{w}}(x) \neq 0}\Pr{p_{\bs{w}}(x) \neq 0} \\
    &= 0 + \frac{1}{2}\left(1 - \Pr{p_{\bs{w}}(x) = 0}\right)
     = \frac{1}{2} - \frac{1}{2}\Pr{p_{\bs{w}}(x) = 0},
\end{align*}
and similarly,
\begin{align*}
    \Pr{\langle \bs{w}, \bs{r}(x, y)\rangle = 0}
    &= \Pr{p_{\bs{w}}(x) = 0} + \frac{1}{2}\left(1 - \Pr{p_{\bs{w}}(x) = 0}\right)
     = \frac{1}{2} + \frac{1}{2}\Pr{p_{\bs{w}}(x) = 0}.
\end{align*}
Therefore,
\begin{align*}
    \left|\E{(-1)^{\langle \bs{w}, \bs{r}(x, y)\rangle}}\right|
    &= \left|1\cdot \Pr{\langle \bs{w}, \bs{r}\rangle = 0} - 1\cdot \Pr{\langle \bs{w}, \bs{r}\rangle = 1}\right| \\
    &= \Pr{p_{\bs{w}}(x) = 0}
    \le \frac{N - 1}{2^b}.
\end{align*}
This proves $\max_{\bs{0}\neq \bs{w}\in \set{0,1}^N}\left|\E{(-1)^{\langle \bs{w}, \bs{r}\rangle}}\right|\le \frac{N - 1}{2^b}$, and hence $\max_{\bs{0}\neq \bs{s}\in \set{0,1}^t}\left|\E{\chi_{\bs{s}}(\bs{z})}\right|\le \frac{N - 1}{2^b}$.
Substituting into Equation~\eqref{eqn:collision}, we obtain $\Pr{h_{x, y}(\bs{x}) = h_{x, y}(\bs{x}')} \le 2^{-t} + \frac{N - 1}{2^b}$.
To make this probability at most $\frac{c}{2^t}$, it suffices to ensure $\frac{N - 1}{2^b} \le \frac{c - 1}{2^t}$, namely $2^b \ge \frac{N - 1}{c - 1}2^t$, equivalently
$b \ge t + \log_2(N - 1) - \log_2(c - 1)$.
Thus, setting $b \coloneqq \left\lceil t + \log_2(tu - 1) - \log_2(c - 1) \right\rceil$ satisfies the desired collision-probability bound, and $\setclass{H}$ is $\frac{c}{2^t}$-almost universal.

Finally, sampling $h_{x, y}$ requires sampling $(x, y)$ uniformly from $\GF(2^b)^2$, which takes $R = 2b = \Theta(t + \log(tu)) = \Theta(\log u + t)$ random bits.
\end{proof}

\subsection{Proof of Theorem~\ref{thm:randomized-upper-large-color} and Theorem~\ref{thm:randomized-upper-small-color}}

Now we can prove Theorem~\ref{thm:randomized-upper-large-color} and Theorem~\ref{thm:randomized-upper-small-color} using Lemma~\ref{lem:hash} as our main tool.
\RandUpperLarge*
\begin{proof}
    We construct an anonymous unordered-neighborhood MPGNN with depth $d = 1$, width $w$ words per edge per round, and word precision $p$ bits.
    We outline the construction and analyze the depth of each step:
    \begin{enumerate}[leftmargin=*]
        \item Each node $u$ sends its current color $\bs{x}_u$ to its neighbors and receives their colors to form its WL-type $(\bs{x}_u, \mset{\bs{x}_v : v \in N(u)})$.
        This step can be done by a model with depth $1$ since $wp \ge \left\lceil \log_2 C\right\rceil$.
        \item Each node encodes its WL-type into a binary vector in a canonical way.
        For example, it can first encode the WL-type as a vector in $\set{0, 1, \cdots, C}^n$ by setting the first entry to $\bs{x}_u$, the next $\deg(u) \le n - 1$ entries to the sorted multiset $\mset{\bs{x}_v : v \in N(u)}$, and padding the remaining entries with a special symbol $0$ not used as a color.
        It then converts each entry into $\left\lceil\log_2(C + 1)\right\rceil$ bits, yielding a binary vector $\bs{c}_u \in \set{0, 1}^{n\left\lceil\log_2(C + 1)\right\rceil}$.
        \item Let $F \coloneqq \left\lceil\frac{3n^2}{2p_f}\right\rceil \le C$. 
        We apply Lemma~\ref{lem:hash} to construct a $\frac{2}{2^{\left\lceil\log_2 F\right\rceil}}$-almost universal hash family $\setclass{H}_1$ from $\set{0, 1}^{n\left\lceil\log_2(C + 1)\right\rceil}$ to $\set{0, 1}^{\left\lceil\log_2 F\right\rceil}$, and use $\Theta\left(\log\left(n\log C\right) + \log F\right) = \Theta\left(\log \frac{n\log C}{p_f}\right)$ random bits to sample a hash function $h_1 \in \setclass{H}_1$.
        \item We then construct another universal hash family $\setclass{H}_2$ from $\set{0, 1}^{\left\lceil\log_2 F\right\rceil}$ to $[F]$ via a standard construction (e.g., $\set{h_{a, b}(x) = 1 + ((ax + b) \bmod P) \bmod F : a \in [P - 1], b \in \set{0, 1, \cdots, P - 1}}$ for some prime $P \ge 2^{\left\lceil\log_2 F\right\rceil}$), and use $\Theta(\log F) = \Theta\left( \log\frac{n}{p_f}\right)$ random bits to sample a function $h_2 \in \setclass{H}_2$.
        \item Finally, we set the new color of node $u$ to be $\bs{y}_u \coloneqq h_2(h_1(\bs{c}_u))$.
    \end{enumerate}
    Next, we bound the collision probability. 
    Fix any $u, v \in V$ such that $\bs{c}_u \neq \bs{c}_v$. 
    We have
    \begin{align*}
        \Pr{\bs{y}_u = \bs{y}_v}
        & = \Pr{h_2(h_1(\bs{c}_u)) = h_2(h_1(\bs{c}_v)) \Given h_1(\bs{c}_u) = h_1(\bs{c}_v)}\Pr{h_1(\bs{c}_u) = h_1(\bs{c}_v)} \\
        & \quad + \Pr{h_2(h_1(\bs{c}_u)) = h_2(h_1(\bs{c}_v)) \Given h_1(\bs{c}_u) \neq h_1(\bs{c}_v)}\Pr{h_1(\bs{c}_u) \neq h_1(\bs{c}_v)} \\
        & \le \Pr{h_1(\bs{c}_u) = h_1(\bs{c}_v)} + \frac{1}{F}\left(1 - \Pr{h_1(\bs{c}_u) = h_1(\bs{c}_v)}\right) \\
        & \le \frac{1}{F} + \left(1 - \frac{1}{F}\right)\Pr{h_1(\bs{c}_u) = h_1(\bs{c}_v)} \\
        & \le \frac{1}{F} + \left(1 - \frac{1}{F}\right)\frac{2}{2^{\left\lceil\log_2 F\right\rceil}}
        \le \frac{3}{F}.  
    \end{align*}
    
    By a union bound, the overall failure probability is
    \begin{align*}
        \Pr{\exists u \neq v \in V \text{ such that } \bs{c}_u \neq \bs{c}_v \text{ and } \bs{y}_u = \bs{y}_v}
        \le \sum_{\substack{\set{u, v} \subseteq V \\ \bs{c}_u \neq \bs{c}_v}} \Pr{\bs{y}_u = \bs{y}_v}
        \le \frac{3n^2}{2F}
        \le p_f.
    \end{align*}
\end{proof}

We now prove Theorem~\ref{thm:randomized-upper-small-color}.
\RandUpperSmall*
\begin{proof}
    We prove this theorem by constructing a rooted, port-aware MPGNN with bandwidth $w$ words per edge per round and word precision $p$ bits.
    The construction proceeds as follows:
    \begin{enumerate}[leftmargin=*]
        \item First, we can use a rooted, port-aware MPGNN of depth $O\left(1 + \frac{\log C}{wp}\right)$ with $\Theta\left(\log\frac{n\log C}{p_f}\right)$ random bits, where each node first forms its WL-type and then maps the canonical encoding of this WL-type to a temporary color $\bs{y}_u \in \left[\left\lceil\frac{3n^2}{2p_f}\right\rceil\right]$.
        \item The root node $r$ initiates \alg{Flood} to construct a BFS spanning tree rooted at $r$.
        This can be done by a model with depth $O(D)$.
        \item We use \alg{Upcast} to collect all node colors $\bs{y}_u$ at the root node $r$ along the spanning tree.
        This can be done by a model with depth $O\left(D + \frac{n \log (n / p_f)}{wp}\right)$ rounds, since there are $O(n)$ colors to gather, each of size $O\left(\log \left\lceil\frac{3n^2}{2p_f}\right\rceil\right) = O\left(\log\frac{n}{p_f}\right)$ bits, and each edge can transmit $wp$ bits per round.
        \item The root node merges and deduplicates the colors $\bs{y}_u$ to form the set $K = \set{\bs{y}_u : u \in V}$, then sorts $K$ and assigns new colors according to rank.
        Concretely, it produces a set $R = \set{(\bs{y}_u, \bs{z}_u) : u \in V}$, where $\bs{z}_u \coloneqq \left|\set{\bs{y}_v: v \in V, \bs{y}_v < \bs{y}_u}\right| + 1$ is the rank of $\bs{y}_u$ among the distinct values in $\set{\bs{y}_v : v \in V}$.
        \item The \alg{Downcast} algorithm is used to send the mapping $R$ back to every node along the spanning tree.
        This process can be carried out by an MPGNN with depth $O\left(D + \frac{n\log (n / p_f)}{wp}\right)$, since there are $O(n)$ messages to transmit, each of size $O(\log (n / p_f) + \log n) = O(\log (n / p_f))$ bits.
        \item Finally, each node determines its new color from its own color $\bs{y}_u$ and the returned mapping $R$.
    \end{enumerate}
    The algorithm succeeds whenever the hashing subroutine is collision-free on the WL-types that occur. 
    Therefore the failure probability is at most $p_f$.
\end{proof}

\section{Proof of Theorem~\ref{thm:randomized-lower-randomness}}
\label{app:proof-randomized-lower-randomness}
In this section, we prove Theorem~\ref{thm:randomized-lower-randomness}.
\RandLowerRandomness*
\begin{proof}
    We reduce $\func{EQ}_n$ to simulating one step of the WL test.
    We use the hard-instance construction from the proof of Theorem~\ref{thm:deterministic-lower} with parameters $m = n$ and $D = 8$.
    Therefore, we have a graph family $\mathcal{G}_n = \set{G_{\bs{a},\bs{b}} : \bs{a}, \bs{b} \in \set{0, 1}^n}$ consisting of connected uncolored graphs with $\Theta(n)$ nodes and $\Theta(n)$ edges. 
    Moreover, the construction has an Alice-Bob partition whose cut contains exactly one graph edge, and it contains two distinguished nodes $x^{(A)}$ and $x^{(B)}$ such that, if $\chi_G^{(t)}$ denotes the coloring after $t$ rounds of the WL test starting from the uniform coloring, then $\chi_G^{(3)}(x^{(A)}) = \chi_G^{(3)}(x^{(B)})$ if and only if $\bs{a} = \bs{b}$.

    Now suppose that there exists a randomized oblivious MPGNN $\mathcal{M}$ with depth $d$, width $w$, precision $p$, and $R$ random bits that simulates one step of the WL test with failure probability at most $p_f$, and suppose further that $dwp = O(\log C)$.
    Alice and Bob use $\mathcal M$ three times in sequence to simulate three steps of the WL test on $G_{\bs a,\bs b}$.
    The composed simulator has depth $3d$ and uses $3R$ random bits. 
    Using three independent copies of $\mathcal{M}$ and conditioning on the correctness of the previous simulated rounds, all three rounds are correct with probability at least $(1 - p_f)^3 > 1 - 3p_f > 2/3$.
    During this simulation, the Alice-Bob cut contains exactly one graph edge. 
    Hence each message-passing layer requires only $O(wp)$ bits of communication across the cut. 
    Over $3d$ layers, the total communication used to simulate the MPGNN is $O(dwp)$. 
    After the simulation, Alice and Bob compare the output colors of $x^{(A)}$ and $x^{(B)}$. 
    Since the output set is $[C]$, this costs another $O(\log C)$ bits. 
    Therefore, we obtain a randomized public-coin protocol for $\func{EQ}_n$ with communication cost $k = O(dwp + \log C) = O(\log C)$, randomness $3R$, and failure probability at most $3p_f < 1/3$.

    By the limited-randomness lower bound for Equality~\cite{jacobs2025communication}, any randomized public-coin protocol for $\func{EQ}_n$ using $\rho$ random bits and failure probability $\varepsilon$ must have communication cost $k\ge \Omega\left(\min\left\{n,\frac{n}{\varepsilon 2^\rho}\right\}\right)$.
    Applying this bound with $\rho = 3R$ and $\varepsilon = 3p_f$ gives $k\ge \Omega\left(\min\left\{n,\frac{n}{p_f2^{3R}}\right\}\right)$.
    Since $k = O(\log C)$ and $\log C = o(n)$, the first alternative $k = \Omega(n)$ is impossible for sufficiently large $n$. 
    Hence we must have $O(\log C)\ge \Omega\left(\frac{n}{p_f2^{3R}}\right)$.
    Rearranging gives $2^{3R}\ge \Omega\left(\frac{n}{p_f\log C}\right)$. 
    Taking logarithms yields $R = \Omega\left(\log\frac{n}{p_f\log C}\right)$, as claimed.
\end{proof}

\section{ExistsEqual Problem and Proof of Theorem~\ref{thm:randomized-lower-small-color}}
\label{app:proof-randomized-lower-small-color}
In this section, we prove Theorem~\ref{thm:randomized-lower-small-color}.
Before that, we introduce a communication complexity problem called $\func{ExistsEqual}_{U, k}$ (or $\exists \func{EQ}_{U, k}$ for short).
\begin{definition}[$\func{ExistsEqual}_{U, k}$, cf.~\cite{huang2020communication}]
    Given a set $U$ and a positive integer $k$, Alice and Bob hold vectors $\bs{x} \in U^k$ and $\bs{y} \in U^k$, respectively. 
    They must determine whether there exists an index $i \in [k]$ such that $\bs{x}_i = \bs{y}_i$.
\end{definition}
The communication complexity of $\func{ExistsEqual}_{U, k}$ has been studied extensively.
For error probability $p_f = \Theta(1)$ and for $p_f = \Omega\left(\frac{1}{\poly(k)}\right)$, Saglam and Tardos~\cite{saglam2013communication} and Brody et al.~\cite{brody2016certifying} showed that, when $|U| = \Omega(k)$, the communication complexity of $\func{ExistsEqual}_{U, k}$ is $\Omega\left(k\log^{(r)}k\right)$, where $\log^{(i)}(x)$ denotes the iterated logarithm, defined by $\log^{(1)}(x) \coloneqq \log x$ and $\log^{(i)}(x) \coloneqq \log\left(\log^{(i-1)}(x)\right)$ for $i \ge 2$, and $r$ denotes the number of communication rounds.

We are now ready to present the proof of Theorem~\ref{thm:randomized-lower-small-color}.
\RandLowerSmallColor*
\begin{proof}
    We reduce the $\exists\func{EQ}$ problem to simulating one step of the WL test.
    The proof follows the same high-level idea as the proof of Theorem~\ref{thm:deterministic-lower}. 
    We first construct a virtually colored hard instance, and then encode the virtual colors by node degrees using shared padding and regularizer nodes.

    \paragraph{Construction.}
    We reduce from $\exists\func{EQ}_{U, n}$, where $U$ is the set of all multisets of size $2$ over $[n]$. 
    Thus $|U| = \binom{n + 1}{2} = \Theta(n^2)$. 

    \textbf{\textit{Step One: Fixed Basic Graph.}}
    For each side $\sigma \in \set{A, B}$, create nodes $x^{(\sigma)}$, nodes $w^{(\sigma)}_1, \cdots, w^{(\sigma)}_n$, and nodes $u^{(\sigma)}_{i, 1}, u^{(\sigma)}_{i, 2}$ for every $i \in [n]$. 
    Add edges $\set{x^{(\sigma)}, w^{(\sigma)}_i}$, $\set{w^{(\sigma)}_i, u^{(\sigma)}_{i, 1}}$, and $\set{w^{(\sigma)}_i, u^{(\sigma)}_{i, 2}}$ for every $i \in [n]$.
    Then, connect the two sides by a path of length $\Theta(D)$, for example by adding nodes $c_1, \cdots, c_{D}$ and edges $\set{x^{(A)}, c_1}, \set{c_1, c_2}, \cdots, \set{c_{D}, x^{(B)}}$.
    Figure \ref{fig:randomized-lower} below gives an illustration of the fixed basic graph, where the connecting path is simplified.
    \begin{figure*}[ht]
    	\centering
    	\begin{tikzpicture}[scale = 0.8, transform shape]
    		\tikzset{every node/.style = {minimum size = 1.0cm, thick, inner sep = 0pt}}
    
    		\tikzset{uNode/.style = {circle, draw, minimum size = 0.75cm}}
    
    		\node[circle, draw] (xA) at (-.8, 0) {$x^{(A)}$};
    		\foreach \i in {1, ..., 5}
    		{                
    			\ifnum \i = 4
    				\node (wA\i) at (-3, {2 * 3 - 2 * \i}) {$\dots$};
    			\else
    				\node[circle, draw] (wA\i) at (-3, {2 * 3 - 2 * \i})
    				{
    					\ifnum \i < 4 $w^{(A)}_{\i}$ \else $w^{(A)}_{n}$ \fi
    				};
    			\fi
    		}
    		\foreach \i in {1, 2, 3, 5}
    		{
    			\draw[thick] (xA) -- (wA\i);
    		}
    
    		\foreach \i in {1, 2, 3, 5}
    		{
    			\ifnum\i<4 \def\idx{\i} \else \def\idx{n} \fi
    			\node[uNode] (uA\i1) at (-5, {2 * 3 - 2 * \i + 0.5}) {$u^{(A)}_{\idx, 1}$};
    			\node[uNode] (uA\i2) at (-5, {2 * 3 - 2 * \i - 0.5}) {$u^{(A)}_{\idx, 2}$};
    			\draw[thick] (wA\i) -- (uA\i1);
    			\draw[thick] (wA\i) -- (uA\i2);
    		}
    
    		\node[circle, draw] (xB) at (.8, 0) {$x^{(B)}$};
    		\foreach \i in {1, ..., 5}
    		{                
    			\ifnum \i = 4
    				\node (wB\i) at (3, {2 * 3 - 2 * \i}) {$\dots$};
    			\else
    				\node[circle, draw] (wB\i) at (3, {2 * 3 - 2 * \i})
    				{
    					\ifnum \i < 4 $w^{(B)}_{\i}$ \else $w^{(B)}_{n}$ \fi
    				};
    			\fi
    		}
    		\foreach \i in {1, 2, 3, 5}
    		{
    			\draw[thick] (xB) -- (wB\i);
    		}
    
    		\foreach \i in {1, 2, 3, 5}
    		{
    			\ifnum\i<4 \def\idx{\i} \else \def\idx{n} \fi
    			\node[uNode] (uB\i1) at (5, {2 * 3 - 2 * \i + 0.5}) {$u^{(B)}_{\idx, 1}$};
    			\node[uNode] (uB\i2) at (5, {2 * 3 - 2 * \i - 0.5}) {$u^{(B)}_{\idx, 2}$};
    			\draw[thick] (wB\i) -- (uB\i1);
    			\draw[thick] (wB\i) -- (uB\i2);
    		}
    
    		\draw[thick, decorate, decoration={snake, amplitude=1mm, segment length = 2mm}] (xA) -- (xB);
    		\draw [dashed, very thick] (0, 5) -- (0, -5);
    
    		\node[anchor = south east] at (-1, 5.) {\Large Alice};
    		\node[anchor = south west] at (1, 5.) {\Large Bob};
    	\end{tikzpicture}
        \caption{The constructed basic graph. The connecting path is simplified.}
        \label{fig:randomized-lower}
    \end{figure*}

    \textbf{\textit{Step Two: Virtually Colored Hard Instance.}}
    Given an input instance $(\bs{a}, \bs{b})\in U^n\times U^n$, Alice and Bob assign virtual colors as follows. 
    For every $i \in [n]$, set $\kappa\left(w^{(A)}_i\right) = \kappa\left(w^{(B)}_i\right) = i$. 
    If $\bs{a}_i = \mset{\alpha, \beta}$ with $\alpha \le \beta$, then Alice sets $\kappa\left(u^{(A)}_{i, 1}\right) = \alpha$ and $\kappa\left(u^{(A)}_{i, 2}\right) = \beta$. 
    Bob does the analogous assignment using $\bs{b}_i$. 
    The nodes $x^{(A)}$, $x^{(B)}$ and connecting nodes $c_i$ do not need virtual colors.

    \textbf{\textit{Step Three: Decolorization.}}
    We now decolorize this virtually colored graph separately on Alice's side and Bob's side, so no new edge crosses the Alice-Bob cut. 
    Let $Z^{(\sigma)} \coloneqq \set{w_{i}^{(\sigma)}: i \in [n]} \cup \set{u_{i, 1}^{(\sigma)}, u_{i, 2}^{(\sigma)}: i \in [n]}$ be the set of virtually colored nodes on side $\sigma$, and let $K\coloneqq n$ be the number of virtual colors used. 
    Let $\Delta_0 = 3$ be an upper bound on the degrees of nodes in $Z^{(\sigma)}$ before decolorization. 
    For every virtual color $k \in [K]$, define a distinct target degree $T_k = \Delta_0 + k = 3 + k$, and define the padding target degree $T_\perp = 3n + 1$ larger than $|Z^{(\sigma)}|$.
    Then all $T_k$ are pairwise distinct, all are larger than $\Delta_0$, and none equals $T_\perp$.
    
    For each $z\in Z^{(\sigma)}$, let $d_0(z)$ be its current degree before decolorization and set $p_z = T_{\kappa(z)} - d_0(z)$. 
    Since $T_{\kappa(z)} > \Delta_0 \ge d_0(z)$, we have $p_z > 0$. 
    Create shared padding nodes $P^{(\sigma)}_1, \ldots,P^{(\sigma)}_{M^{(\sigma)}}$, where $M^{(\sigma)} = \max_{z\in Z^{(\sigma)}}p_z$, and connect each $z\in Z^{(\sigma)}$ to $P^{(\sigma)}_1,\ldots, P^{(\sigma)}_{p_z}$. 
    Thus every virtually colored node $z$ obtains degree exactly $T_{\kappa(z)}$.

    We further modify the degrees of the padding nodes to ensure that they have the same first-round WL color.
    Let $s^{(\sigma)}_i \coloneqq \left|\set{z\in Z^{(\sigma)} : p_z\ge i}\right|$ be the number of neighbors currently adjacent to $P^{(\sigma)}_i$. 
    Since $s^{(\sigma)}_i\le \left|Z^{(\sigma)}\right| = 3n < T_\perp$, we can regularize the degrees of the padding nodes to $T_{\perp}$. 
    Let $R^{(\sigma)} \coloneqq \max_{i\in [M^{(\sigma)}]} (T_\perp - s^{(\sigma)}_i)$. 
    Create regularizer nodes $Q^{(\sigma)}_1, \cdots,Q^{(\sigma)}_{R^{(\sigma)}}$. 
    Since $R^{(\sigma)} > 0$, there exists at least one regularizer node on each side.
    For each $i\in [M^{(\sigma)}]$, connect $P^{(\sigma)}_i$ to $Q^{(\sigma)}_1, \cdots, Q^{(\sigma)}_{T_\perp - s^{(\sigma)}_i}$. 
    Then every padding node has degree exactly $s^{(\sigma)}_i + (T_\perp - s^{(\sigma)}_i) = T_\perp$. 
    Moreover, every regularizer node is adjacent to some padding node by the definition of $R^{(\sigma)}$, so this step does not create isolated nodes. 
    
    Finally, add equal numbers of filler leaves adjacent to $x^{(A)}$ and $x^{(B)}$, if necessary, so that $\deg(x^{(A)}) = \deg(x^{(B)}) = 3n + 2$ and this common degree is larger than $T_\perp$. 
    We denote the resulting graph by $G_{\bs{a},\bs{b}}$, and let $\mathcal{G}_{n, D}$ be the family of all such graphs.

    \paragraph{Analysis of Graph Parameters.}
    Similar to the proof of Theorem~\ref{thm:deterministic-lower}, the decolorized graph is connected and has $\Theta(n)$ nodes and $\Theta(D)$ diameter.
    
    \paragraph{Analysis of WL Colors.}
    Let $\chi_G^{(t)}$ denote the coloring after $t$ steps of the WL test starting from the uniform coloring on the decolorized graph $G = G_{\bs{a}, \bs{b}}$.

    \begin{itemize}[leftmargin=*]
        \item \textit{\textbf{Round $1$.}} 
        In round $1$, WL colors are determined by degrees.
        By construction, every virtually colored node $z\in Z^{(A)}\cup Z^{(B)}$ has degree $T_{\kappa(z)}$, and the target degrees $T_1, \cdots, T_K$ are pairwise distinct.
        Hence, for any $z, z'\in Z^{(A)}\cup Z^{(B)}$, we have $\chi_G^{(1)}(z) = \chi_G^{(1)}(z')$ if and only if $\kappa(z) = \kappa(z')$. 
        Thus the first step of the WL test recovers the intended virtual colors on all $w$-nodes and $u$-nodes.
        Also, $\chi_G^{(1)}(x^{(A)}) = \chi_G^{(1)}(x^{(B)})$, because the two distinguished nodes have the same degree. 
        The padding nodes all have the same round-$1$ color, corresponding to degree $T_\perp$.

        \item \textit{\textbf{Round $2$.}} 
        Fix $i, j \in [n]$. 
        If $i\neq j$, then we have $\chi_G^{(2)}\left(w^{(A)}_i\right) \neq \chi_G^{(2)}\left(w^{(B)}_j\right)$, because $\chi_G^{(1)}\left(w^{(A)}_i\right)\neq \chi_G^{(1)}\left(w^{(B)}_j\right)$. 
        Thus only nodes with the same index can possibly match.
        Now fix $i \in [n]$. 
        The round-$2$ color of $w^{(A)}_i$ is determined by its round-$1$ color together with the multiset of round-$1$ colors of its neighbors. 
        These neighbors consist of $x^{(A)}$, the two nodes $u^{(A)}_{i, 1}, u^{(A)}_{i, 2}$, and a number of padding nodes determined only by the virtual color of $w^{(A)}_i$. 
        The same holds for $w^{(B)}_i$. 
        The contribution from $x^{(A)}$ and $x^{(B)}$ is identical, and the padding contribution is also identical on the two sides. 
        Therefore, $\chi_G^{(2)}\left(w^{(A)}_i\right) = \chi_G^{(2)}\left(w^{(B)}_i\right)$ if and only if the two multisets of virtual colors assigned to $u^{(A)}_{i, 1},u^{(A)}_{i, 2}$ and $u^{(B)}_{i, 1},u^{(B)}_{i, 2}$ are equal, which holds if and only if $\bs{a}_i=\bs{b}_i$.
        Consequently, the two sets of round-$2$ colors $\set{\chi_G^{(2)}\left(w^{(A)}_i\right) : i\in[n]}$ and $\set{\chi_G^{(2)}\left(w^{(B)}_i\right) : i\in[n]}$ intersect if and only if there exists $i\in[n]$ such that $\bs{a}_i = \bs{b}_i$. 
        In other words, two steps of the WL test on the decolorized graph solve the predicate $\exists\func{EQ}_{U, n}$.
    \end{itemize}

    \paragraph{Lower Bound.}
    Suppose that there exists a randomized oblivious MPGNN $\mathcal{M}$ with depth $d$, width $w$, and precision $p$ that simulates one step of the WL test on $\mathcal{G}_{n, D}$ with failure probability at most $p_f$.
    By composing two independent copies of $\mathcal{M}$, Alice and Bob can simulate two steps of the WL test on $G_{\bs{a},\bs{b}}$.
    The composed simulator has depth $2d$, and it succeeds with probability at least $(1 - p_f)^2 > 1 - 2p_f > 2/3$.

    \begin{itemize}[leftmargin=*]
        \item \textit{\textbf{The Communication Term.}}
        First consider the communication lower bound. 
        Alice and Bob simulate the composed MPGNN across the Alice-Bob cut. 
        The cut contains exactly one graph edge, and hence each message-passing layer requires only $O(wp)$ bits of communication across the cut. 
        Over $2d$ layers, the simulation uses $O(dwp)$ bits. 
        After the simulation, Alice sends Bob a $C$-bit mask indicating which colors appear in the set $\set{\chi_G^{(2)}\left(w^{(A)}_i\right) : i \in [n]}$. 
        Since $C = \Theta(n)$, this costs $O(n)$ bits. 
        Bob accepts if and only if one of the colors in $\set{\chi_G^{(2)}\left(w^{(B)}_i\right) : i \in [n]}$ appears in Alice's mask. 
        By the WL analysis above, this protocol solves $\exists\func{EQ}_{U, n}$ with failure probability at most $2p_f < 1/3$, communication cost $O(dwp + n)$, and round complexity $O(d)$.        
        The round-sensitive lower bound for $\exists\func{EQ}_{U, n}$ with $|U| = \Theta(n^2)$~\cite{saglam2013communication, brody2016certifying} gives communication $\Omega(n\log^{(r)} n)$ for $r$-round public-coin protocols. 
        If $r = o(\log^* n)$, then $\log^{(r)} n\to\infty$, so the required communication is $\omega(n)$. 
        Thus an $O(n)$-communication protocol must have $r = \Omega(\log^* n)$.
        If $dwp = o(n)$, the constructed protocol has $O(n)$ communication and $O(d)$ rounds, so $d = \Omega(\log^* n)$.
        Hence either $d = \Omega(\log^* n)$ or $dwp = \Omega(n)$. 
        Equivalently, $d = \Omega\left(\min\left\{\log^* n,\frac{n}{wp}\right\}\right)$.

        \item \textit{\textbf{The Diameter Term.}}
        It remains to prove the diameter term. 
        The two side gadgets are connected by a path of length $\Theta(D)$. 
        Suppose, for contradiction, that $d = o(D)$. 
        Then the depth-$2d$ composed MPGNN cannot transmit information from Alice's side to Bob's side through the connecting path. 
        Thus the colors output on Alice's $w$-nodes are functions only of Alice's input and the public randomness, and the colors output on Bob's $w$-nodes are functions only of Bob's input and the public randomness. 
        Alice can therefore locally compute her color set, send the same $O(n)$-bit color mask to Bob, and Bob can decide whether the two color sets intersect. 
        This gives a one-way ($r = 1$) randomized protocol for $\exists\func{EQ}_{U, n}$ with $O(n)$ communication and bounded error.
        However, the one-way randomized communication complexity of $\exists\func{EQ}_{U, n}$ with $|U| = \Theta(n^2)$ is $\Omega(n\log n)$ for bounded error~\cite{saglam2013communication, brody2016certifying}. 
        This contradicts the $O(n)$-bit one-way protocol above. 
        Therefore $d = \Omega(D)$.
    \end{itemize}    
    Combining $d = \Omega(D)$ and $d = \Omega\left(\min\left\{\log^* n, \frac{n}{wp}\right\}\right)$ gives the desired lower bound $d = \Omega\left(D + \min\left\{\log^* n, \frac{n}{wp}\right\}\right)$.
\end{proof}

\section{Proof of Theorem~\ref{thm:instance-dependent-representation}}
\label{app:proof-instance-dependent-representation}
In this section we prove Theorem~\ref{thm:instance-dependent-representation}.
\InstDepRep*
\begin{proof}
Let $K = \set{\chi(u): u \in V(G)}$ be the set of input colors that occur in the instance.
Since $|K| \le |V(G)| = n$, we can choose an injective compression function $c: K \to \set{0, 1}^{\lceil \log_2 n\rceil}$.
The instance-dependent construction first uses an input embedding function that maps each color $k \in K$ to $c(k)$ and maps the remaining colors in $[C] \setminus K$ arbitrarily.
Since $wp \ge \lceil \log_2 n\rceil$, each node can send the compressed color to all neighbors in one message-passing layer.

After this layer, node $u$ has access to the compressed WL-type $\left(c(\chi(u)), \mset{c(\chi(v)): v \in N(u)}\right)$, represented by a canonical vectorization, for example by the method used in Theorem~\ref{thm:randomized-upper-large-color}.
We denote this vector by $\bs{t}_u$.
Because $c$ is injective on the input colors that occur, this compressed WL-type is the same for two nodes $u, v \in V(G)$ if and only if $\WLtype_{G,\chi}(u) = \WLtype_{G,\chi}(v)$.

There are at most $n$ compressed WL-types that occur in the input instance.
Since $C \ge n$, we can choose an injective mapping from these vectorized compressed WL-types $\bs{t}_u$ to $[C]$.
The instance-dependent decoder implements this lookup table.
Therefore, two nodes receive the same output color if and only if they have the same WL-type.
\end{proof}

\section{Proof of Theorem~\ref{thm:instance-dependent-adaptation-lower}}
\label{app:proof-instance-dependent-adaptation-lower}
In this section we prove Theorem~\ref{thm:instance-dependent-adaptation-lower}.

\InstDepAdaLower*
\begin{proof}
By the seed-fixing argument in Section~\ref{subsec:setup}, it suffices to prove the deterministic case.
We reduce the Equality problem to the decentralized parameter-adaptation process. 
Given an instance $(\bs{a}, \bs{b})\in \set{0, 1}^m \times \set{0, 1}^m$, Alice and Bob construct their two sides of the hard graph $G_{\bs{a},\bs{b}}$ from the proof of Theorem~\ref{thm:deterministic-lower}. 
The initial parameter $\theta_0$, the adaptation algorithm, and all target-independent information are fixed before the inputs $(\bs{a}, \bs{b})$ are known.

Starting from the uniform coloring $\widehat{\chi}^{(0)}$, Alice and Bob simulate three consecutive applications of the decentralized parameter-adaptation framework. 
Concretely, for $t = 0, 1, 2$, they run $\theta^{(t)}_{r} = \alg{Adapt}^{(r)}(\theta_0, G_{\bs a,\bs b}, \widehat\chi^{(t)})$ and then compute $\widehat\chi^{(t + 1)} = \mathcal{N}_{\theta^{(t)}_{r}}\left(\widehat\chi^{(t)}; G_{\bs{a},\bs{b}}\right)$.
By induction on $t$, each $\widehat\chi^{(t)}$ is equivalent to the coloring $\chi^{(t)}_{G_{\bs a, \bs b}}$ after $t$ steps of the WL test.
Indeed, $\widehat{\chi}^{(0)}$ is the uniform coloring, and the success guarantee of the framework implies that one application maps any valid input coloring to the coloring obtained after one further step of the WL test, up to relabeling.

We first prove the communication lower bound. 
During one application of the decentralized parameter-adaptation framework, the $r$ adaptation rounds transmit at most $O(rdwp)$ bits across the Alice-Bob cut, while the final depth-$d$ inference pass transmits at most $O(dwp)$ bits across the cut. 
Repeating the framework three times only changes the communication by a constant factor. 
After the three applications, Alice holds the final color of $x^{(A)}$, and Bob holds the final color of $x^{(B)}$. 
They compare these two colors using $O(\log C)$ additional bits.

By the hard-instance property proved in Theorem~\ref{thm:deterministic-lower}, $\chi^{(3)}_{G_{\bs{a}, \bs{b}}}\left(x^{(A)}\right)=\chi^{(3)}_{G_{\bs{a}, \bs{b}}}\left(x^{(B)}\right) \iff \bs{a} = \bs{b}$.
Since $\widehat\chi^{(3)}\equiv \chi^{(3)}_{G_{\bs{a},\bs{b}}}$, comparing the two final output colors solves $\func{EQ}_m$. 
Therefore the decentralized parameter-adaptation framework yields a deterministic communication protocol for $\func{EQ}_m$ with communication $O(rdwp + dwp + \log C)$.
The deterministic communication complexity of $\func{EQ}_m$ is $\Omega(m)$. 
Since $\log C = o(m)$, we obtain $(r + 1)dwp = \Omega(m)$.

We next prove the locality lower bound. 
Let $\ell = \func{dist}_{G_{\bs a,\bs b}}\left(x^{(A)}, x^{(B)}\right) = \Theta(D)$. 
To simulate one WL step, the procedure runs $r$ adaptation rounds followed by a depth-$d$ inference pass.
Information can travel $O(d)$ hops per adaptation round and at most $d$ hops during inference.
Across all three simulated WL steps, the propagation radius is therefore $O(rd)$ for $r\ge 1$ and $O(d)$ for $r=0$.
Suppose, for contradiction, that this radius is less than $\ell/2$ on every graph in the hard family.
Then the final output color of $x^{(A)}$ is independent of Bob's input $\bs{b}$, and the final output color of $x^{(B)}$ is independent of Alice's input $\bs{a}$.
Hence there exist functions $F, H : \set{0, 1}^m\to [C]$ such that the final output color of $x^{(A)}$ is $F(\bs{a})$, while the final output color of $x^{(B)}$ is $H(\bs{b})$.
Correctness of the adaptation framework and the hard-instance property imply $F(\bs{a}) = H(\bs{b}) \iff \bs a = \bs{b}$ for all $\bs{a}, \bs{b} \in \set{0, 1}^m$. 
In particular, $F(\bs{a}) = H(\bs{a})$ for every $\bs{a}$. 
Moreover, $F$ must be injective. 
If $F(\bs{a}) = F(\bs{a}')$, then $F(\bs{a}') = F(\bs{a}) = H(\bs{a})$, which implies $\bs{a}' = \bs{a}$. 
Therefore $|\func{Im}(F)| = 2^m$, and hence $C\ge 2^m$. 
This contradicts $\log C = o(m)$. 
Thus the propagation radius is at least $\ell/2=\Omega(D)$ on some graph in the family, so $(r+1)d=\Omega(D)$.

For $r=0$, these two bounds give the necessary condition $d=\Omega\left(D+\frac{m}{wp}\right)$ for simulation without adaptation.
Combining the two bounds for $r \ge 1$, we obtain $r = \Omega\left(\frac{D}{d} + \frac{m}{dwp}\right)$.
\end{proof}

\newpage
\section*{NeurIPS Paper Checklist}
\begin{enumerate}

\item {\bf Claims}
    \item[] Question: Do the main claims made in the abstract and introduction accurately reflect the paper's contributions and scope?
    \item[] Answer: \answerYes{}
    \item[] Justification: The abstract and introduction summarize the simulation settings in Section~\ref{subsec:setup} and the resource bounds in Sections~\ref{sec:oblivious-results} and~\ref{sec:instance-dependent-results}, also collected in Table~\ref{tab:results}.
    \item[] Guidelines:
    \begin{itemize}
        \item The answer \answerNA{} means that the abstract and introduction do not include the claims made in the paper.
        \item The abstract and/or introduction should clearly state the claims made, including the contributions made in the paper and important assumptions and limitations. A \answerNo{} or \answerNA{} answer to this question will not be perceived well by the reviewers. 
        \item The claims made should match theoretical and experimental results, and reflect how much the results can be expected to generalize to other settings. 
        \item It is fine to include aspirational goals as motivation as long as it is clear that these goals are not attained by the paper. 
    \end{itemize}

\item {\bf Limitations}
    \item[] Question: Does the paper discuss the limitations of the work performed by the authors?
    \item[] Answer: \answerYes{}
    \item[] Justification: Section~\ref{sec:discussion} discusses the limitations of our models and theoretical results.
    \item[] Guidelines:
    \begin{itemize}
        \item The answer \answerNA{} means that the paper has no limitation while the answer \answerNo{} means that the paper has limitations, but those are not discussed in the paper. 
        \item The authors are encouraged to create a separate ``Limitations'' section in their paper.
        \item The paper should point out any strong assumptions and how robust the results are to violations of these assumptions (e.g., independence assumptions, noiseless settings, model well-specification, asymptotic approximations only holding locally). The authors should reflect on how these assumptions might be violated in practice and what the implications would be.
        \item The authors should reflect on the scope of the claims made, e.g., if the approach was only tested on a few datasets or with a few runs. In general, empirical results often depend on implicit assumptions, which should be articulated.
        \item The authors should reflect on the factors that influence the performance of the approach. For example, a facial recognition algorithm may perform poorly when image resolution is low or images are taken in low lighting. Or a speech-to-text system might not be used reliably to provide closed captions for online lectures because it fails to handle technical jargon.
        \item The authors should discuss the computational efficiency of the proposed algorithms and how they scale with dataset size.
        \item If applicable, the authors should discuss possible limitations of their approach to address problems of privacy and fairness.
        \item While the authors might fear that complete honesty about limitations might be used by reviewers as grounds for rejection, a worse outcome might be that reviewers discover limitations that aren't acknowledged in the paper. The authors should use their best judgment and recognize that individual actions in favor of transparency play an important role in developing norms that preserve the integrity of the community. Reviewers will be specifically instructed to not penalize honesty concerning limitations.
    \end{itemize}

\item {\bf Theory assumptions and proofs}
    \item[] Question: For each theoretical result, does the paper provide the full set of assumptions and a complete (and correct) proof?
    \item[] Answer: \answerYes{}
    \item[] Justification: Sections~\ref{sec:preliminaries}--\ref{sec:instance-dependent-results} specify the models, assumptions, and theorem statements. Full proofs appear in Appendices~\ref{app:proof-deterministic-lower}--\ref{app:proof-instance-dependent-adaptation-lower}.
    \item[] Guidelines:
    \begin{itemize}
        \item The answer \answerNA{} means that the paper does not include theoretical results. 
        \item All the theorems, formulas, and proofs in the paper should be numbered and cross-referenced.
        \item All assumptions should be clearly stated or referenced in the statement of any theorems.
        \item The proofs can either appear in the main paper or the supplemental material, but if they appear in the supplemental material, the authors are encouraged to provide a short proof sketch to provide intuition. 
        \item Inversely, any informal proof provided in the core of the paper should be complemented by formal proofs provided in appendix or supplemental material.
        \item Theorems and Lemmas that the proof relies upon should be properly referenced. 
    \end{itemize}

    \item {\bf Experimental result reproducibility}
    \item[] Question: Does the paper fully disclose all the information needed to reproduce the main experimental results of the paper to the extent that it affects the main claims and/or conclusions of the paper (regardless of whether the code and data are provided or not)?
    \item[] Answer: \answerNA{}
    \item[] Justification: Our paper is purely theoretical and does not include experiments.
    \item[] Guidelines:
    \begin{itemize}
        \item The answer \answerNA{} means that the paper does not include experiments.
        \item If the paper includes experiments, a \answerNo{} answer to this question will not be perceived well by the reviewers: Making the paper reproducible is important, regardless of whether the code and data are provided or not.
        \item If the contribution is a dataset and\slash or model, the authors should describe the steps taken to make their results reproducible or verifiable. 
        \item Depending on the contribution, reproducibility can be accomplished in various ways. For example, if the contribution is a novel architecture, describing the architecture fully might suffice, or if the contribution is a specific model and empirical evaluation, it may be necessary to either make it possible for others to replicate the model with the same dataset, or provide access to the model. In general. releasing code and data is often one good way to accomplish this, but reproducibility can also be provided via detailed instructions for how to replicate the results, access to a hosted model (e.g., in the case of a large language model), releasing of a model checkpoint, or other means that are appropriate to the research performed.
        \item While NeurIPS does not require releasing code, the conference does require all submissions to provide some reasonable avenue for reproducibility, which may depend on the nature of the contribution. For example
        \begin{enumerate}
            \item If the contribution is primarily a new algorithm, the paper should make it clear how to reproduce that algorithm.
            \item If the contribution is primarily a new model architecture, the paper should describe the architecture clearly and fully.
            \item If the contribution is a new model (e.g., a large language model), then there should either be a way to access this model for reproducing the results or a way to reproduce the model (e.g., with an open-source dataset or instructions for how to construct the dataset).
            \item We recognize that reproducibility may be tricky in some cases, in which case authors are welcome to describe the particular way they provide for reproducibility. In the case of closed-source models, it may be that access to the model is limited in some way (e.g., to registered users), but it should be possible for other researchers to have some path to reproducing or verifying the results.
        \end{enumerate}
    \end{itemize}

\item {\bf Open access to data and code}
    \item[] Question: Does the paper provide open access to the data and code, with sufficient instructions to faithfully reproduce the main experimental results, as described in supplemental material?
    \item[] Answer: \answerNA{}
    \item[] Justification: Our paper is purely theoretical and does not include experiments.
    \item[] Guidelines:
    \begin{itemize}
        \item The answer \answerNA{} means that paper does not include experiments requiring code.
        \item Please see the NeurIPS code and data submission guidelines (\url{https://neurips.cc/public/guides/CodeSubmissionPolicy}) for more details.
        \item While we encourage the release of code and data, we understand that this might not be possible, so \answerNo{} is an acceptable answer. Papers cannot be rejected simply for not including code, unless this is central to the contribution (e.g., for a new open-source benchmark).
        \item The instructions should contain the exact command and environment needed to run to reproduce the results. See the NeurIPS code and data submission guidelines (\url{https://neurips.cc/public/guides/CodeSubmissionPolicy}) for more details.
        \item The authors should provide instructions on data access and preparation, including how to access the raw data, preprocessed data, intermediate data, and generated data, etc.
        \item The authors should provide scripts to reproduce all experimental results for the new proposed method and baselines. If only a subset of experiments are reproducible, they should state which ones are omitted from the script and why.
        \item At submission time, to preserve anonymity, the authors should release anonymized versions (if applicable).
        \item Providing as much information as possible in supplemental material (appended to the paper) is recommended, but including URLs to data and code is permitted.
    \end{itemize}

\item {\bf Experimental setting/details}
    \item[] Question: Does the paper specify all the training and test details (e.g., data splits, hyperparameters, how they were chosen, type of optimizer) necessary to understand the results?
    \item[] Answer: \answerNA{}
    \item[] Justification: Our paper is purely theoretical and does not include experiments.
    \item[] Guidelines:
    \begin{itemize}
        \item The answer \answerNA{} means that the paper does not include experiments.
        \item The experimental setting should be presented in the core of the paper to a level of detail that is necessary to appreciate the results and make sense of them.
        \item The full details can be provided either with the code, in appendix, or as supplemental material.
    \end{itemize}

\item {\bf Experiment statistical significance}
    \item[] Question: Does the paper report error bars suitably and correctly defined or other appropriate information about the statistical significance of the experiments?
    \item[] Answer: \answerNA{}
    \item[] Justification: Our paper is purely theoretical and does not include experiments.
    \item[] Guidelines:
    \begin{itemize}
        \item The answer \answerNA{} means that the paper does not include experiments.
        \item The authors should answer \answerYes{} if the results are accompanied by error bars, confidence intervals, or statistical significance tests, at least for the experiments that support the main claims of the paper.
        \item The factors of variability that the error bars are capturing should be clearly stated (for example, train/test split, initialization, random drawing of some parameter, or overall run with given experimental conditions).
        \item The method for calculating the error bars should be explained (closed form formula, call to a library function, bootstrap, etc.)
        \item The assumptions made should be given (e.g., Normally distributed errors).
        \item It should be clear whether the error bar is the standard deviation or the standard error of the mean.
        \item It is OK to report 1-sigma error bars, but one should state it. The authors should preferably report a 2-sigma error bar than state that they have a 96\% CI, if the hypothesis of Normality of errors is not verified.
        \item For asymmetric distributions, the authors should be careful not to show in tables or figures symmetric error bars that would yield results that are out of range (e.g., negative error rates).
        \item If error bars are reported in tables or plots, the authors should explain in the text how they were calculated and reference the corresponding figures or tables in the text.
    \end{itemize}

\item {\bf Experiments compute resources}
    \item[] Question: For each experiment, does the paper provide sufficient information on the computer resources (type of compute workers, memory, time of execution) needed to reproduce the experiments?
    \item[] Answer: \answerNA{}
    \item[] Justification: Our paper is purely theoretical and does not include experiments.
    \item[] Guidelines:
    \begin{itemize}
        \item The answer \answerNA{} means that the paper does not include experiments.
        \item The paper should indicate the type of compute workers CPU or GPU, internal cluster, or cloud provider, including relevant memory and storage.
        \item The paper should provide the amount of compute required for each of the individual experimental runs as well as estimate the total compute. 
        \item The paper should disclose whether the full research project required more compute than the experiments reported in the paper (e.g., preliminary or failed experiments that didn't make it into the paper). 
    \end{itemize}
    
\item {\bf Code of ethics}
    \item[] Question: Does the research conducted in the paper conform, in every respect, with the NeurIPS Code of Ethics \url{https://neurips.cc/public/EthicsGuidelines}?
    \item[] Answer: \answerYes{}
    \item[] Justification: We confirm that our paper conforms, in every respect, with the NeurIPS Code of Ethics.
    \item[] Guidelines:
    \begin{itemize}
        \item The answer \answerNA{} means that the authors have not reviewed the NeurIPS Code of Ethics.
        \item If the authors answer \answerNo, they should explain the special circumstances that require a deviation from the Code of Ethics.
        \item The authors should make sure to preserve anonymity (e.g., if there is a special consideration due to laws or regulations in their jurisdiction).
    \end{itemize}

\item {\bf Broader impacts}
    \item[] Question: Does the paper discuss both potential positive societal impacts and negative societal impacts of the work performed?
    \item[] Answer: \answerNA{}
    \item[] Justification: Our paper only theoretically studies upper and lower bounds on the resources required for MPGNNs to implement one step of the WL test under different settings. Therefore, we view our work as purely theoretical, with no concrete or direct societal impacts.
    \item[] Guidelines:
    \begin{itemize}
        \item The answer \answerNA{} means that there is no societal impact of the work performed.
        \item If the authors answer \answerNA{} or \answerNo, they should explain why their work has no societal impact or why the paper does not address societal impact.
        \item Examples of negative societal impacts include potential malicious or unintended uses (e.g., disinformation, generating fake profiles, surveillance), fairness considerations (e.g., deployment of technologies that could make decisions that unfairly impact specific groups), privacy considerations, and security considerations.
        \item The conference expects that many papers will be foundational research and not tied to particular applications, let alone deployments. However, if there is a direct path to any negative applications, the authors should point it out. For example, it is legitimate to point out that an improvement in the quality of generative models could be used to generate Deepfakes for disinformation. On the other hand, it is not needed to point out that a generic algorithm for optimizing neural networks could enable people to train models that generate Deepfakes faster.
        \item The authors should consider possible harms that could arise when the technology is being used as intended and functioning correctly, harms that could arise when the technology is being used as intended but gives incorrect results, and harms following from (intentional or unintentional) misuse of the technology.
        \item If there are negative societal impacts, the authors could also discuss possible mitigation strategies (e.g., gated release of models, providing defenses in addition to attacks, mechanisms for monitoring misuse, mechanisms to monitor how a system learns from feedback over time, improving the efficiency and accessibility of ML).
    \end{itemize}
    
\item {\bf Safeguards}
    \item[] Question: Does the paper describe safeguards that have been put in place for responsible release of data or models that have a high risk for misuse (e.g., pre-trained language models, image generators, or scraped datasets)?
    \item[] Answer: \answerNA{}
    \item[] Justification: We do not use models that have a high risk for misuse.
    \item[] Guidelines:
    \begin{itemize}
        \item The answer \answerNA{} means that the paper poses no such risks.
        \item Released models that have a high risk for misuse or dual-use should be released with necessary safeguards to allow for controlled use of the model, for example by requiring that users adhere to usage guidelines or restrictions to access the model or implementing safety filters. 
        \item Datasets that have been scraped from the Internet could pose safety risks. The authors should describe how they avoided releasing unsafe images.
        \item We recognize that providing effective safeguards is challenging, and many papers do not require this, but we encourage authors to take this into account and make a best faith effort.
    \end{itemize}

\item {\bf Licenses for existing assets}
    \item[] Question: Are the creators or original owners of assets (e.g., code, data, models), used in the paper, properly credited and are the license and terms of use explicitly mentioned and properly respected?
    \item[] Answer: \answerNA{}
    \item[] Justification: Our paper does not use existing assets.
    \item[] Guidelines:
    \begin{itemize}
        \item The answer \answerNA{} means that the paper does not use existing assets.
        \item The authors should cite the original paper that produced the code package or dataset.
        \item The authors should state which version of the asset is used and, if possible, include a URL.
        \item The name of the license (e.g., CC-BY 4.0) should be included for each asset.
        \item For scraped data from a particular source (e.g., website), the copyright and terms of service of that source should be provided.
        \item If assets are released, the license, copyright information, and terms of use in the package should be provided. For popular datasets, \url{paperswithcode.com/datasets} has curated licenses for some datasets. Their licensing guide can help determine the license of a dataset.
        \item For existing datasets that are re-packaged, both the original license and the license of the derived asset (if it has changed) should be provided.
        \item If this information is not available online, the authors are encouraged to reach out to the asset's creators.
    \end{itemize}

\item {\bf New assets}
    \item[] Question: Are new assets introduced in the paper well documented and is the documentation provided alongside the assets?
    \item[] Answer: \answerNA{}
    \item[] Justification: Our paper does not release new assets.
    \item[] Guidelines:
    \begin{itemize}
        \item The answer \answerNA{} means that the paper does not release new assets.
        \item Researchers should communicate the details of the dataset\slash code\slash model as part of their submissions via structured templates. This includes details about training, license, limitations, etc. 
        \item The paper should discuss whether and how consent was obtained from people whose asset is used.
        \item At submission time, remember to anonymize your assets (if applicable). You can either create an anonymized URL or include an anonymized zip file.
    \end{itemize}

\item {\bf Crowdsourcing and research with human subjects}
    \item[] Question: For crowdsourcing experiments and research with human subjects, does the paper include the full text of instructions given to participants and screenshots, if applicable, as well as details about compensation (if any)? 
    \item[] Answer: \answerNA{}
    \item[] Justification: Our paper does not involve crowdsourcing nor research with human subjects.
    \item[] Guidelines:
    \begin{itemize}
        \item The answer \answerNA{} means that the paper does not involve crowdsourcing nor research with human subjects.
        \item Including this information in the supplemental material is fine, but if the main contribution of the paper involves human subjects, then as much detail as possible should be included in the main paper. 
        \item According to the NeurIPS Code of Ethics, workers involved in data collection, curation, or other labor should be paid at least the minimum wage in the country of the data collector. 
    \end{itemize}

\item {\bf Institutional review board (IRB) approvals or equivalent for research with human subjects}
    \item[] Question: Does the paper describe potential risks incurred by study participants, whether such risks were disclosed to the subjects, and whether Institutional Review Board (IRB) approvals (or an equivalent approval/review based on the requirements of your country or institution) were obtained?
    \item[] Answer: \answerNA{}
    \item[] Justification: Our paper does not involve crowdsourcing nor research with human subjects.
    \item[] Guidelines:
    \begin{itemize}
        \item The answer \answerNA{} means that the paper does not involve crowdsourcing nor research with human subjects.
        \item Depending on the country in which research is conducted, IRB approval (or equivalent) may be required for any human subjects research. If you obtained IRB approval, you should clearly state this in the paper. 
        \item We recognize that the procedures for this may vary significantly between institutions and locations, and we expect authors to adhere to the NeurIPS Code of Ethics and the guidelines for their institution. 
        \item For initial submissions, do not include any information that would break anonymity (if applicable), such as the institution conducting the review.
    \end{itemize}

\item {\bf Declaration of LLM usage}
    \item[] Question: Does the paper describe the usage of LLMs if it is an important, original, or non-standard component of the core methods in this research? Note that if the LLM is used only for writing, editing, or formatting purposes and does \emph{not} impact the core methodology, scientific rigor, or originality of the research, declaration is not required.
    \item[] Answer: \answerNA{}
    \item[] Justification: We confirm that the core method development in this research does not involve LLMs as any important, original, or non-standard components.
    \item[] Guidelines:
    \begin{itemize}
        \item The answer \answerNA{} means that the core method development in this research does not involve LLMs as any important, original, or non-standard components.
        \item Please refer to our LLM policy in the NeurIPS handbook for what should or should not be described.
    \end{itemize}

\end{enumerate}

\end{document}